%% file: main.tex
\PassOptionsToPackage{unicode}{hyperref}
\PassOptionsToPackage{naturalnames}{hyperref}

\documentclass[11pt]{article}
\usepackage[margin=1in]{geometry}
\usepackage{hyperref}
\usepackage[round]{natbib}
\usepackage{amsmath}
\usepackage{amsthm}
\usepackage{amsfonts}
\usepackage{amssymb}
\usepackage{enumitem}

\usepackage{amsbsy}
\usepackage{appendix}
\usepackage{comment}
\usepackage{bbm}
\usepackage{lmodern}
\usepackage[normalem]{ulem}
\RequirePackage[capitalize]{cleveref}
\RequirePackage{crossreftools}

\allowdisplaybreaks
\newcommand{\unitballd}{\mathbb{B}^{d}}
\newcommand{\vecentry}[2]{#1[#2]}

\newcommand{\empf}{\widehat{F}}

\newcommand{\popf}{F}
\newcommand{\erm}{\widehat{w}_*}
\newcommand{\indsec}{k}
\newcommand{\indsuff}{m}
\newcommand{\D}{\mathcal{D}}
\newcommand{\suff}{w_{T,\indsuff}}

\newcommand{\sgdind}{\text{SGD}}

\newcommand{\gdindf}{f}
\newcommand{\optind}{\text{OPT}}
\newcommand{\optindf}{f^{\optind}}
\newcommand{\doptindf}{\Tilde{f}^{\optind}}
\newcommand{\gdindD}{\D}

\newcommand{\sgdindD}{\D^{\sgdind}}
\newcommand{\sgdindZ}{Z^{\sgdind}}

\newcommand{\gdindempf}{\empf}
\newcommand{\gdindpopf}{\popf}
\newcommand{\sgdindf}{f^{\sgdind}}
\newcommand{\sgdindempf}{\empf^{\sgdind}}

\newcommand{\dsgdindf}{\Tilde{f}^{\sgdind}}
\newcommand{\dsgdindempf}{\widehat{\dgdindpopf}^{\sgdind}}
\newcommand{\dsgdindpopf}{\Tilde{F}^{\sgdind}}

\newcommand{\dgdindf}{\Tilde{f}}
\newcommand{\dgdindempf}{\widehat{\dgdindpopf}}
\newcommand{\dgdindpopf}{\Tilde{F}}

\newcommand{\gdindgam}{\ell}

\newcommand{\sgdindgam}{\ell^{\sgdind}}
\newcommand{\dsgdindgam}{\Tilde{\ell}^{\sgdind}}

\newcommand{\dgdindgam}{\Tilde{\ell}}

\newcommand{\vecpart}[2]{{#1}^{(#2)}}
\newcommand{\last}{{0}}
\newcommand{\enc}[1]{\vecpart{#1}{{\last}}}
\newcommand{\wenc}{\enc{w}}
\newcommand{\wenct}[1]{\enc{w_{#1}}}

\newcommand{\opt}{w_*}

\newcommand{\wh}[1]{\smash{\widehat{#1}}}

\newcommand{\iid}{\mathop{\smash[t]{\overset{\mathrm{iid}}{\sim}}}}

\usepackage{header}

\title{The Dimension Strikes Back with Gradients:
Generalization of Gradient Methods in Stochastic Convex Optimization}

\author{Matan Schliserman \and Uri Sherman \and Tomer Koren}

\begin{document}
\maketitle 
\begin{abstract}
    \input{abstract}
\end{abstract}

\section{Introduction}
\label{sec:intro}

The study of generalization properties of stochastic optimization algorithms has been at the heart of contemporary machine learning research.
While in the more classical frameworks studies largely focused on the learning \emph{problem}~\citep[e.g.,][]{alon1997scale, blumer1989learnability},
in the past decade it has become clear that in modern scenarios the particular algorithm used to learn the model plays a vital role in its generalization performance.
As a prominent example, heavily over-parameterized deep neural networks trained by first order methods output models that generalize well, despite the fact that an arbitrarily chosen Empirical Risk Minimizer (ERM) may perform poorly \citep{zhang2017understanding, neyshabur2014search, neyshabur2017exploring}.
The present paper aims at understanding the generalization behavior of gradient methods, specifically in connection with the problem dimension, in the fundamental Stochastic Convex Optimization (SCO) learning setup; a well studied, theoretical framework widely used to study stochastic optimization algorithms.

The seminal work of \citet{shalev2010learnability} was the first to show that uniform convergence, the canonical condition for generalization in statistical learning~\citep[e.g.,][]{vapnik1971uniform, bartlett2002rademacher} may not hold in high-dimensional SCO: they demonstrated learning problems where there exist certain ERMs that overfit the training data (i.e., exhibit large population risk), while models produced by e.g., Stochastic Gradient Descent (SGD) or regularized empirical risk minimization generalize well.
The construction presented by \citet{shalev2010learnability}, however, featured a learning problem with dimension exponential in the number of training examples, which only served to prove an $\Omega(\log{d})$ lower bound on the sample complexity for reaching non-trivial population risk performance, where $d$ is the problem dimension.
In a followup work, \citet{feldman2016generalization} showed how to dramatically improve the dimension dependence and established an $\Omega(d)$ sample complexity lower bound, matching (in terms of $d$) the well-known upper bound obtained from standard covering number arguments \citep[see e.g.,][]{shalev2014understanding}.

Despite settling the dimension dependence of uniform convergence in SCO, it remained unclear from \citet{shalev2010learnability,feldman2016generalization} whether the sample complexity lower bounds for uniform convergence actually transfer to natural learning algorithms in this framework, and in particular, to common gradient-based optimization methods.
Indeed, it is well-known that in SCO there exist simple algorithms, such as SGD, that the models they produce actually generalize well with high probability \citep[see e.g.,][]{shalev2014understanding}, despite these lower bounds.  
More technically, the construction of \citet{feldman2016generalization} relied heavily on the existence of a ``peculiar'' ERM which does not seem reachable by gradient steps from a data-independent initialization, and it was not at all clear (and in fact, stated as an open problem in \citealp{feldman2016generalization}) how to adapt the construction so as to pertain to ERMs that could be found by gradient methods.
 
In an attempt to address this issue, \citet{amir2021sgd} recently studied the population performance of batch Gradient Descent (GD) in SCO, and demonstrated problem instances where it leads (with constant probability) to an approximate ERM that generalizes poorly, unless the number of training examples is dimension-dependent.%
\footnote{Here we refer to GD as performing $T=n$ iterations with stepsize $\eta = \Theta(1/\sqrt{n})$, where $n$ denotes the size of the training set, but our results hold more generally; see below for a more detailed discussion of the various regimes.} 
Subsequently, \citet{amir2021never} generalized this result to the more general class of batch first-order algorithms.
However, due to technical complications, the constructions in these papers were based in part on the earlier arguments of \citet{shalev2010learnability} rather than the developments by \citet{feldman2016generalization}, and therefore fell short of establishing their results in dimension polynomial in the number of training examples. 
As a consequence, their results are unable to rule out a sample complexity upper bound for GD that depends only (poly-)logarithmically on problem dimension.

In this work, we resolve the open questions posed in both \citet{feldman2016generalization} and \citet{amir2021sgd}: Our first main result demonstrates a convex learning problem where GD, unless trained with at least $\Omega(\sqrt d)$ training examples, outputs a bad ERM with constant probability.
This bridges the gap between the results of \citet{feldman2016generalization} and actual, concrete learning algorithms (albeit with a slightly weaker rate of $\Omega(\sqrt d)$, compared to the $\Omega(d)$ of the latter paper) and greatly improves on the previous $\Omega(\log{d})$ lower bound of \citet{amir2021sgd}, establishing that the sample complexity of batch GD in SCO has a significant, polynomial dependence on the problem dimension.

Furthermore, in our second main result we show how an application of the same construction technique provides a similar improvement in the dimension dependence of the \emph{empirical risk} lower bound presented in the recent work of \citet{benignunderfit}, thus also resolving the open question left in their work.
This work demonstrated that in SCO, well-tuned SGD may \emph{underfit} the training data despite achieving optimal population risk performance.
At a deeper level, the overfitting of GD and underfitting of SGD both stem from a combination of two conditions: lack of algorithmic stability, and failure of uniform convergence; as it turns out, this combination allows for the output models to exhibit a large \emph{generalization gap}, defined as the difference in absolute value between the empirical and population risks.
Our work presents a construction technique for such generalization gap lower bounds that achieves small polynomial dimension dependence, providing for an exponential improvement over previous works.

\subsection{Our contributions}

In some more detail, our main contributions are as follows:
\begin{enumerate}[label=(\roman*)]
    \item 
    We present a construction of a learning problem in dimension $d = O(nT+n^2+\eta^2T^2)$ where running GD for $T$ iterations with step $\eta$ over a training set of $n$ i.i.d.-sampled examples leads, with constant probability, to a solution with population error $\Omega(\eta\sqrt T + 1/\eta T)$.%
    \footnote{By population error (or test error) we mean the population excess risk, namely the gap in population risk between the returned solution and the optimal solution.}
    In particular, for the canonical configuration of $T=n$ and $\eta = \Theta(1/\sqrt{n})$, the lower bound becomes $\Omega(1)$ and demonstrates that GD suffers from catastrophic overfitting already in dimension $d = O(n^2)$.
    Put differently, this translates to an $\widetilde\Omega(\sqrt{d})$ lower bound the number of training examples required for GD to reach nontrivial test error.
    See \cref{GDmain} below for a formal statement and further implications of this result.
    \item 
    Furthermore, we give a construction of dimension $d=\widetilde{O}(n^2)$ where the empirical error of one-pass SGD trained over $T=n$ training examples is $\Omega(\eta\sqrt n + 1/\eta n)$.
    Assuming the standard setting of $\eta = \Theta(1/\sqrt{n})$, chosen for optimal test performance, the empirical error lower bound becomes $\Omega(1)$, showing that the ``benign underfitting'' phenomena of one-pass SGD is exhibited already in dimension polynomial in the number of training samples.
    Rephrasing this lower bound in terms of the number of training examples required to reach nontrivial empirical risk, we again obtain an $\widetilde\Omega (\sqrt{d})$ sample complexity lower bound.
    See \cref{SGDmain} for the formal statement and further implications.
\end{enumerate}

Both of the results above are tight (up to logarithmic factors) in view of existing matching upper bounds of \cite{bassily2020stability}.
We remark that the constructions leading to the results feature \emph{differentiable} Lipschitz and convex loss functions, whereas the lower bounds in previous works concerned with gradient methods~\citep{amir2021sgd,amir2021never,benignunderfit} crucially applied only to the class of non-differentiable loss functions. From the perspective of general non-smooth convex optimization, this feature of the results imply that our lower bounds remain valid under \emph{any} choice of a subgradient oracle of the loss function (as opposed to only claiming that there \emph{exists} a subgradient oracle under which they apply, like prior results do).

\subsection{Main ideas and techniques}

Our work builds primarily on two basic ideas. 
The first is due to \citet{feldman2016generalization}, whereby an exponential number (in $n$) of approximately orthogonal directions, that represent the potential candidates for a ``bad ERM,'' are embedded in a $\Theta(n)$-dimensional space.
The second idea, underlying \citet{bassily2020stability,amir2021sgd,amir2021never,benignunderfit} is to augment the loss function with a highly non-smooth component, that is capable of generating large (sub-)gradients around initialization directed at all candidate directions, that could steer GD towards a bad ERM that overfits the training set.

The major challenge is in making these two components play in tandem: since the candidate directions of \citet{feldman2016generalization} are only \emph{nearly} orthogonal, the progress of GD towards one specific direction gets hampered by its movement in other, irrelevant directions.
And indeed, previous work in this context fell short of resolving this incompatibility and instead, opted for a simpler construction with a perfectly-orthogonal set of candidate directions, that was used in the earlier work of \citet{shalev2010learnability}.  
Unfortunately though, this latter construction requires the ambient dimensionality to be exponential in the number of samples $n$, which is precisely what we aim to avoid.

Our solution for overcoming this obstacle, which we describe in length in \cref{sec:sketch}, is based on several novel ideas.
Firstly, we employ multiple copies of the original construction of \citet{feldman2016generalization} in orthogonal subspaces, in a way that it suffices for GD to make a \emph{single} step within each copy so as to reach, across all copies, a bad ERM solution; this serves to circumvent the ``collisions'' between consecutive GD steps alluded to above.
Secondly, we carefully design a convex loss term that, when augmented to the loss function, forces successive gradient steps to be taken in a round-robin fashion between the different copies, so that each subspace indeed sees a single update step through the GD execution.
Lastly, we introduce a novel technique that memorizes the full training set by ``encoding'' it into the iterates in a convex and differentiable manner, so that the GD iterate itself (to which the subgradient oracle has access) contains the information required to ``decode'' the right movement direction towards a bad ERM.  We further show how all of these added loss components can be made differentiable, so as to allow for a differentiable construction overall.
A detailed overview of these construction techniques and a virtually complete description of our construction are provided in \cref{sec:sketch}.

\subsection{Additional related work}

\paragraph{Learnability and generalization in the SCO model.} 
Our work belongs to the body of literature on stability and generalization in modern statistical learning theory, 
pioneered by \citet{shalev2010learnability} and the earlier foundational work of \citet{bousquet2002stability}.
In this line of research, \citet{hardt2016train, bassily2020stability} study algorithmic stability of SGD and GD in the smooth and non-smooth (convex) cases, respectively.
In the general non-smooth case which we study here, \cite{bassily2020stability} gave an iteration complexity upper bound of $O(\eta\sqrt T + \ifrac{1}{\eta T} + \ifrac{\eta T}{n})$ test error for $T$ iterations 
with step size $\eta$ over a training set of size $n$.
The more recent work of \citet{amir2021sgd} showed this to be tight up to log-factors in the dimension independent regime, and \citet{amir2021never} further extends this result to any optimization algorithm making use of only batch gradients (i.e., gradients of the empirical risk). Even more recently, \citet{kale2021sgd} considers (multi-pass) SGD and GD in a more general SCO model where individual losses may be non-convex (but still convex on average), and prove a sample complexity lower bound for GD showing it learns in a suboptimal rate with any step size and any number of iterations.

\paragraph{Sample complexity of ERMs.}
With relation to the sample complexity of an (arbitrary) ERM in SCO, \citet{feldman2016generalization} showed that reaching $\epsilon$-test error requires $\Omega(\ifrac d\epsilon + \ifrac1{\epsilon^2})$ training samples, but did not establish optimality of this bound.
In a recent work, \cite{carmon2023sample} show this to be nearly tight and presents a $\widetilde O(\ifrac d\epsilon + \ifrac1{\epsilon^2})$ upper bound for any ERM, improving over the $O(\ifrac{d}{\varepsilon^2})$ upper bound that can be derived from standard covering number arguments.
Another recent work related to ours is that of \citet{magen2023initialization}, who provided another example for a setting in which learnability can be achieved without uniform convergence, showing that uniform convergence may not hold in the class of vector-valued linear (multi-class) predictors. However, the dimension of their problem instance was exponential in the number of training examples.

\paragraph{Implicit regularization and benign overfitting.}
Another relevant body of research focuses on understanding the effective generalization of over-parameterized models trained to achieve zero training error through gradient methods (see e.g., \citealp{bartlett2020benign,bartlett2021deep,belkin2021fit}). 
This phenomenon appears to challenge conventional statistical wisdom, which emphasizes the importance of balancing data fit and model complexity, and motivated the study of implicit regularization (or bias) as a notion for explaining generalization in the over-parameterized regimes.
Our findings in this paper could be viewed as an indication that, at least in SCO, generalization does not stem from some form of an implicit bias or regularization;
see \citet{amir2021sgd,benignunderfit} for a more detailed discussion.

\section{Problem setup and main results}
\label{sec:setup}

We consider the standard setting of Stochastic Convex Optimization (SCO).
The problem is characterized by a population distribution $\mathcal{D}$ over an instance set $Z$, and loss function $f : W \times Z \rightarrow \mathbb{R}$ defined over convex domain $W \subseteq \mathbb{R}^d$ in $d$-dimensional Euclidean space.
We assume that, for any fixed instance $z \in Z$, the function $f(w, z)$ is both convex and $L$-Lipschitz with respect to its first argument $w$. 
In this setting, the learner is interested in minimizing the \emph{population loss} (or \emph{risk}) which corresponds to the expected value of the loss function over $\mathcal{D}$, defined as
\[
    F(w) = \mathbb{E}_{z \sim \cD}[f(w, z)],
    \tag{population risk/loss}
\]
namely, finding a model $w \in W$ that achieves an $\varepsilon$-optimal population loss, namely such that $\popf(w) \leq \popf(\opt) + \varepsilon,$ where $\opt \in \argmin_{w\in W} F(w)$ is a population minimizer.

To find such a model $w$, the learner uses a set of $n$ training examples $S = \{z_1, \ldots, z_n\}$, drawn i.i.d.\ from the unknown distribution $\D$. 
Given the sample $S$, the corresponding \emph{empirical loss} (or \emph{risk}), denoted $\widehat{F}(w)$, is defined as the average loss over samples in $S$:
\[
    \widehat{F}(w) = \frac{1}{n}\sum_{i=1}^n f(w, z_i).
    \tag{empirical risk/loss}
\]
We let $\erm\in \arg\min_{w\in W} \empf(w)$ denote a minimizer of the empirical risk, refered to as an empirical risk minimizer (ERM).
Moreover, for every $w\in W$, we define the \emph{generalization gap} at $w$ as the absolute value of the difference between the population loss and the empirical loss, i.e., $\abs{ F(w)-\widehat{F}(w) }$.

\paragraph{Optimization algorithms.}

We consider several canonical first-order optimization algorithms in the context of SCO.  First-order algorithms make use of a (deterministic) subgradient oracle that takes as input a pair $(w,z)$ and returns a subgradient $g(w,z) \in \partial_w f(w,z)$ of the convex loss function $f(w,z)$ with respect to $w$.
If $\abs{\partial_w f(w,z)} = 1$, the loss $f(\cdot,z)$ is differentiable at $w$ and the subgradient oracle simply returns the gradient at $w$; otherwise, the subgradient oracle is allowed to emit any subgradient in the subdifferential set $\partial_w f(w,z)$.

First, we consider standard gradient descent (GD) with a fixed step size $\eta>0$
applied to the empirical risk $\wh{F}$.
We allow for a potentially projected, $m$-suffix averaged version of the algorithm that takes the following form:
\begin{equation} \label{gd_update_rule}  
\begin{alignedat}{2}
    &\text{initialize at} &
    &w_1 \in W ;
    \\
    &\text{update} &
    \qquad
    &w_{t+1} 
    = 
    \Pi_W\sbr[3]{ w_t - \frac{\eta}{n}\sum_{i=1}^n g(w_t,z_i) },
    \qquad \forall ~ 1 \leq t < T ;
    \\
    &\text{return} &
    &\suff \eqq \frac{1}{\indsuff}\sum_{i=1}^\indsuff w_{T-i+1}
    .
\end{alignedat}
\end{equation}
Here $\Pi_W:\R^d\to W$ denotes the Euclidean projection onto the set $W$; when $W$ is the entire space $\R^d$, this becomes simply unprojected GD.
The algorithm returns either the final iterate, the average of the iterates, or more generally, any $\indsuff$-suffix average ($1\leq \indsuff\leq T$) of iterates.

The second method that we analyze is Stochastic Gradient Descent (SGD), which is again potentially projected and/or suffix averaged.  
This method uses a fixed     stepsize $\eta>0$ and takes the following form:
\begin{equation} \label{sgd_update_rule}
\begin{alignedat}{2}
    &\text{initialize at} &
    &w_1 \in W ;
    \\
    &\text{update} &
    \qquad
    &w_{t+1} =  \Pi_W\sbr[1]{ w_t - \eta g(w_t,z_t) },
    \qquad \forall ~ 1 \leq t < T ;
    \\
    &\text{return} &
    &\suff \eqq \frac{1}{\indsuff}\sum_{i=1}^\indsuff w_{T-i+1}
    .
\end{alignedat}
\end{equation}

\paragraph{Main results.}

Our main contributions in the context of SCO are tight lower bounds for the population loss of GD and for the empirical loss of SGD, where the problem dimension is polynomial in the number of samples $n$ and steps $T$.
First, for the population risk performance of GD, we prove the following:

\begin{theorem}\label{GDmain}
Fix $n>0$, $T>3200^2$ and $0\leq \eta\leq \frac{1}{5\sqrt{T}}$ and
let $d=178nT+2n^2+\max\set{1,25\eta^2T^2}$. There exists a distribution $\D$ over instance set $Z$ and a convex, differentiable and $1$-Lipschitz loss function $f:\R^d \times Z\to \R$ such that for
GD (either projected or unprojected; cf.~\cref{gd_update_rule}
with $W=\unitballd$ or $W=\R^d$ respectively) initialized at $w_1=0$ with step size $\eta$, 
    for all $m=1,\ldots,T$, the $m$-suffix averaged iterate has, with probability at least $\frac16$ over the choice of the training sample,
    \begin{align}
    \label{eq_lower_bound_GD}
        F(\suff)
        - F(\opt) 
        = 
        \Omega\br{\min\cbr{ \eta\sqrt{T}+\ifrac{1}{\eta T}, 1 }}
        .
        \end{align}
\end{theorem}

For SGD, we prove the following theorem concerning its convergence on the empirical risk: %

\begin{theorem} \label{SGDmain}
Fix $n>2048$ and $0\leq \eta\leq \frac{1}{5\sqrt{n}}$ and let $d=712n\log n+2n^2+\max\set{1,25\eta^2n^2}$.
There exists a distribution $\D$ over instance set $Z$ and a convex, $1$-Lipschitz and differentiable loss function $f:\R^d \times Z\to \R$ such that for
one-pass SGD (either projected or unprojected; cf.~\cref{sgd_update_rule}
with $W=\unitballd$ or $W=\R^d$ respectively) over $T=n$ steps initialized at $w_1=0$ with step size $\eta$, for all $m=1,\ldots,T$, the $m$-suffix averaged iterate has, with probability at least $\frac12$ over the choice of the training sample,
\begin{align}
\label{eq_lower_bound_sgd}
    \empf(\suff) - \empf(\erm)
    = 
     \Omega\br{\min\cbr{ \eta\sqrt{T}+\ifrac{1}{\eta T}, 1 }}
    .
\end{align}
\end{theorem}

\paragraph{Discussion.}

As noted in the introduction, both of the bounds above are tight up to logarithmic factors in view of matching upper bounds due to \cite{bassily2020stability}.
For GD tuned for optimal convergence on the empirical risk, where $T=n$ and $\eta = \Theta(1/\sqrt{n})$, \cref{GDmain} gives an $\Omega(1)$ lower bound for the population error, which precludes any sample complexity upper bound for this algorithm of the form $O(d^p/\epsilon^q)$ unless $p \geq \tfrac12$. 
In particular, this implies an $\Omega(\sqrt{d})$ lower bound the number of training examples required for GD to reach a nontrivial population risk.
In contrast, lower bounds in previous work~\citep{amir2021sgd} only implies an exponentially weaker $\Omega(\log{d})$ dimension dependence in the sample complexity.
We note however that there is still a small polynomial gap between our sample complexity lower bounds to the known (nearly tight) bounds for generic ERMs~\citep{feldman2016generalization,carmon2023sample}; we leave narrowing this gap as an open problem for future investigation.

More generally, with GD fixed to perform $T = n^{\alpha}, \alpha>0$ steps, and setting $\eta$ so as to optimize the lower bound, the right-hand side in \cref{eq_lower_bound_GD} becomes $\Theta(n^{-\ifrac{\alpha}{4}})$, which rules out any sample complexity upper bound of the form $O(d^p/\epsilon^q)$ unless it satisfies $\max\{2,\alpha+1\} p + \tfrac14 \alpha q \geq 1$.%
\footnote{To see this, let $r=\max\{2,\alpha+1\}$ and note that for our construction $d = O(nT+n^2)=O(n^r)$ and $\epsilon = \Omega(n^{-\ifrac{\alpha}{4}})$; the sample complexity upper bound $O(d^p/\epsilon^q)$ can be therefore rewritten in terms of $n$ as $O(n^{rp + \alpha q/4})$, and since this should asymptotically upper bound the number of samples $n$, one must have that $rp + \tfrac14 \alpha q \geq 1$.}
Specifically, we see that any dimension-free upper bound with $T=n$ must have at least an $1/\epsilon^4$ dependence on $\epsilon$; and that for matching the statistically optimal sample complexity rate of $1/\epsilon^2$, one must either run GD for $T=n^2$ steps or suffer a polynomial dimension dependence in the rate (e.g., for $T=n$ this dependence is at least $d^{1/4}$).

Similar lower bounds (up to a logarithmic factor) are obtained for SGD through \cref{SGDmain}, but for the empirical risk of the algorithm when tuned for optimal performance on the population risk with $T=n$. 
In this case, the bounds provide an exponential improvement in the dimension dependence over the recent results of \citet{benignunderfit}, showing that the ``benign underfitting'' phenomena they revealed for one-pass SGD is exhibited already in dimension polynomial in the number of training samples.

Finally, we remark that our restriction on $\eta$ is only meant for placing focus on the more common and interesting range of stepsizes in the context of stochastic optimization.  It is not hard to extend the result of \cref{GDmain,SGDmain} to larger values of $\eta$ (in this case the lower bounds are $\Omega(1)$, the same rate the theorems give for $\eta=\Theta(1/\sqrt{T})$), in the same way this is done in previous work (e.g.,~\citealp{amir2021sgd,benignunderfit}).

\section{Overview of constructions and proof ideas}
\label{sec:sketch}

In this section we outline the main ideas leading to our main results and give an overview of the lower bound constructions.
As discussed above, the main technical contribution of this paper is in establishing the first $\Omega(\eta\sqrt{T})$ term in \cref{eq_lower_bound_GD,eq_lower_bound_sgd} using a loss function in dimension polynomial in $n$ and $T$, and this is also the focus of our presentation in this section.  
In \cref{prelim,sec:make_unif_reach,setup_encode,sec_sktch_stab} we focus on GD and describe the main ideas and technical steps towards proving our first main result; in \cref{SGD_sketch} we survey the additional steps and adjustments needed to obtain our second main result concerning SGD.

Starting with GD and \cref{GDmain}, recall that our goal is to establish a learning scenario where GD is likely to converge to a ``bad ERM'', namely a minimizer of the empirical risk whose population loss is large.  We will do that in four steps: we will first establish that such a ``bad ERM'' actually exists; then, we will show how to make such a solution reachable by gradient steps from the origin; we next describe how the information required to identify this solution can be ``memorized'' by GD into its iterates; and finally, we show how to combine these components and actually drive GD towards a bad ERM.

\subsection{A preliminary: existence of bad ERMs}%
\label{prelim}

\newcommand{\hfeld}{h_{\text{F16}}}

Our starting point is the work of \cite{feldman2016generalization} that demonstrated that in SCO, an empirical risk minimizer might fail to generalize, already in dimension linear in the number of training samples.
More concretely, they showed that for any sample size $n$, there exists a distribution $\D$ over convex loss functions in dimension $d=\Theta(n)$ such that, with constant probability, there exists a ``bad ERM'': one that overfits the training sample and admits a large generalization gap. 

Their approach was based on a construction of a set of unit vectors of size $2^{\Omega(n)}$, denoted $U$, that are ``nearly orthogonal'': the dot product between any two distinct $u, v \in U$ satisfies $|\langle u,v\rangle| \leq \frac{1}{8}$.%
\footnote{The original construction by \cite{feldman2016generalization} satisfied slightly different conditions, which we adjust here for our analysis.}
Then, they take the power set $Z = P(U)$ of $U$ as the sample space (namely, identifying samples with subsets of $U$), the distribution $\D$ to be uniform over $Z$, and the (convex, Lipschitz) loss $\hfeld: \R^{\Theta(n)} \times Z \to \R$ to be defined as follows: 
\begin{equation}\label{feldman_func}
    \hfeld(w,V)
    =
    \max\cbr[2]{ \tfrac{1}{2},\max_{u\in V} \langle u,w\rangle }.
\end{equation}
For this problem instance, they show that with constant probability over the choice of a sample $S=\{V_1,\ldots,V_n\} \iid \D^n$ of size $n=O(d)$, at least one of the vectors in $U$, say $u_0 \in U$, will not be observed in any of the sets $V_1,\ldots,V_n$, namely $u_0 \in U \setminus \bigcup_{i=1}^n V_i$.
Finally, they prove that such a vector $u_0$ is in fact an $\Omega(1)$-bad ERM (for which the generalization gap is $\Omega(1)$). 

To see why this is the case, note that, 
since every vector $u\in U$ is in every training example $V_i$ with probability $\frac{1}{2}$, the set $U$ (whose size is exponential in $n$) is large enough to guarantee the existence of a vector $u_0 \notin \bigcup_{i=1}^n V_i$ with constant probability.
Consequently, the empirical loss of such $u_0$ equals $\frac12$ (since $\abr{u_0,v} \leq \tfrac18$ for any $v \in V_i$ and therefore $\hfeld(u_0,V_i)=\tfrac12$ for all $i$). 
However, for a fresh example $V \sim \D$, with probability $\tfrac12$ it holds that $u_0\in V$, and therefore $\hfeld(u_0,V)=1$, in which case the population risk of $u_0$ is at least $=\frac{1}{2}\cdot\frac{1}{2}+\frac{1}{2}\cdot \frac{1}{2}=\frac{3}{4}$ and the generalization gap is therefore at least $\tfrac14$.

\subsection{Ensuring that bad ERMs are reachable by GD}
\label{sec:make_unif_reach}

As \cite{feldman2016generalization} explains in their work, although there exists an ERM with a large generalization gap, it is not guaranteed that such a minimizer is at all reachable by gradient methods, within a reasonable (say, polynomial in $n$) number of steps.
This is because in their construction, the loss function $\hfeld$ remains flat (and equals $\tfrac12$, see \cref{feldman_func}) inside a ball of radius $\Omega(1)$ around the origin, where GD is initialized; within this ball, all models are essentially ``good ERMs'' that admit zero generalization gap.
It remains unclear how to steer GD, with stepsize of order $\eta = O(1/\sqrt{T})$ over $T$ steps, away from this flat region of the loss towards a bad ERM, such as the $u_0$ identified above.

To address this challenge, we modify the construction of \cite{feldman2016generalization} in a fundamental way.  The key idea is 
increase dimensionality and replicate \citeauthor{feldman2016generalization}'s construction in $T$ orthogonal subspaces; this would allow us to decrease, in each of the subspaces, the distance to a bad ERM to only $O(\eta)$, rather than $\Theta(1)$ as before. 
Then, while each of these subspace ERMs is only $\Omega(\eta)$-bad, when taken together, they constitute an $\Omega(1)$-bad ERM in the lifted space whose distance from the origin is roughly $\eta \sqrt{T} = \Theta(1)$, yet is still reachable by $T$ steps of GD.

More concretely, 
we introduce a loss function $h:\R^{d'}\times P(U)\to\R$ (for $d'=\Theta(n)$) that resembles \citeauthor{feldman2016generalization}'s function from \cref{feldman_func} up to a minor adjustment: %
\begin{equation}\label{our_feldman}
    h(w',V)
    =
    \max\cbr[2]{ \tfrac{3}{32}\eta , \max_{u\in V} \abr[1]{u , w'} }
    .
\end{equation}
As in the original construction by \citeauthor{feldman2016generalization}, $V$ here ranges over subsets of a set $U \subseteq \R^{d'}$ of size $2^{\Omega(d')}$, the elements of which are nearly-orthogonal unit vectors.
Then, we construct a loss function in dimension $d = Td'$ by applying $h$ in $T$ orthogonal subspaces of dimension $d'$, denoted $\vecpart{W}{1},\ldots,\vecpart{W}{T}$, as follows:%
\footnote{The summation starts at $\indsec=2$ due to technical reasons that will become apparent later in this proof sketch.}
\begin{equation}\label{uniform_convergence_objective}
    \gdindgam_1(w,V) 
    = 
    \sqrt{\sum_{\indsec=2}^{T} \br[2]{ h(\vecpart{w}{\indsec},V) }^2 }
    ~.
\end{equation}
Here and throughout, $\vecpart{w}{\indsec}$ refers to the $\indsec$'th orthogonal component of the vector $w$, that resides in the subspace $W^{(\indsec)}$.
Finally, the distribution $\D$ is again taken to be uniform over $Z = P(U)$, and a training set is formed by sampling $S=\{V_1,\ldots,V_n\}\sim \cD^n$.  As before, we know that with probability at least $\frac12$, there exists a vector $u_0$ such that $u_0 \in U \setminus \bigcup_{i=1}^n V_i$. 

With this setup, it can be shown that $\gdindgam_1$ is indeed convex and $O(1)$-Lipschitz, and further, that any vector $w$ satisfying $\vecpart{w}{\indsec} = c\eta u_0$ for a sufficiently large constant $c>0$ and $\Omega(T)$-many components $\indsec$, is an $\Omega(1)$-bad ERM with respect to $\gdindgam_1$.
The important point is that, unlike in \citeauthor{feldman2016generalization}'s original construction, such bad ERMs are potentially reachable by GD: it is sufficient to guide the algorithm to make a single, small step (with stepsize $\eta$) towards $u_0$ in each subspace $\vecpart{W}{\indsec}$.

\subsection{Memorizing the dataset in the iterate}
\label{setup_encode}
There is one notable obstacle to the plan we just described: 
the vector $u_0$ is determined in a rather complex way by the full description of the training set and it is unclear how to reproduce such a vector through subgradients of the loss function.  Indeed, recall that the only access GD has to the training set is through subgradients of individual functions $g(w, V_1), \ldots, g(w, V_n)$ (and linear combinations thereof), and none of these has direct access to the full training set that could allow for determining a vector $u_0 \in U \setminus \bigcup_{i=1}^n V_i$.

To circumvent this difficulty, another key aspect of our construction involves a mechanism that effectively \emph{memorizes the full training set in the iterate $w$ itself}, using the first few steps of GD. 
For this memorization, we can, for example, further increase the dimension of the domain $W$ and create an ``encoding subspace,'' denoted as $\vecpart{W}{0}$ (and the corresponding component of a vector $w \in W$ is indicated by $\vecpart{w}{\last}$), which is orthogonal to $\vecpart{W}{1}, \ldots, \vecpart{W}{T}$.
In this subspace, each step taken with respect to (a linear combination of) individual gradients $g(w_t, V_i)$ encodes the information on the sets $V_1,\ldots,V_n$ into the iterate $w_t$. Then, since the subgradient oracle receives $w_t$ as input, it can reconstruct the training set encoded in $\vecpart{w_t}{0}$ and recover $u_0$, in every subsequent step.
 
On its own, the task of memorizing the training set is not particularly challenging and can be addressed in a rather straightforward manner.%
\footnote{One simple approach is to utilize a one-dimensional encoding space
to encode every individual training example $V \in P(U)$ in the least significant bits of $\vecpart{w_t}{0}$, in a way that guarantees there are no collisions between different possible values of such sets.
This allows the subgradient oracle to calculate the specific $u_0$ from the training set encoded in the least significant bits of $\vecpart{w_t}{0}$ and return it as a subgradient.}
What turns out to be more challenging is to design the encoding in such a way that $u_0$ is realized as the \emph{unique} subgradient (i.e., the gradient) of the loss function.  This would be crucial for establishing that our lower bound is valid for \emph{any} subgradient oracle, and not only for an adversarially chosen one (as well for making the construction differentiable; we discuss this later on, in \cref{togetheranddiff}).

Let us describe an encoding mechanism where $u_0$ acts as the unique subgradient at $w_1=0$.
We will employ an encoding subspace $\vecpart{W}{0}$ of dimension $\Theta(n^2)$, and augment samples with a number $j \in [n^2]$, drawn uniformly at random; namely, each sample in the training set is now a pair $(V_i, j_i) \in P(U) \times [n^2]$, for $i=1,\ldots,n$.
We then create an encoding function $\phi : P(U)\times[n^2]\to \vecpart{W}{0}$ such that $\phi(V,j)$ maps the set $V$ into the $j$'th ($2$-dimensional) subspace of the encoding space. 
The role of $j$ is to ensure that, with constant probability, sets in the training sample are mapped to distinct subspaces of the encoding space $\vecpart{W}{0}$, and thus can be uniquely inferred given the encoding.
To implement the encoding within the optimization process, we introduce the following term into the loss:
\begin{equation}
\label{endcoding_term_gd}
    \gdindgam_2(w,(V,j))
    \eqq 
    \langle -\phi(V,j),\wenc \rangle 
    .
\end{equation}
Following a single step of GD, the iterate becomes $\vecpart{w_2}{0}=\frac{\eta}{n}\sum_{i=1}^n \phi(V_i,j_i)$, and by the properties of the encoding $\phi$ it is then possible, with constant probability, to fully recover the sets $V_1, \ldots, V_n$ in the training set, given the iterate $w_2$ . 

Next, 
we introduce an additional term into the loss function, whose role is to ``decode'' the training set $V_1, \ldots, V_n$ from $w_t$ and produce a vector $u_0\in U\setminus\bigcup_{i=1}^nV_i$ as a subgradient.
For this, we represent every potential training set using a vector $\psi\in \Psi\subseteq \mathbb B^{2n^2}$, and define a mapping $\alpha: \R^{2n^2}\to U$ that, for every $\psi\in\Psi$, provides a vector $\alpha(\psi) \in U$ that does not appear in any of the sets $V_i$ in the training sample associated with $\psi$ (if such a vector exists).
Finally, we add the following term to the loss function, 
\begin{equation} \label{decoding_term_gd}
    \gdindgam_3(w)
    \eqq 
    \max\cbr{ \delta_1,\max_{\psi\in \Psi} \cbr{ \langle \psi,\wenc \rangle - \beta\langle \alpha(\psi), \vecpart{w}{1}\rangle } }
    ,
\end{equation}
where $\beta,\delta_1>0$ are small predefined constants. %
We can show, assuming that in the first step, the training set was encoded to the iterate ($\vecpart{w_2}{0}=\frac{\eta}{n}\sum_{i=1}^n \phi(V_i,j_i)$), that for a suitable choice of the encoder ($\phi$) and decoder ($\psi$ and $\alpha$), in the following iteration, the vector 
$\psi^*\in\Psi$ that represents the actual training set $V_1, \ldots, V_n$ is realized as a unique maximizer in \cref{decoding_term_gd}, which in turn triggers a gradient step along $u_0 \eqq \alpha(\psi)$ in the subspace $W^{(1)}$.

\subsection{Making GD converge to a bad ERM}
\label{sec_sktch_stab}

Our final task is to finally make GD converge to a ``bad ERM,'' namely to a model $w$ such that $\vecpart{w}{\indsec} = c\eta u_0$  for a sufficiently large constant $c>0$ and $\Omega(T)$-many values of $\indsec$, assuming it was successfully initialized at $w$ with $\vecpart{w}{1}=c_1u_0$ (and $\vecpart{w}{\indsec}=0$ for $\indsec>1$) as we just detailed.
We will accomplish this by forcing GD into making a single step towards $u_0$ in $\Omega(T)$ of the subspaces $\vecpart{W}{1}, \ldots, \vecpart{W}{T}$.

To this end, we employ a variation of a technique used in previous lower bound constructions~\citep{bassily2020stability,amir2021sgd,benignunderfit} to induce gradient instability around the origin.
In these prior instances, however, the potential directions of progress---analogous to vectors in our set $U$---were perfectly orthogonal (and thus, the dimension of space was required to be exponential in $n$).
In contrast, in our scenario the vectors in $U$ are only approximately orthogonal, and directly applying this approach could lead to situations where gradient steps from consecutive iterations may interfere with progress made in correlated directions in previous iterations.

To address this, we introduce a careful variation on this technique, based on augmenting the loss function with the following convex term:
\begin{equation} \label{stability_term_GD}
    \gdindgam_4(w)
    =
    \max\cbr{ \delta_2,\max_{u \in U,\, \indsec < T} \cbr{ \tfrac{3}{8}\langle u,\vecpart{w}{\indsec}\rangle - \tfrac{1}{2}\langle u,\vecpart{w}{\indsec+1}\rangle }}
    ,
\end{equation}
where $\delta_2>0$ is a small constant (that will be set later).
The key idea here is that following the initialization stage, the inner maximization above is always attained at the same vector $u=u_0$, and for values of $\indsec$ that increase by $1$ in every iteration of GD. 
Consequently, subgradient steps with respect to this term will result in making a step towards $u_0$ in each of the components $\vecpart{w}{1}, \vecpart{w}{2}, \ldots$ one by one, avoiding interference between consecutive steps.
At the end of this process, there are $\Omega(T)$ values of $\indsec$ such that $\vecpart{w}{\indsec} = \tfrac{1}{8}\eta u_0$, which is what we set to achieve.

In some more detail, assuming GD is successfully initialized at a vector $w$ with $\vecpart{w}{1}=c_1u_0$ and $\vecpart{w}{\indsec}=0$ for $\indsec>1$ ($c_1 >0$ is a small constant), note that the maximum in \cref{stability_term_GD} is uniquely attained at $\indsec=1$ and $u=u_0$.
Consequently, the subgradient of $\gdindgam_4$ at initialization is a vector $g$ such that $\vecpart{g}{1}=\frac{3}{8}u_0$, $\vecpart{g}{2}=-\frac{1}{2}u_0$ (and $\vecpart{g}{\indsec}=0$ for $\indsec \neq 1,2$), and 
taking a subgradient step with stepsize $\eta$ 
results in
$\vecpart{w}{1} = \br[0]{\eta\beta-\frac{3\eta}{8}} u_0$ 
and $\vecpart{w}{2}=\frac{\eta}{2}u_0$ (for $\indsec \neq 1,2$, $\vecpart{w}{\indsec}$ remains as is). 
In each subsequent iteration, the maximization in \cref{stability_term_GD} is attained at an index~$\indsec$ for which $\vecpart{w}{\indsec}=\frac{\eta}{2}u_0$ and at $u=u_0$.%
\footnote{For this value of $\indsec$, it holds that $\max_{u \in U} \cbr{\tfrac{3}{8}\langle u,\vecpart{w}{\indsec}\rangle - \tfrac{1}{2}\langle u,\vecpart{w}{\indsec+1}\rangle} = \frac{3}{16}\eta$ (attained at $u=u_0$), whereas for other values of $\indsec$ this quantity is at most $\approx \frac18 \cdot \frac38 \eta + \frac18 \cdot \left(\frac12\right)^2 \eta < \frac{1}{8} \eta$ due to the near-orthogonality of vectors in $U$.
It follows that 
the subgradient is a vector $g$ such that $\vecpart{g}{\indsec}=\frac{3}{8}u_0$, $\vecpart{g}{\indsec+1}=-\frac{1}{2}u_0$ (and zeros elsewhere).} 
Subsequently, every gradient step adds $-\frac{3\eta}{8}u_0$ to $\vecpart{w}{\indsec}$ and $\frac{\eta}{2}u_0$ to $\vecpart{w}{\indsec+1}$ and results in $\vecpart{w}{2} = \vecpart{w}{3} = \ldots = \vecpart{w}{\indsec} = \frac{\eta}{8}u_0$ and $\vecpart{w}{\indsec+1}=\frac{\eta}{2}u_0$ (whereas for all $s>\indsec+1$, $\vecpart{w}{s}$ remains zero).

Finally, we note that the GD dynamics we described ensure that the iterates $w_1,\ldots,w_T$ remain strictly within the unit ball $\unitballd$, even when the algorithm does not employ any projections. As a consequence, the construction we described applies equally to a projected version of GD, with projections to the unit ball, and the resulting lower bound will apply to both versions of the algorithm.

\subsection{Putting things together}
\label{togetheranddiff}

We can now integrate the ideas described in \cref{prelim,sec:make_unif_reach,setup_encode,sec_sktch_stab} into a construction of a learning problem where GD overfits the training data (with constant probability), that would serve to prove our lower bound for gradient decent.
To summarize this construction:
\begin{itemize}
    \item The examples in the learning problem are parameterized by pairs $(V,j) \in Z \eqq P(U) \times [n^2]$, where $U$ is the set of nearly-orthogonal vectors described in \cref{prelim}, and $P(U)$ is its power set;
    \item The population distribution $\gdindD$ is uniform over pairs $(V,j) \in Z$, namely such that $V \sim \Unif(P(U))$ (i.e., $V$ is formed by including every element $u\in U$ independently with probability $\tfrac{1}{2}$) and $j \sim \Unif([n^2$]);
  \item The loss function in this construction, $\gdindf: W \times (P(U)\times [n^2])\to \R$, is then given by:
\begin{align} \label{loss_def_gd_in_setup_nondiff}
    \forall ~  (V,j)\in Z , \qquad
    \gdindf(w, (V,j))
    &\eqq
    \gdindgam_{1}(w, V) + \gdindgam_{2}(w, (V,j))+ \gdindgam_{3}(w)+ \gdindgam_{4}(w)
    ,
\end{align} 
with the terms $\gdindgam_{1},\gdindgam_{2},\gdindgam_{3},\gdindgam_{4}$ as defined in \cref{uniform_convergence_objective,endcoding_term_gd,decoding_term_gd,stability_term_GD} respectively.
\end{itemize}
With a suitable choice of parameters, this construction serves to proving \cref{GDmain}.
We remark that, while $f$ in this construction is convex and $O(1)$-Lipschitz, it is evidently non-differentiable.
For obtaining a construction with a differentiable objective that maintains the same lower bound and establish the full claim of  \cref{GDmain},
we add one final step of randomized smoothing of the objective. 
This argument hinges on the fact that the subgradients of $f$ are unique \emph{along any possible trajectory of GD}, so that smoothing in a sufficiently small neighborhood would preserve gradients along any such trajectory (and thus does not affect the the dynamics of GD), while making the objective differentiable everywhere. 
The full proof of \cref{GDmain} is deferred to \cref{sec:const_diff}.

\subsection{Additional adjustments for SGD}
\label{SGD_sketch}

Moving on to discuss our second main result for SGD, we provide here a brief overview of the necessary modifications upon the construction for GD to establish the lower bound for SGD in \cref{SGDmain}; further details can be found in \cref{sec:sgd_appendix}. In the case of SGD, our goal is to establish underfitting: namely, to show that the algorithm may converge to a solution with an excessively large empirical risk despite successfully converging on the population risk.

The main ideas leading to our construction for SGD are similar to what we discussed above, but there are several necessary modifications that arise from the fact that, whereas in GD the entire training set is revealed already in the first iteration, in SGD it is revealed sequentially, one training sample at a time.
In particular, unlike in the case of GD where it is possible to identify a bad ERM $u_0$ at the few first steps of the algorithm and steer the algorithm in this direction in every subspace $\vecpart{W}{1}, \vecpart{W}{2}, \ldots$, for SGD the required progress direction in $\vecpart{W}{t},$ represented as a ``bad solution'' $u_t$, can be only determined in the $t$'th step based on the encoded training set up to that point, $V_1, \ldots, V_{t-1}$.
As a result, it is crucial to modify the loss function such that the process of decoding such $u_t$ from $V_1, \ldots, V_{t-1}$ occurs in every iteration $t$.

Another essential adjustment involves identifying a solution with a large generalization gap (namely large empirical risk, low population risk) and guiding the SGD iterates to converge to such a solution.
Considering the function $\gdindgam_1$ defined in \cref{uniform_convergence_objective}, such a solution is represented by a vector $u\in U$ that appears in all of the sets $V_1,\ldots,V_n$ in the training sample. 
However, since $u_t$ cannot depend on future examples, our goal within every subspace $\vecpart{W}{t}$ is to take a single gradient step towards a vector $u_t$ present only in sets up to that point, namely in $\bigcap_{i=1}^{t-1} V_i$ (note that such $u_t$ maximizes the corresponding loss functions $\gdindgam_1(w,V_1)\ldots \gdindgam_1(w,V_{t-1})$). 
Additionally, to ensure that gradients for future loss functions remain zero and do not affect the algorithm's dynamics, it is necessary to to guarantee that $u_t \in \bigcap_{i=t}^n \overline{V}_i$; in other words, we are looking for a solution $u_t \in \bigcap_{i=1}^{t-1} V_i \cap \bigcap_{i=t}^n \overline{V}_i$.
For ensuring that such a vector actually exists (with constant probability), we lift the dimension of the set $U$ and the subspaces $\{\vecpart{W}{\indsec}\}_{\indsec=1}^{n}$ to $d=\Theta(n\log n)$ (instead of $\Theta(n)$ as before) and modify the distribution $\gdindD$ so as to have that $V$ is sampled such that every element $u\in U$ is included in $V$ independently with probability $\ifrac{1}{4n^2}$. %

With these adaptations in place, we can obtain \cref{SGDmain}; for more details we refer to \cref{sec:sgd_appendix}.

\section{Overfitting of GD: Proof of \texorpdfstring{\cref{GDmain}}{Theorem \ref{GDmain}}}
\label{sec:GD}
In this section, we provide a formal proof of our main result for GD.
We establish a lower bound of $\Omega\br[0]{\eta \sqrt{T}}$ for the population loss of GD, where the hard loss function is defined in a $d$-dimensional Euclidean space, where the dimension $d$ is polynomial in the number of examples $n$.
In \cref{sec:const_diff} we complete the proof of \cref{GDmain}, by showing a lower bound of $\min\cbr[0]{ \ifrac{1}{\eta T},1 }$, and a construction of a differentiable objective that holds the lower bound stated in \cref{GDmain}.

\paragraph{Full construction.}
For the first step, for a dimension $d'$ that will be set later, we use a set of approximately orthogonal vectors in $\R^{d'}$ with size (at least) exponential in $d'$,  the existence of which is given by the following lemma, adapted from \citet{feldman2016generalization}.

\begin{lemma}
\label{lem:set_direc_exists}
For any $d' \geq 256$, 
there exists a set $U_{d'} \subseteq \R^{d'}$, with 
$|U_{d'}| \geq 2^{d'/178}$, such that for all $u, v\in U_{d'}, u\neq v$, it holds that $|\abr{u, v}|\leq \frac{1}{8}$.
\end{lemma}

Now, let $n$ be the number of examples in the training set. 
We define the set $U\eqq U_{d'}$ to be a set as specified by \cref{lem:set_direc_exists} for $d'=178n$.
Then, as outlined in \cref{sec:sketch}, we define the sample space $Z\eqq \{(V,j): V\subseteq U, j\in[n^2]\}$
and the hard distribution $\gdindD$ as the uniform distribution. 

Moreover, we consider the loss function $\gdindf:\R^d\to \R$ (defined in \cref{loss_def_gd_in_setup_nondiff} for $d\eqq Td'+2n^2=178nT+2n^2$.
This loss function
is convex and $5$-Lipschitz over $\R^d$, as established in the following lemma:
\begin{lemma}
\label{convex_lip}
For every $(V,j) \in Z$, the loss function $\gdindf(w,(V,j))$ is convex and $5$-Lipschitz over $\R^d$ with respect to its first argument.
\end{lemma}

For this construction of distribution and loss function, we obtain the following theorem.

\begin{theorem}
\label{nonspec_lower_bound_GD}
Assume that $n>0$, $T>3200^2$ and $\eta\leq\frac{1}{\sqrt{T}}$. %
Consider the distribution $\gdindD$ and the loss function $\gdindf$ that defined in \cref{togetheranddiff} for $d=178nT+2n^2$, $\varepsilon = \frac{1}{n^2}
\br[0]{1-\cos\br[0]{\frac{2\pi}{|P(U)|}}}$, $\beta=\frac{\epsilon}{4T^2}$, $\delta_1=\frac{\eta}{2n}$ and $ \delta_2=\frac{3\eta\beta}{16}$.
Then, %
 for Unprojected GD (cf. \cref{gd_update_rule} with $W=\R^d$) on $\gdindempf$, initialized at $w_1=0$ with step size $\eta$, we have,  %
with probability at least $\frac{1}{6}$ over the choice of the training sample:
\begin{enumerate}[label=(\roman*)]
    \item 
    The iterates of GD remain within the unit ball, namely $w_t \in \unitballd$ for all $t=1,\ldots,T$;
    \item
    For all $m=1,\ldots,T$, the $m$-suffix averaged iterate has:
    \[ 
        F(\suff)
        - F(\opt) 
        = 
        \Omega\br[1]{\eta\sqrt{T}}
        .
    \]
\end{enumerate}
\end{theorem}

\paragraph{Algorithm's dynamics.}

We next give a key lemma that characterizes the trajectory of GD when applied to the empirical risk $\gdindempf$ formed by the loss function $f$ and the training sample $S=\{(V_i,j_i)\}_{i=1}^n$.
The characterization holds under a certain ``good event'', given as follows:
\begin{align}
\label{good_event_gd}
    \cE \eqq \cbr[1]{
        \bigcup\nolimits_{i=1}^n V_i\neq U
    } \cap \cbr[1]{
        j_k\neq j_l,\;
        \forall k\neq l
    }.
\end{align}
In words, under the event $\cE$ there exists at least one  ``bad direction'' (which is a vector in the set $U \setminus \bigcup\nolimits_{i=1}^n V_i$) and there is no collision between the indices $j_1,\ldots,j_n$. In the following lemma we show that $\cE$ holds with a constant probability. The proof is deferred to \cref{sec:proofs_alg_GD}.
\begin{lemma} 
\label{intersection_events_gd}
For the event $\cE$ defined in \cref{good_event_gd}, it holds that $\Pr(\cE) \geq \frac{1}{6}$.
\end{lemma}

Under this event, the dynamics of GD are characterized as follows.

\begin{lemma}
\label{diff_GD_expression}
Assume the conditions of \cref{nonspec_lower_bound_GD}, and consider the iterates of unprojected GD on $\gdindempf$, with step size $\eta \leq \ifrac{1}{\sqrt{T}}$ initialized at $w_1=0$. 
Under the event $\cE$, we have for all $t\geq 5$ that
\begin{align}
\label{eq:gd_exp}
    \vecpart{w_t}{\indsec} = 
    \begin{cases}
        \frac{\eta}{n}\sum_{i=1}^n\phi(V_i,j_i) 
        &\quad \indsec=\text{\last} ;
        \\
        \br[1]{-\tfrac{3}{8} + \frac{t-2}{4}\frac{\epsilon}{T^2} } \eta u_{0}
        &\quad \indsec=1 ;
        \\
        \frac{1}{8}\eta u_{0} 
        &\quad 2\leq \indsec\leq t-3 ;
        \\
        \frac{1}{2}\eta u_{0} 
        &\quad \indsec=t-2 ;
        \\
        0 
        &\quad t-1\leq \indsec\leq T ,
    \end{cases}
\end{align}
where $u_0$ is a vector such that $u_0\in U \setminus \bigcup_{i=1}^n V_i$.
\end{lemma}
To prove \cref{diff_GD_expression}, we break down to the different components of the loss and analyze how the terms $\gdindgam_1,\gdindgam_3$ and $\gdindgam_4$ affects the dynamics of $GD$ under the event $\cE$. For each of these components, which involve maximum over linear functions, we show which term achieves the maximum value for each $w_t$ and derive the expressions for the gradients at those points by the maximizing terms.
First, we show that under this event, the gradients of $\ell_1$ do not affect the dynamics since in any iteration $t$ the gradient of $\ell_1$ is zero, as stated in the following lemma. The proof is deferred to \cref{sec:GD_proofs}.
\begin{lemma}
\label{gradient_uc}
    Assume the conditions of \cref{nonspec_lower_bound_GD} and the event $\cE$. 
    Let $w\in \R^d$ be such that for every $2\leq k\leq T$, $\vecpart{w}{k}=c\eta u_0$ for $c\leq \frac{1}{2}$ and $u_0\in U\setminus \bigcup_{i=1}^n V_i$. Then, for every $i$, it holds that %
    \begin{enumerate}[label=(\roman*)]
    \item for every $k\geq2$, it holds $\max_{u\in V_i}  \langle \vecpart{w}{k}, u_0 \rangle\leq \frac{\eta}{16};$
    \item $\gdindgam_1$ is differentiable at $w$ and for all $i \in [n]$, we have
    $\nabla\gdindgam_1(w,V_i)=0.$
    \end{enumerate}
\end{lemma}

Next, for the term $\gdindgam_3$, as outlined in \cref{setup_encode},
it is used for identifying the actual training set $S=\{(V_i,j_i)\}_{i=1}^n$ given an encoding $\psi^*=\frac{1}{n}\sum_{i=1}^n\phi(V_i,j_i)$
in the iterate $\vecpart{w_t}{\last}$ and ensuring a performance of gradient step in $\vecpart{W}{1}$ towards a corresponding vector $u_0\in U\setminus \bigcup_{i=1}^n V_i$ in the following iteration. It is done by getting $\psi^*$ as a maximum of linear functions (with positive constant margin) over the set $\Psi$ which contains all possible encoded datasets.
This idea is formalized in the following lemma.
\begin{lemma}
\label{gradient_decode}
    Assume the conditions of \cref{nonspec_lower_bound_GD} and the event $\cE$. 
    Let $\psi^*=\frac{1}{n}\sum_{i=1}^n\phi(V_i,j_i)$ and $w \in \R^d$ be such $\wenc=\eta\psi^*$, 
    and let $\vecpart{w}{1}=c\eta u_0$ for $|c|\leq 1$ and $u_0\in U\setminus \bigcup_{i=1}^n V_i$.  Then
    \begin{enumerate}[label=(\roman*)]
    \item  For every $\psi\in\Psi$, $\psi\neq \psi^*$:
    \begin{align*}
    \langle\wenc, \psi^*\rangle - \tfrac{\epsilon}{4T^2}\langle \alpha(\psi^*), \vecpart{w}{1}\rangle
    >  \langle\wenc, \psi\rangle - \tfrac{\epsilon}{4T^2}\langle \alpha(\psi), \vecpart{w}{1}\rangle + \tfrac{\eta\epsilon}{4}
    ;
    \end{align*}
\item For $\psi = \psi^*$, it holds that
\begin{align*}
   \langle\wenc, \psi^*\rangle - \tfrac{\epsilon}{4T^2}\langle \alpha(\psi^*), \vecpart{w}{1}\rangle
    > \delta_1 + \tfrac{\eta}{16n};
\end{align*}
    \item $\gdindgam_3$ is differentiable at $w$ and the gradient is given as follows:
    \begin{align*}
    \vecpart{(\nabla \gdindgam_{3}(w))}{\indsec}=\begin{cases}
        \psi^* &\quad \indsec=\last;\\
        - \tfrac{\epsilon}{4T^2}u_0 &\quad \indsec=1;\\
        0 &\quad \text{otherwise}.
    \end{cases}
\end{align*}
\end{enumerate}
\end{lemma}

Finally, for $\gdindgam_4$, as detailed in \cref{sec_sktch_stab}, the role of this term is to make the last iterate $w_T$ hold $\vecpart{w_T}{k}=\frac{\eta}{8}\eta u_0$ for $\Omega(T)$ many sub-spaces $\vecpart{W}{k}$. In the following lemma,
we show that in every iteration $t$, every gradient step increases the amount of such $k$s by $1$, namely, in every iteration $t$, the maximum of $\gdindgam_4$ is attained at $u=u_0$ and index $k_t=\arg\max\{k: \vecpart{w_t}{k}\neq 0\}$, which increases by 1 in every iteration, making the $\vecpart{w_{t+1}}{k_t}=\frac{\eta}{8}\eta u_0$.
\begin{lemma}
\label{gradient_stab}
Assume the conditions of \cref{nonspec_lower_bound_GD} and the event $\cE$. 
Let $w\in \R^d$, $u_0\in U\setminus \bigcup_{i=1}^n V_i$ and $3\leq m < T$ be such that $\vecpart{w}{1}=c\eta u_0$ for $-\frac{3}{8}\leq c\leq 0$, $\vecpart{w}{k}=\frac{\eta}{8}u_0$ for every $2\leq k\leq m-1$, $\vecpart{w}{k}=\frac{\eta}{2}u_0$ and $\vecpart{w}{k}=0$ for every $k\geq m$. Then, it holds that,
\begin{enumerate}[label=(\roman*)]
    \item  For every pair $u\in U$ and $k<T$ such that $k\neq m$ or $u\neq u_0$, 
    \begin{align*}
    \tfrac{3}{8}\langle u_0,\vecpart{w}{m}\rangle - \tfrac{1}{2}\langle u_0,\vecpart{w}{m+1}\rangle 
    >  \tfrac{3}{8}\langle u,\vecpart{w}{\indsec}\rangle - \tfrac{1}{2}\langle u,\vecpart{w}{\indsec+1}\rangle + \tfrac{\eta}{64}
\end{align*}
\item \begin{align*}
   \tfrac{3}{8}\langle u_0,\vecpart{w}{m}\rangle - \tfrac{1}{2}\langle u_0,\vecpart{w}{m+1}\rangle 
    > \delta_2+\tfrac{\eta}{64}.
\end{align*}
    \item $\gdindgam_4$ is differentiable at $w$ and the gradient is given as follows:\begin{align*}
  \vecpart{\br[1]{\nabla \gdindgam_{4}(w)}}{\indsec}=\begin{cases}
        \frac{3}{8}u_0 &\quad \indsec=m;\\
        -\frac{1}{2}u_0 &\quad \indsec=m+1;\\    
        0 &\quad \text{otherwise}.
    \end{cases}
\end{align*}
\end{enumerate}
\end{lemma}
\begin{proof} [of \cref{diff_GD_expression}]
We prove the lemma by induction on $t$; the base case, for $t=5$, is proved in \cref{diff_GD_w_5} in \cref{sec:GD_proofs} and here we focus on the induction step.  For this, fix any $t \geq 5$ and assume the that the lemma holds for $w_t$; we will prove the claim for $w_{t+1}$.

First, for $\gdindgam_1$, note that, by the hypothesis of the induction, for every $2\leq \indsec \leq T$, $\vecpart{w_t}{\indsec}=c\eta u_0$ for $c\leq\frac{1}{2}$, thus, by \cref{gradient_uc}, for every $i$,
$\nabla \gdindgam_{1}(w_t,V_i)=0$.

For $\gdindgam_2$, we know that, for every $i$,
\begin{align*}
   \vecpart{\br[1]{\nabla \gdindgam_{2}(w_t,(V_i,j_i))}}{\indsec}=\begin{cases}
        -
    \phi(V_i,j_{i}) &\quad \indsec=\last;\\
        0 &\quad \text{otherwise}.
    \end{cases}
\end{align*}
For $\gdindgam_{3}$, using the hypothesis of the induction, which implies that $\vecpart{w_t}{1}=c\eta u_0$ for $|c|\leq1$ and $\wenct{t}=\frac{\eta}{n}\sum_{i=1}^n\phi(V_i,j_i)$, by \cref{gradient_decode},  we get that, \begin{align*}
   \vecpart{\br[1]{\nabla \gdindgam_{3}(w_t)}}{\indsec}=\begin{cases}
        \frac{1}{n}\sum_{i=1}^n \phi(V_i,j_{i}) &\quad \indsec=\last;\\
        -\frac{\epsilon}{4T^2}u_0 &\quad \indsec=1;\\    
        0 &\quad \text{otherwise}.
    \end{cases}
\end{align*}
For $\ell_4$, by the hypothesis of the induction we know that
$\vecpart{w_t}{1}=\br[1]{-\frac{3}{8}+\frac{(t-2)}{4}\frac{\epsilon}{T^2}}\eta u_0$, thus, $\vecpart{w_t}{1}=c\eta u_0$ for $-\frac{3}{8}c\leq 0$.
Then the conditions of \cref{gradient_stab} hold for $m=t-2$, thus, it holds that,
\begin{align*}
    \vecpart{\br[1]{\nabla \gdindgam_4(w_{t})}}{\indsec}
    =
    \begin{cases}
         \frac{3}{8}u_0 &\quad \indsec=t-2;\\
          -\frac{1}{2}u_0 &\quad \indsec=t-1;\\
        0 &\quad \text{otherwise}
        .\end{cases}
\end{align*}
Combining all together, we get that,
\begin{align*}
    \vecpart{\br[1]{\nabla \gdindempf(w_t)}}{\indsec}=\begin{cases}
        -\frac{\epsilon}{4T^2} u_{0} &\quad \indsec=1;\\
        \frac{3}{8} u_{0} &\quad \indsec=t-2;
        \\
        -\frac{1}{2} u_{0} &\quad \indsec=t-1;
        \\
        0 &\quad \text{otherwise},
    \end{cases}
\end{align*}
where $u_0\in U \setminus \bigcup_{i=1}^n V_i$, and the lemma follows.
\end{proof}

\paragraph{Proof of Lower Bound.}
Now we can turn to prove \cref{nonspec_lower_bound_GD}. 
Here we prove the lower bound for the case of suffix averaging with $\indsuff=1$, namely, when the output solution is the final iterate $w_T$ of GD; the full proof for the more general case can be found in \cref{sec:proof_nondiff_GD}.

\begin{proof}[of \cref{nonspec_lower_bound_GD} ($\indsuff=1$ case)] %
We prove the theorem under the condition that $\cE$ occurs.
First, in \cref{bounded_norm_gd} in appendix \cref{sec:GD_proofs} we know that for every $t$, we have that $\|w_t\|\leq 1$.

Next,
$w_T$ is as in \cref{eq:gd_exp}.
Now, we notice that if a vector $v\in U$ is in a set $V\subseteq U$, it holds that $\max_{u\in V} \langle u, v\rangle=1$. However, if $v\notin V$, it holds that $\max_{u\in V} \langle u,v\rangle=\frac{1}{8}$.
As a result, by the fact that every vector for a fresh pair $(V,j)\sim D$, $u_0\in U$ is in $V$ with probability $\frac{1}{2}$, the following holds:
\begin{align*}
    \E_V\sqrt{\sum_{\indsec=2}^T\max\cbr{ \frac{3\eta}{32}, \max_{u\in V} \langle u,\vecpart{w_T}{\indsec}\rangle }^2}
    &\geq 
    \E_V\sqrt{\sum_{\indsec=2}^{T-3}\max\cbr{ \frac{3\eta}{32}, \max_{u\in V} \langle u,\vecpart{w_T}{\indsec}\rangle }^2}
    \\&=
    \E_V\sqrt{(T-4)\max\cbr{ \frac{3\eta}{32}, \max_{u\in V} \langle u, \frac{\eta}{8}u_{0}\rangle }^2}
    \\&=
    \frac{\eta\sqrt{T-4}}{8}\E_V\max\cbr{ \frac{3}{4}, \max_{u\in V} \langle u, u_{0}\rangle }
    \\&\geq
    \frac{\eta\sqrt{T-4}}{8}\left(\frac{3}{4}\Pr(u_{0}\notin V) + \Pr(u_{0}\in V)\right)
    \\&=
    \frac{7\eta}{64}\sqrt{T-4}
    .
\end{align*}
Moreover, we notice that for every $t$, $V\subseteq U$ and $j\in[n^2]$, $\gdindgam_2(w_t, (V,j))\geq-\|\wenct{t}\|\geq -\eta$, $\gdindgam_3(w_t)\geq \delta_1$ and  $\gdindgam_4(w_t)\geq \delta_2$, thus, it holds that
\begin{align*}
    \gdindpopf(w_T)
    &\geq 
    \frac{7\eta}{64}\sqrt{T}+\delta_1+\delta_2-\eta
    \geq 
    \eta\left(\frac{7}{64}\sqrt{T}-1\right) 
    ;\\
    \gdindpopf(\opt)
    &\leq  
    \gdindpopf(0)
    \leq
    \frac{3\eta}{32}\sqrt{T}+\eta
    .
\end{align*}
Then, since $T$ is assumed large enough so that $2 \leq \frac{1}{128} \sqrt{T}$, we conclude
\begin{align*}
    \gdindpopf(w_T) - \gdindpopf(\opt) 
    \geq 
    \eta \br{ \frac{1}{64}\sqrt{T}-2 }
    \geq 
    \frac{\eta}{128}\sqrt{T}
    .
    &\qedhere
\end{align*}
\end{proof}
\input{appndix_sgd_underfit}

\subsection*{Acknowledgments}

This project has received funding from the European Research Council (ERC) under the European Union’s Horizon 2020 research and innovation program (grant agreements No.\ 101078075; 882396).
Views and opinions expressed are however those of the author(s) only and do not necessarily reflect those of the European Union or the European Research Council. Neither the European Union nor the granting authority can be held responsible for them.
This work received additional support from the Israel Science Foundation (ISF, grant number 2549/19), from the Len Blavatnik and the Blavatnik Family foundation, and from the Adelis Foundation.

\bibliographystyle{abbrvnat}
\bibliography{main}

\input{new_appendix}

\end{document}

%% file: abstract.tex
We study the generalization performance of gradient methods in the fundamental stochastic convex optimization setting, focusing on its dimension dependence. 
First, for full-batch gradient descent (GD) we give a construction of a learning problem in dimension $d=O(n^2)$, where the canonical version of GD (tuned for optimal performance of the empirical risk) trained with $n$ training examples converges, with constant probability, to an approximate empirical risk minimizer with $\Omega(1)$ population excess risk.
Our bound translates to a lower bound of $\Omega (\sqrt{d})$ on the number of training examples required for standard GD to reach a non-trivial test error, answering an open question raised by \citet*{feldman2016generalization} and \citet*{amir2021sgd} and showing that a non-trivial dimension dependence is unavoidable.
Furthermore, for standard one-pass stochastic gradient descent (SGD), we show that an application of the same construction technique provides a similar $\Omega(\sqrt{d})$ lower bound for the sample complexity of SGD to reach a non-trivial empirical error, despite achieving optimal test performance.
This again provides an exponential improvement in the dimension dependence compared to previous work~\citep*{benignunderfit}, resolving an open question left therein.

%% file: appndix_sgd_underfit.tex
\section{Underfitting of SGD: Proof of \texorpdfstring{\cref{SGDmain}}{Theorem \ref{SGDmain}}}
\label{sec:sgd_appendix}

In this section we show a formal proof of our main result for SGD.
As in $GD$, we construct a hard loss function, which is defined in a $d$-dimensional Euclidean space such that $d$ is polynomial in the number of examples $n$. Using this construction, we establish a lower bound of $\Omega\br[0]{ \eta \sqrt{T} }$ for the empirical loss of SGD with $T=n$ iterations. %
We complete the proof of \cref{SGDmain} in \cref{sec:const_diff}.

\paragraph{Full construction.}
For the first step of the construction, we use \cref{lem:set_direc_exists} (see \cref{sec:GD}), which shows for every dimension $d'$ an existence of a set of approximately orthogonal vectors in $U\in \R^{d'}$ with size exponential in $d'$. We define the set $U$ to be $U\eqq U_{d'}$ for $d'=712n \log n$ and the sample space to be $\sgdindZ\eqq \{V: V\subseteq U\}$.
Moreover, we define the hard distribution $\sgdindD$ to be such that every $u\in U$ is included in $V \subseteq U$ independently with probability $\delta=\frac{1}{4n^2}$.

For the hard loss function, we continue referring to every vector $w\in \R^d$ as a concatenation of vectors, $w = (\wenc, \vecpart{w}{1},\vecpart{w}{2},\ldots,\vecpart{w}{n} $), where for $1\leq k\leq n$, $\vecpart{w}{k}\in \R^{712n \log n}$ and $\wenc \in \R^{2n^2}$. 
In this construction, $\wenc$ is also a concatenation of $n$ vectors
$\vecpart{w}{\last, 1},\ldots,\vecpart{w}{\last, n}$ such that each for every $r\in[n]$,
$\vecpart{w}{\last, r}\in \R^{2n}$

Our approach is, as in $GD$, in every iteration $t$, to encode the set $V_t$, sampled from $\sgdindD$ into the iterate $\vecpart{w}{\last}_{t+1}$.
For this, we construct an encoder, $\phi:P(U) \times [n]\to \R^{2n}$, a decoder $\alpha:\R^{2n}\to U$, a real number $\epsilon>0$ and $n$ sets denoted as $\Psi_1,\ldots,\Psi_n$. 
Here, the idea behind the construction is such set $\psi_k$ represents all of the possible training sets with $k$ examples, $\{V_1\ldots,V_k\}$, and in every iteration ${t}$, it is possible to get the vector $\psi^*_{t-1}\in \Psi_{t-1}$ that is recognized with the actual sets $V_1,\ldots,V_{t-1}$ that are sampled before this iteration, as a maximizer of a linear function with margin $\epsilon$. 
Then, as outlined in \cref{SGD_sketch}, we aim to output a vector $u_t \in \bigcap_{i=t}^n \overline{V}_i$. 
The exact construction of such $\epsilon,\phi,\alpha,\Psi_{1},\ldots,\Psi_n$ is detailed in \cref{diff_set_encode exists_SGD} in \cref{sec:proofs_SGD}.

Then, for $\epsilon,\phi,\alpha,\Psi_1,\ldots,\Psi_n$ and $d\eqq nd'+2n^2=712n^2\log n+2n^2$ we define the loss function in our construction. 
The loss function $\sgdindf$ is composed of three terms: $\sgdindgam_{1}$, $\sgdindgam_{2}$, $\sgdindgam_{3}$, and is defined as follows,
\begin{align}
\label{sgd_loss_func}
&\sgdindf(w, V) 
\eqq
\underbrace{\sqrt{\sum_{k=2}^T
\max\left(\frac{3\eta}{32}, \max_{u\in V} \langle u,\vecpart{w}{k}\rangle \right)^2}}_{\sgdindgam_1(w,V) \eqq}
\\&\quad +
\max\bigg(\delta_1,\max_{k\in [n-1], u\in U,\psi\in \Psi_k}\bigg(\frac{3}{8}\langle u,\vecpart{w}{k}\rangle-\frac{1}{2}\langle \alpha(\psi),\vecpart{w}{k+1}\rangle +\langle\vecpart{w}{ \last,k},\frac{1}{4n}\psi\rangle \notag
\\ &\quad\quad\quad\quad\quad\quad -\langle\vecpart{w}{ \last,k+1},\frac{1}{4n}\psi\rangle+\langle \vecpart{w}{\last,k+1},-\frac{1}{4n^2}\phi(V,k+1)\rangle\bigg)\bigg)\notag
\\& \notag
+\underbrace{\langle \vecpart{w}{\last,1},-\frac{1}{4n^2}\phi(V,1)\rangle -\langle \frac{1}{n^3} u_1, \vecpart{w}{1}\rangle}_{\sgdindgam_3(w,V)\eqq}
,
\end{align}
where the second term is denoted $\sgdindgam_2(w,V)$ and $u_1$ is an arbitrary vector in $U$.
In the following lemma, we establish that the above loss function in indeed convex and Lipschitz over $\R^d$. The proof appears in \cref{sec:proofs_SGD}.
\begin{lemma}
\label{convex_lip_SGD}
For every $V\in Z$, the loss function $\sgdindf(w,V)$ is convex and $4$-Lipschitz over $\R^d$ with respect to its first argument.
\end{lemma}
For this construction of distribution and loss function, we show the following theorem,
\begin{theorem}
\label{SGD_lower_bound}
Assume that $n>2048$ and $\eta\leq \frac{1}{\sqrt{T}}$.
Consider the distribution
$\sgdindD$ %
and the loss function $\sgdindf$ with $d=712 n^2\log n +2n^2$, $\varepsilon = \frac{1}{n^2}
\br[0]{1-\cos\br[0]{\frac{2\pi}{|P(U)|}}}$ and $\delta_1=\frac{\eta}{8n^3}$.
Then, %
for Unprojected SGD (cf. \cref{sgd_update_rule} with $W=\R^d$) with $T=n$ iterations, initialized at $w_1=0$ with step size $\eta$, we have,  %
with probability at least $\frac{1}{2}$ over the choice of the training sample,
\begin{enumerate}[label=(\roman*)]
    \item 
    The iterates of SGD remain within the unit ball, namely $w_t \in \unitballd$ for all $t=1,\ldots,T$;
    \item
    For all $m=1,\ldots, T$, the $m$-suffix averaged iterate has:
    \[ 
        \sgdindempf(w_{T,m})
        - \sgdindempf(\erm) 
        = 
        \Omega\br[1]{\eta\sqrt{T}}
        .
    \]
\end{enumerate}
\end{theorem}
\paragraph{Algorithm's dynamics.}
As in GD, we provide a key lemma that characterizes the trajectory of SGD
under a certain "good event".
For this good event, given a random training set sample $S=\{V_i\}_{i=1}^n$, we denote $P_t=\bigcap_{i=1}^{t-1}V_i$ and $S_t=\bigcap_{i=t}^{t=n}\overline{V_i}$. Moreover, if $P_t\neq \emptyset$, we denote $r_t=\argmin \{r:V_t\in P_t\}$ and $J_t=v_{r_{t}}\in U$.
The good event is given as follows,
\begin{equation}
    \label{good_event_sgd}
    \cE'=\{\forall t\leq T \ P_t\neq\emptyset \ \text{and} \ J_t\in S_t\}
\end{equation}
In the following lemma we show that $\cE'$ occurs with a constant probability. The proof appears in \cref{sec:proofs_SGD}.
\begin{lemma}
    \label{prefix_event}
    For $T=n$ and the event $\cE'$ defined in \cref{good_event_sgd}, it holds that $\Pr(\cE')\geq \frac{1}{2}$.
\end{lemma}

Under this event, the dynamics of $SGD$ is characterized as follows,

\begin{lemma}
\label{diffSGDexpression}
Assume the conditions of \cref{SGD_lower_bound}, and consider the iterates of \emph{unprojected} $SGD$, with step size $\eta\leq \frac{1}{\sqrt{T}}$ initialized at $w_1=0$.
Under the event $\cE'$, we have for $t\geq 4$ and $s\neq \last$,
\begin{align*}
    \vecpart{w_t}{k}=\begin{cases}
        -\frac{3}{8}\eta u_{1} + (t-1)\frac{\eta}{n^3} u_1 &\quad k=1
        \\
        \frac{1}{8}\eta u_{k} &\quad 2\leq k\leq t-2
        \\
        \frac{1}{2}\eta u_{t-1} &\quad k=t-1
        \\
        0 &\quad t\leq k\leq n,
    \end{cases}
\end{align*} 
and for $s=\last$,
\begin{align*}
    \vecpart{w_t}{\last,k}=\begin{cases}
         \frac{\eta}{4n^2}\sum_{i=2}^{t-1} \phi(V_i,1) & k= 1
        \\ \frac{\eta}{4n^2}\sum_{i=1}^{t-1}\phi(V_i,i) & k= t-1
        \\ 0 & k\notin\{1,t-1\}.
    \end{cases}
\end{align*}
    where $u_1 \in U$ and 
    every another vector $u_{k}$ holds $u_{k}\in \bigcap_{i=1}^{k-1}V_i\cap \bigcap_{i=k}^{n}\overline{V_i}$.
\end{lemma}
For proving this key lemma, we analyze how the terms $\sgdindgam_1,\sgdindgam_2$ affects the dynamics of SGD under the event $\cE'$. First, we show that the gradients of $\sgdindgam_1$ does not affect the dynamics of $SGD$, as the gradient of this term in any iterate $w_t$ is zero. The idea is formalized in the following lemma. The proof is deferred to \cref{sec:proofs_SGD}.
\begin{lemma}
\label{gradient_uc_sgd}
    Assume the conditions of \cref{SGD_lower_bound} and the event $\cE'$. 
    Let $w\in \R^d$ and $t$ be such that for every $2\leq k\leq t-1$, $\vecpart{w}{k}=c\eta u_k$ for $c\leq \frac{1}{2}$ and every such $u_k$ holds $u_k\in \bigcap_{i=1}^{k-1}V_i\cap \bigcap_{i=k}^{n}\overline{V_i}$, and for every $t\leq k\leq T$, $\vecpart{w}{k}=0$. Then, for every $t$, it holds that, $\sgdindgam_1$ is differentiable at $(w,V)$ and for all $i \in [n]$, we have
    $\nabla\sgdindgam_1(w,V_t)=0.$
\end{lemma}
Now, we analyze the gradient of $\sgdindgam_2$. The role of this component is to decode the next "bad solution" $\alpha\left(\frac{1}{n}\sum_{i=1}^{t-1}\phi(V_i, i)\right)$ from the sets $V_1,\ldots,V_{t-1}$, and make a progress in this direction in some subspace $\vecpart{W}{t-1}$. In the following lemma, we show that the gradient of $\sgdindgam_2$, serves this goal.
\begin{lemma}
    \label{gradient_decode_sgd}
    Assume the conditions of \cref{SGD_lower_bound} and the event $\cE'$. For every $k$, let $\psi^*_k=\frac{1}{n}\sum_{t=1}^{k}\phi(V_t, t)$. Moreover, let $m\geq 3$ and $w\in \R^d$ such that $\vecpart{w}{1}=c\eta u_1$ for $ -\frac{3}{8}\leq c\leq 0$ and $u_1\in U$, for every $2\leq k\leq m-1$, $\vecpart{w}{k}=\frac{1}{8}\eta u_k$ such that every $u_k$ holds $u_k\in \bigcap_{t=1}^{k-1}V_t\cap \bigcap_{t=k}^{n}\overline{V_t}$, $\vecpart{w}{m}=\frac{1}{2}\eta u_m$ where $u_m$ holds $u_m\in \bigcap_{t=1}^{m-1}V_t\cap \bigcap_{i=m}^{n}\overline{V_t}$
    and for every $m+1\leq k\leq T$, $\vecpart{w}{k}=0$.
    Moreover, assume that $w$ holds $\vecpart{w}{0,m}=\frac{\eta}{4n}\psi^*_m$, $\|\vecpart{w}{0,1}\|\leq \eta$ and for every $k\notin \{m,1\}$, $\vecpart{w}{0,k}=0$.
    Then, 
for every $V\subseteq U$,
    $\sgdindgam_2$ is differentiable at $(w,V)$ and, we have for $k\neq 0$,
    \begin{align*}
    \vecpart{\nabla \sgdindgam_2(w,V)}{k}=
    \begin{cases}
        \frac{3}{8} u_{m} &\quad k=m
        \\
        -\frac{1}{2} \alpha(\psi_{m}^*) &\quad k=m+1
        \\
        0 &\quad k\notin \{m,m+1\}   
    \end{cases}
\end{align*}
and, 
\begin{align*}
    \vecpart{\nabla \sgdindgam_2(w,V)}{\last, k}=
    \begin{cases}
        \frac{1}{4n^2}\sum_{t=1}^{m}\phi(V_t,i)&\quad k= m
        \\ -\frac{1}{4n^2}\sum_{t=1}^{m}\phi(V_t,i)-\frac{1}{4n^2}\phi(V,i)&\quad k= m+1
        \\  0&\quad k\notin \{m,m+1\}.
    \end{cases}
\end{align*}
\end{lemma}
Now we can prove \cref{diffSGDexpression}.
\begin{proof}[of \cref{diffSGDexpression}]
    We assume that $\cE'$ holds and prove the lemma by induction on $t$.
    We begin from the basis of the induction, $t=4$, which is proved in \cref{sgd_exp_4} in \cref{sec:proofs_SGD}. 
Now, we assume the hypothesis of the induction, that the lemma holds for iteration $t$ and turn to show the required for iteration $t+1$.

First, we notice that for every $2\leq k\leq t-1$, $\vecpart{w_t}{k}=c\eta u_k$ for $c\leq \frac{1}{2}$ and every such $u_k$ holds $u_k\in \bigcap_{i=1}^{k-1}V_i\cap \bigcap_{i=k}^{n}\overline{V_i}$, and for every $t\leq k\leq T$, $\vecpart{w_t}{k}=0$. Then, by \cref{gradient_uc_sgd}, we have that $\nabla\sgdindgam_1(w_t,V_t)=0.$
    
Second, $\sgdindgam_3$ is a linear function, thus,
\[\nabla \vecpart{\sgdindgam_3(w_t,V_t)}{s}=\begin{cases}
    -\frac{1}{n^3}u_1 &s=1
    \\ -\frac{1}{4n^2}\phi(V_t,1)& s= \last,1
    \\0 &
\text{otherwise}.\end{cases}\]
Third, For $\sgdindgam_2(w_t,V_t)$, we notice for $m=t-1\geq 3$ it holds that $\vecpart{w_t}{1}=c\eta u_1$ for $ -\frac{3}{8}\leq c\leq 0$ and $u_1\in U$, for every $2\leq k\leq m-1$, $\vecpart{w_t}{k}=\frac{1}{8}\eta u_k$ such that every $u_k$ holds $u_k\in \bigcap_{t=1}^{k-1}V_t\cap \bigcap_{t=k}^{n}\overline{V_t}$, $\vecpart{w_t}{m}=\frac{1}{2}\eta u_m$ where $u_m$ holds $u_m\in \bigcap_{t=1}^{m-1}V_t\cap \bigcap_{i=m}^{n}\overline{V_t}$,
    and for every $m+1\leq k\leq T$, $\vecpart{w}{k}=0$.
    Moreover, $w_t$ holds $\vecpart{w}{0,m}=\frac{\eta}{4n}\psi^*_m$, $\|\vecpart{w_t}{0,1}\|\leq \eta$ and for every $k\notin \{m,1\}$, $\vecpart{w_t}{0,k}=0$.
Then, by \cref{gradient_decode_sgd}, we get that, we have for $k\neq 0$,
    \begin{align*}
    \vecpart{\nabla \sgdindgam_2(w_t,V_t)}{k}=
    \begin{cases}
        \frac{3}{8} u_{t-1} &\quad k=t-1
        \\
        -\frac{1}{2} \alpha(\psi_{t-1}^*) &\quad k=t
        \\
        0 &\quad k\notin \{t-1,t\}   
    \end{cases}
\end{align*}
and, 
\begin{align*}
    \vecpart{\nabla \sgdindgam_2(w_t,V_t)}{\last, k}=
    \begin{cases}
        \frac{1}{4n^2}\sum_{i=1}^{t-1}\phi(V_i,i)&\quad k= m
        \\ -\frac{1}{4n^2}\sum_{i=1}^{t}\phi(V_i,i)&\quad k= m+1
        \\  0&\quad k\notin \{m,m+1\}.
    \end{cases}
\end{align*}
Now, by \cref{diff_set_encode exists_SGD}, for $j=\argmin_{i}\{i:v_i \in \bigcap_{i=1}^{t-1} V_i\}$, we get that
\begin{align*}
    \alpha(\psi_{t-1}^*)=v_{j}\in \bigcap_{i=1}^{t-1} V_i.
\end{align*}
 We notice that $\bigcap_{i=1}^{t-1} V_i=P_t$ and thus $\alpha(\psi_{t-1}^*)=J_t$. Then, 
by $\cE'$, $\alpha(\psi_{t-1}^*)$ also holds $\alpha(\psi_{t-1}^*)\in S_{t}$.
Combining the above together, we get, for $u_t=\alpha(\psi_{t-1}^*)\in P_t\cap S_t$,
\begin{align*}
    \vecpart{\nabla f(w_t,V_t)}{k}=
    \begin{cases}
        - \frac{1}{n^3}u_1&\quad k=1\\
        \frac{3}{8} u_{t-1} &\quad k=t-1
        \\
        -\frac{1}{2} u_{t} &\quad k=t
        \\
        0 &\quad k\notin \{1,t-1,t\},
        \end{cases}
\end{align*}
and,
\begin{align*}
    \vecpart{\nabla f(w_t,V_t)}{\last, k}=
    \begin{cases}
        -\frac{1}{4n^2}\phi(V_3,1) &\quad k=1
        \\ \frac{1}{4n^2}\sum_{i=1}^{t-1}\phi(V_i,i)&\quad k=t-1
        \\ -\frac{1}{4n^2}\sum_{i=1}^{t}\phi(V_i,i)&\quad k=t
        \\  0&\quad k\notin \{1,t-1,t\},
    \end{cases}
\end{align*}
and the lemma follows.
\end{proof}
The proof of \cref{SGD_lower_bound} is similar to \cref{nonspec_lower_bound_GD}, using \cref{diffSGDexpression} instead of \cref{diff_GD_expression}, and is deferred to \cref{sec:proofs_SGD}.

%% file: new_appendix.tex
\appendix
\input{appendix_differentiable}

\input{appendix_proofs_GD}

\input{appendix_proof_sgd}
\input{appendix_proof_differentiable}

\input{appendix_optimization_bound}

%% file: appendix_differentiable.tex
\section{Differentiability and Proofs of \cref{GDmain,SGDmain}}
\label{sec:const_diff}

In this section, we complete the proof of \cref{GDmain,SGDmain}, by showing a construction of a differentiable objective that maintains the same lower bounds given in \cref{nonspec_lower_bound_GD,nonspec_lower_bound_GD_small_eta,SGD_lower_bound}.
Our general approach is to use a randomized smoothing of the original objectives. Then, we use the fact that the subgradients are unique along any possible trajectory of GD, 
to show that when smoothing is applied within a sufficiently small neighborhood, gradients along any such trajectory are preserved. Consequently, this approach does not impact the dynamics of the optimization algorithm, while simultaneously ensuring the objectives become differentiable everywhere.

\subsection{Proof of \cref{GDmain}}
\label{sec:gd_diff}

\paragraph{Full construction.}
The hard distribution $\gdindD$ is defined to be as in \cref{sec:GD}. The hard loss function is a smoothing of $\gdindf$ (\cref{loss_def_gd_in_setup_nondiff}), and is defined as \begin{align}
\label{loss_def_gd_in_setup_diff}
   \dgdindf(w,(V,j))
    \eqq
    \E_{v\in B}\sbr{ \gdindf(w+\delta v,(V,j)) }
    ,
\end{align}
for a sufficiently small $\delta>0$ and the $d$-dimensional unit ball $B$. 
Analogously, we denote the empirical loss and the population loss with respect to the loss function $\dgdindf$ as $\dgdindempf(w)=\frac{1}{n}\sum_{i=1}^n \dgdindf(w,(V_i,j_i))$ and $\dgdindpopf(w)=\E_{(V,j)\sim \gdindD} \dgdindf(w,(V,j))$, respectively. 
The loss function $\dgdindf$ is differentiable, $5$-Lipschitz with respect to its first argument and convex over $\R^d$, as stated in the following lemma. 
\begin{lemma}
\label{convex_lip_diff}
For every $(V,j) \in Z$, the loss function $\dgdindf$ is differentiable, convex and $5$-Lipschitz with respect to its first argument and over $\R^d$.
\end{lemma}
We first prove the following theorem, %

\begin{theorem}
\label{nonspec_lower_bound_GD_diff}
Assume that $n>0$, $T>3200^2$ and $\eta\leq\frac{1}{\sqrt{T}}$. %
Consider the distribution $\gdindD$ and the loss function $\dgdindf$ for $d=178nT+2n^2$, $\varepsilon = \frac{1}{n^2}
\br[0]{1-\cos\br[0]{\frac{2\pi}{|P(U)|}}}$, $\beta=\frac{\epsilon}{4T^2}$, $\delta=\frac{\eta\beta}{32}$, $\delta_1=\frac{\eta}{2n}$ and $ \delta_2=\frac{3\eta\beta}{16}$.
Then, %
 for Unprojected GD (cf. \cref{gd_update_rule} with $W=\R^d$) on $\gdindempf$, initialized at $w_1=0$ with step size $\eta$, we have,  %
with probability at least $\frac{1}{6}$ over the choice of the training sample:
\begin{enumerate}[label=(\roman*)]
    \item 
    The iterates of GD remain within the unit ball, namely $w_t \in \unitballd$ for all $t=1,\ldots,T$;
    \item
    For all $m=1,\ldots,T$, the $m$-suffix averaged iterate has:
    \[ 
        \dgdindpopf(\suff)
        - \dgdindpopf(\opt) 
        = 
        \Omega\br[1]{\eta\sqrt{T}}
        .
    \]
\end{enumerate}
\end{theorem}

\paragraph{Algorithm dynamics.}
Now we show that the dynamics of GD when is applied on $\dgdindempf$ is identical to dynamics of the algorithm on $\gdindempf$, as stated in the following lemma.
\begin{lemma}
    \label{iden_grad_iter}
Under the conditions of \cref{nonspec_lower_bound_GD_diff,nonspec_lower_bound_GD}, let $w_t,\Tilde{w}_t$ be the iterates of Unprojected GD with step size $\eta\leq\frac{1}{\sqrt{T}}$ and $w_1=0$, on $\gdindempf$ and $\dgdindempf$ respectively. 
    Then, if $\cE$ occurs, then for every $t\in[T]$, it holds that 
    $w_t=\Tilde{w}_t$. 
\end{lemma}

\paragraph{Proof of \cref{nonspec_lower_bound_GD_diff}.}
Next, we set out to establish the proof for \cref{nonspec_lower_bound_GD_diff}. 
\begin{proof}[of \cref{nonspec_lower_bound_GD_diff}]
Let $\overline{\suff}$ be the $\indsuff$-suffix average of $GD$ when is applied on $\gdindempf$. Let $\overline{\opt}=\argmin_w \gdindpopf(w)$. By \cref{iden_grad_iter}, we know that, with probability of at least $\frac{1}{6}$, $\cE$ occurs and $\suff=\overline{\suff}$. Then, by \cref{nonspec_lower_bound_GD} and \cref{expec_change_bound},
\begin{align*}
   \frac{\eta}{3200}\sqrt{T}&\leq  \gdindpopf(\overline{\suff})-\gdindpopf(\overline{\opt}) 
   \\&=\gdindpopf(\suff)-\gdindpopf(\overline{\opt})
    \\&\leq \dgdindpopf(\suff)+5\delta-\dgdindpopf(\overline{\opt})+5\delta 
    \\&\leq \dgdindpopf(\suff)+5\delta-\dgdindpopf(\opt)+5\delta,
\end{align*}
and, 
\begin{align*}
    \dgdindpopf(\suff)-\dgdindpopf(\opt)&\geq  \frac{\eta}{3200}\sqrt{T}-\frac{10\eta\epsilon}{128T^2}
    \\&\geq 
    \frac{\eta}{3200}\sqrt{T}-\frac{\eta}{10T^2}
    \\&\geq 
    \frac{\eta}{6400}\sqrt{T} \tag{$T\geq 30$}
   . 
\end{align*}
\end{proof}

Now we can finally prove \cref{GDmain}. The proof is an immediate corollary from \cref{nonspec_lower_bound_GD_diff}
and the lower bound of $\Omega\left(\min\left(\frac{1}{\eta T},1\right)\right)$ given in \cref{nonspec_lower_bound_GD_diff_small_eta} in \cref{sec_small_eta}.
It's important to highlight that we offer a rigorous proof for a modified version of \cref{GDmain}, where the loss function $\gdindf$ possesses a Lipschitz constant of only 5.
By scaling down this loss function by a factor of $\frac{1}{5}$ and simultaneously adjusting the step size $\eta$ by a factor of 5, we can employ the same proof to establish the validity of \cref{GDmain}.
\begin{proof}[of \cref{GDmain}]
    We know that $\eta\leq \frac{1}{5\sqrt{T}}$.
    First, by \cref{nonspec_lower_bound_GD_diff}, we know that for Unprojected $GD$ and $d_1=178nT+2n^2$, there exist a distribution $\gdindD$ over a probability space $Z$, a constant $C_1$ and a loss function $\dgdindf:\R^{d_1}\times Z\to \R$ such that, with probability of at least $\frac{1}{6}$, \[\dgdindpopf(\suff)-\dgdindpopf(\opt)\geq C_1\eta\sqrt{T.}\]
    Second,  by \cref{nonspec_lower_bound_GD_diff_small_eta}, we know that for Unprojected $GD$ and $d_2=\max(25\eta^2T^2,1)$, there exist a constant $C_2$ and a deterministic loss function $\doptindf:\R^{d_2}\to \R$ such that \[\doptindf(\suff)-\doptindf(\opt)\geq
    C_2\min\left(1,\frac{1}{\eta T}\right)\]
    Now, let $C=\frac{1}{2}\min \left(C_1, C_2\right)$.
    If $\eta\geq T^{-\frac{3}{4}}$, then, $\eta\sqrt{T}\geq \min(1,\frac{1}{\eta T})$, and we get,
    \[\dgdindpopf(\suff)-\dgdindpopf(\opt)\geq C\left(\eta\sqrt{T} +\min\left(1,\frac{1}{\eta T}\right)\right)\geq C\left( \min\left(1,\eta\sqrt{T}+\frac{1}{\eta T}\right)\right).\]
    Otherwise, we get that, 
    \[\doptindf(\suff)-\doptindf(\opt)\geq C\left(\eta\sqrt{T} +\min\left(1,\frac{1}{\eta T}\right)\right)\geq C\left( \min\left(1,\eta\sqrt{T}+\frac{1}{\eta T}\right)\right).\]
    Since in both cases, by \cref{nonspec_lower_bound_GD_diff_small_eta,nonspec_lower_bound_GD_diff}, $w_t\in \unitballd$ for every $t\in[T]$, the theorem is applicable also for Projected GD.
\end{proof}

\subsection{Proof of \cref{SGDmain}}
\label{sec:sgddiff}
\paragraph{Full construction.}
The hard distribution $\sgdindD$ is defined to be as in \cref{sec:sgd_appendix}. The hard loss function is a smoothing of $\sgdindf$ (\cref{sgd_loss_func}), and is defined as \begin{align}
\label{loss_def_sgd_in_setup_diff}
   \dsgdindf(w,V)
    \eqq
    \E_{v\in B}\sbr{ \sgdindf(w+\delta v,V) }
    ,
\end{align}
for a sufficiently small $\delta>0$ and the $d$-dimensional unit ball $B$. 
Analogously, we denote the empirical loss and the population loss with respect to the loss function $\dsgdindf$ as $\dsgdindempf(w)=\frac{1}{n}\sum_{i=1}^n \dsgdindf(w,V_i)$ and $\dsgdindpopf(w)=\E_{V\sim \gdindD} \dsgdindf(w,V)$, respectively. 
The loss function $\dsgdindf$ is differentiable, $4$-Lipschitz with respect to its first argument and convex over $\R^d$, as stated in the following lemma. 
\begin{lemma}
\label{convex_lip_diff_sgd}
For every $V\in Z$, the loss function $\dsgdindf$ is differentiable, convex and $4$-Lipschitz with respect to its first argument and over $\R^d$.
\end{lemma}
We first prove the following theorem, %

\begin{theorem}
\label{nonspec_SGD_lower_bound_diff}
Assume that $n>2048$ and $\eta\leq \frac{1}{\sqrt{n}}$.
Consider the distribution
$\sgdindD$ %
and the loss function $\dsgdindf$ with $d=712 n^2\log n +2n^2$, $\varepsilon = \frac{1}{n^2}
\br[0]{1-\cos\br[0]{\frac{2\pi}{|P(U)|}}}$, $\delta=\frac{\eta\varepsilon}{32n^3}$ and $\delta_1=\frac{\eta}{8n^3}$.
Then, %
for Unprojected SGD (cf. \cref{sgd_update_rule} with $W=\R^d$) with $T=n$ iterations, initialized at $w_1=0$ with step size $\eta$, we have,  %
with probability at least $\frac{1}{2}$ over the choice of the training sample,
\begin{enumerate}[label=(\roman*)]
    \item 
    The iterates of SGD remain within the unit ball, namely $w_t \in \unitballd$ for all $t=1,\ldots,n$;
    \item
    For all $m=1,\ldots, n$, the $m$-suffix averaged iterate has:
    \[ 
        \dsgdindempf(w_{n,m})
        - \dsgdindempf(\erm) 
        = 
        \Omega\br[1]{\eta\sqrt{n}}
        .
    \]
\end{enumerate}
\end{theorem}
\paragraph{Algorithm's dynamics.}
Now, As in GD, the main step in proving \cref{nonspec_SGD_lower_bound_diff} is to show that taking expectation of $\sgdindf$ for every point $w$ in a ball with small enough radius does not change the dynamics of $SGD$.
\begin{lemma}
    \label{iden_grad_iter_sgd}
    Under the conditions of \cref{nonspec_SGD_lower_bound_diff,SGD_lower_bound}, let $w_t,\Tilde{w}_t$ be the iterates of Unprojected SGD with step size $\eta\leq\frac{1}{\sqrt{T}}$ and $w_1=0$, on $\sgdindempf$ and $\dsgdindempf$ respectively. 
    Then, if $\cE'$ occurs, then for every $t\in[T]$, it holds that 
    $w_t=\Tilde{w}_t$. 
\end{lemma}
\paragraph{Proof of \cref{nonspec_SGD_lower_bound_diff}.}
Next, we set out to establish the proof for  \cref{nonspec_SGD_lower_bound_diff}.
\begin{proof}[of \cref{nonspec_SGD_lower_bound_diff}]
Let $\overline{w_{n,m}}$ be the $m$- suffix average of $SGD$ when is applied on $\sgdindf$ and let $\overline{\erm}=\argmin_w \sgdindempf(w)$. By \cref{iden_grad_iter_sgd}, we know that, with a probability $\frac{1}{2}$ ,$w_{n,m}=w^\sgdind_{n,m}$. Then, by \cref{SGD_lower_bound} and \cref{expec_change_bound},
\begin{align*}
   \frac{\eta}{64000}\sqrt{n}&\leq  \sgdindempf(\overline{w_{n,m}})-\sgdindempf(\overline{\erm}) \\&=\sgdindempf(w_{n,m})-\sgdindempf(\overline{\erm})
    \\&\leq \dsgdindempf(w_{n,m})+4\delta-\dsgdindempf(\overline{\erm})+4\delta
    \\&\leq \dsgdindempf(w_{n,m})+4\delta-\dsgdindempf(\erm)+4\delta,
\end{align*}
and, 
\begin{align*}
    \dsgdindempf(w_{n,m})-\dsgdindempf(\erm)&\geq  \frac{\eta}{64000}\sqrt{n}-\frac{\eta\epsilon}{4n^3}
    \\&\geq 
    \frac{\eta}{64000}\sqrt{n}-\frac{\eta}{4n^3}
    \\&\geq 
    \frac{\eta}{128000}\sqrt{n} \tag{$n\geq 40$}
   . 
\end{align*}
\end{proof}

Now we can finally prove \cref{SGDmain}.
\begin{proof}[proof of \cref{SGDmain}]
We know that $T=n$ and $\eta\leq \frac{1}{5\sqrt{T}}$.
    First, by \cref{nonspec_SGD_lower_bound_diff}, we know that for Unprojected $SGD$ and $d_1=712n^2\log n+2n^2$, there exist a distribution $\sgdindD$ over a probability space $Z$, a constant $C_1$ and a loss function $\dsgdindf:\R^{d_1}\times Z\to \R$ such that, with probability of at least $\frac{1}{2}$, \[\dsgdindempf(\suff)-\dsgdindempf(\erm)\geq C_1\eta\sqrt{T.}\]
    Second,  by \cref{nonspec_lower_bound_GD_diff_small_eta}, we know that for Unprojected $SGD$ and $d_2=\max(25\eta^2T^2,1)$, there exist a constant $C_2$ and a deterministic loss function $\doptindf:\R^{d_2}\to \R$ such that \[\doptindf(\suff)-\doptindf(\erm)\geq
    C_2\min\left(1,\frac{1}{\eta T}\right)\]
    Now, let $C=\frac{1}{2}\min \left(C_1, C_2\right)$.
    If $\eta\geq T^{-\frac{3}{4}}$, then, $\eta\sqrt{T}\geq \min(1,\frac{1}{\eta T})$, and we get,
    \[\dsgdindempf(\suff)-\dsgdindempf(\erm)\geq C\left(\eta\sqrt{T} +\min\left(1,\frac{1}{\eta T}\right)\right)\geq C\left( \min\left(1,\eta\sqrt{T}+\frac{1}{\eta T}\right)\right).\]
    Otherwise, we get that, 
    \[\doptindf(\suff)-\doptindf(\erm)\geq C\left(\eta\sqrt{T} +\min\left(1,\frac{1}{\eta T}\right)\right)\geq C\left( \min\left(1,\eta\sqrt{T}+\frac{1}{\eta T}\right)\right).\]
    Since in both cases, by \cref{nonspec_lower_bound_GD_diff_small_eta,nonspec_SGD_lower_bound_diff}, $w_t\in \unitballd$ for every $t\in[T]$, the theorem is applicable also for Projected SGD.
\end{proof}

%% file: appendix_proofs_GD.tex
\section{Proofs of \texorpdfstring{\cref{sec:GD}}{Section \ref{sec:GD}}}
\label{sec:GD_proofs}

\subsection{Proofs for the full construction}
\label{sec:gd_const_proofs}

\begin{proof}[of \cref{lem:set_direc_exists}] Let $r=2^{\frac{-d'}{178}}$. 
    For every $1\leq i\leq r$ and $1\leq j\leq d'$ we define the random variable $u_i^j$ be a random variable to be $\frac{1}{\sqrt{d'}}$ with probability $\frac{1}{2}$ and $-\frac{1}{\sqrt{d'}}$ with probability $\frac{1}{2}$. Then, for every $1\leq i\leq r$, we define the vector $u_i$ which its $j$th entry is $u_i^j$ and look at the set $U=\{u_1,u_2,...u_{r}\}$. This set will hold the required property with positive probability.
    First, for every $i\neq k$, $\langle u_i,u_k\rangle$ are sums of $d$ random variables that taking values in $[-\frac{1}{d'},\frac{1}{d'}]$ with $\E\langle u_i,u_k\rangle=0$.
    Then by Hoeffding's inequality,
    \begin{align*}
        Pr(|\langle u_i,u_k\rangle|\geq \frac{1}{8})&\leq 2e^{\frac{-2\left(\frac{1}{8}\right)^2}{d'\cdot \frac{4}{d'^2}}}=2e^{-\frac{d'}{128}}
    \end{align*}
Then, by union bound on the $\binom{r}{2}$ pairs of vectors in $U$,
\begin{align*}
        Pr(\exists i,k \ |\langle u_i,u_k\rangle|\geq \frac{1}{8})&\leq 2e^{-\frac{d'}{128}}\cdot \binom{r}{2}< 2e^{-\frac{d'}{128}}\cdot \frac{1}{2}r^2\leq 1.
    \end{align*}
\end{proof}

\begin{lemma}
\label{diff_set_encode exists_GD}
Let $n, d \geq 1$ and a set $U \subseteq \unitballd$. Let $P(U)$ be the power set of $U$. Then, 
there exist a set $\Psi\subseteq \R^{2n^2}$, a number $0<\epsilon<\frac{1}{n}$ and two mappings $\phi:P(U) \times [n^2]\to \R^{2n^2}$, $\alpha:\R^{2n^2}\to U$ such that, 
\begin{enumerate}[label=(\roman*)]
    \item For every $j\in[n^2]$ and $V\subseteq U$, it holds $\|\phi\left(V,j\right)\|\leq 1$;
    \item For every $\psi\in\Psi$, it holds $\|\psi\|\leq 1$; 
    \item  Let $V_1,\ldots,V_n$ be arbitrary subsets of $U$. 
    If $j_1,\ldots,j_n$ hold that $j_i\neq j_k$ for $i\neq k$, $\psi^*=\frac{1}{n}\sum_{i=1}^n\phi(V_i,j_i)$ is that,
    \begin{itemize}
        \item %
        \[\abr[2]{ \psi^*, \frac{1}{n}\sum_{i=1}^n\phi(V_i,j_i) } > \frac{7}{8n} ;\]
        \item For every $\psi\in \Psi$, $\psi\neq \psi^*$:
        \[\abr[2]{ \psi^*, \frac{1}{n}\sum_{i=1}^n\phi(V_i,j_i) } \geq \abr[2]{ \psi, \frac{1}{n}\sum_{i=1}^n\phi(V_i,j_i) } + \epsilon ;\]
        \item If $\bigcup_{i=1}^nV_i\neq U$, then it holds that $\alpha(\psi^*)=v_{i^*} \in U\setminus\bigcup_{i=1}^{n}V_{i}$ for $i^* = \min\cbr{ i:v_i \in U\setminus\bigcup_{i=1}^nV_i }$.
       \end{itemize}
     \end{enumerate}
\end{lemma}
\begin{proof}
First, we consider an arbitrary enumeration of $P(U)=\{V^1,...V^{|P(U)|}\}$ and define $g:P(U)\to \R^2$, $g(V^i)= \left(\sin\left(\frac{2\pi i}{|P(U)|}\right),\cos\left(\frac{2\pi i}{|P(U)|}\right)\right)$. 
Now, we refer to a vector $a\in \R^{2n^2}$ as a concatenation of $n^2$ vectors in $\R^2$, $\vecpart{a}{1},...,\vecpart{a}{n^2}$. Then, we define $\delta=1-\cos\left(\frac{2\pi}{|P(U)|}\right)$, $\epsilon=\frac{\delta}{n^2}$ and \[\vecpart{\phi(V,j)}{i}=
\left\{
\begin{array}{cc}
    g(V) & i=j \\
    0 & \text{otherwise} 
\end{array}
\right.\]
As a result, for every $V^i,j$ it holds that
\begin{align*}
    \|\phi(V^i,j)\|=\|g(V^i)\|=\sqrt{\sin\left(\frac{2\pi i}{|P(U)|}\right)^2+\cos\left(\frac{2\pi i}{|P(U)|}\right)^2}=1
\end{align*}
Moreover, if $j_1\neq j_2$,
\begin{align*}
    \langle \phi(V^i,j_1),\phi(V^i,j_2)\rangle =0,
\end{align*}
and if $i>k$,
\begin{align*}
   \langle \phi(V^i,j),\phi(V^k,j)\rangle=&\langle  g(V^i), g(V^k)\rangle \\&=\sin\left(\frac{2\pi i}{|P(U)|}\right)\sin\left(\frac{2\pi k}{|P(U)|}\right) + \cos\left(\frac{2\pi i}{|P(U)|}\right)\cos\left(\frac{2\pi k}{|P(U)|}\right)
   \\&=\cos\left(\frac{2\pi (i-k)}{|P(U)|}\right)
   \\&\leq \cos\left(\frac{2\pi}{|P(U)|}\right) 
   \tag{$\cos$ is monotonic decreasing in $[0,\pi/2]$}
   \\&=1-\delta
\end{align*}
We notice that $0<\delta< 1$.
Now, we consider an arbitrary enumeration of $U=\{v_1,...v_{|U|}\}$, and define the following set $\Psi\subseteq \R^{2n^2}$ and the following two mappings $\sigma: R^{2n^2}\to P(U), \alpha:  R^{2n^2}\to U$,
\[\Psi = \{ \frac{1}{n}\sum_{i=1}^n \phi(V_i,j_i): \forall i\  V_i\ \subseteq U,\ j_i\in [n^2] \text{ and } i\neq l \implies j_i\neq j_\ell\}\]
Note that, for every $\psi \in \Psi$,
\[\|\psi\|=\|\frac{1}{n}\sum_{i=1}^n \phi(V_i,j_i)\|\leq \frac{1}{n}\sum_{i=1}^n \|\phi(V_i,j_i)\|=1.\]
Then, for every $a\in \R^{2n^2}$ and $j\in[n^2]$, we denote the index $q(a,j)\in [|P(U)|]$ as \[q(a,j)=\argmax_{r}\langle g(V_r), \vecpart {a}{j}\rangle,\] and define the following mapping $\sigma:\R^{2n^2}\to P(U)$,
\[\sigma(a)=\bigcup_{j=1,\vecpart {a}{j}\neq 0}^{n^2}V_{q(a,j)}.\]
Moreover, for every $a\in \R^{2n^2}$, we denote the index $p(a)\in [|U|]$ as \[p(a)=\argmin_{i}\{i:v_i \in U\setminus\sigma(a)\},\] and define the following mapping $\alpha:\R^{2n^2}\to U$,
\[\alpha(a)=\left\{\begin{array}{cc}
    v_{|U|} & \sigma(a)=U \\
     v_{p(a)} & \sigma(a)\neq U
\end{array}\right..\]

Now, Let $V_1,\ldots, V_n\subseteq U$ and $j_1,...j_n$ that are sampled uniformly from $[n^2]$, 
We prove the last part of the lemma under the condition that $j_i\neq j_k$ for $i\neq k$.
$\psi^*=\frac{1}{n}\sum_{i=1}^n \phi(V_i,j_i)$ holds 
\begin{align*}
    \langle \psi^*, \frac{1}{n}\sum_{i=1}^n \phi(V_i,j_i)\rangle&=
   \langle \frac{1}{n}\sum_{i=1}^n \phi(V_i,j_i), \frac{1}{n}\sum_{i=1}^n \phi(V_i,j_i)\rangle \\&= \frac{1}{n^2}\sum_{i=1}^n\langle \phi(V_i,j_i),\phi(V_i,j_i)\rangle
   \\&= \frac{1}{n}
   \\& > \frac{7}{8n}
\end{align*}
For $\psi=\frac{1}{n}\sum_{l=1}^n \phi(V'_l,j'_l)$ such that $\psi\neq \psi^*$, 
there are at most $n$ pairs $i, l$ such that $ \langle \phi(V'_i,j'_i),\phi(V'_l,j'_l)\rangle\neq 0$.
thus, 
there exists a pair $(V'_r,j'_r)$ that 
$(V'_r,j'_r)\notin \{(V_i,j_i) : i\in [n]\}$.
and for every $i$, $\langle \phi(V_i,j_i),\phi(V'_l,j'_l)\rangle\leq 1-\delta$.
As a result,
\begin{align*}
\langle \psi, \frac{1}{n}\sum_{l=1}^n \phi(V_i,j_i)\rangle&=
   \langle \frac{1}{n}\sum_{i=1}^n \phi(V'_l,j'_l), \frac{1}{n}\sum_{i=1}^n \phi(V_i,j_i)\rangle \\&=\frac{1}{n^2}\sum_{i=1}^n\sum_{l=1}^n\langle \phi(V_i,j_i),\phi(V'_l,j'_l)\rangle 
\\&\leq \frac{1}{n^2}\left( 1-\delta+\sum_{i=1,i\neq r}^n 1\right)
   \\&\leq \frac{1}{n^2}(1-\delta+n-1)\\&= \frac{1}{n}-\frac{\delta}{n^2}
   \\&= \langle \psi^*, \frac{1}{n}\sum_{i=1}^n \phi(V_i,j_i)\rangle-\epsilon
\end{align*}
Furthermore, since if all $j_i$ are distinct, for every $i$ it holds that, $\vecpart {\frac{1}{n}\sum_{i=1}^n\phi=(V_i,i)}{j_i}=\frac{1}{n}g(V_i)$, thus,
\begin{align*}
q\left(\frac{1}{n}\sum_{i=1}^n\phi(V_i,j_i),j_i\right)&=\argmax_{r}\langle g(V_r), \vecpart {\frac{1}{n}\sum_{i=1}^n\phi(V_i,j_i)}{j_i}\rangle
\\&=\argmax_{r}\langle g(V_r), \frac{1}{n}g(V_i)\rangle
\\&=i, 
\end{align*}
and we get,
\begin{align*}
    \sigma (\psi^*)&=\sigma \left(\frac{1}{n}\sum_{i=1}^n \phi(V_i,j_i)\right)\\
&=\bigcup_{j=1,\vecpart {\frac{1}{n}\sum_{i=1}^n\phi(V_l,j_i)}{j}\neq 0}^{n^2}V_{q\left(\frac{1}{n}\sum_{i=1}^n \phi(V_i,j_i),j\right)}
\\&=\bigcup_{i=1}^{n}V_{q\left(\frac{1}{n}\sum_{i=1}^n\phi(V_i,j_i),j_i\right)} \tag{The indices that are non-zero are $\{j_i\}_{i=1}^n\}$}
\\&=\bigcup_{i=1}^{n}V_{i} %
\end{align*}
Finally, assuming that $\bigcup_{i=1}^nV_i\neq U$,
\[\alpha(\psi^*)= v_{p(a)}
\in
U\setminus\bigcup_{i=1}^{n}V_{i}.\]
\end{proof}

\begin{proof}[of \cref{convex_lip}]
    We prove that $\gdindgam_{1},\gdindgam_{2}$ and $\gdindgam_{4}$ are convex and $1$-Lipschitz and $\gdindgam_{3}$ is convex and $1
    2$-Lipschitz. 
    
    First, by \cref{lem:set_direc_exists,diff_set_encode exists_GD} for every $u\in U$ and $V\in P(U)$, $j\in[n^2]$, it holds that $\|u\|=1$, $\|\phi(V,j)\|=1$.
    Then, $\gdindgam_{2}$ is a $1$-Lipschitz linear function, and $\gdindgam_{4}$ is a maximum over $1$-Lipschitz linear functions, thus, both functions are convex and $1$-Lipschitz.
    Moreover, for every possible $\psi\in \Psi$
    \[\|\psi\|=\|\frac{1}{n}\sum_{l=1}^n\phi(V_l)\|\leq \frac{1}{n}\sum_{l=1}^n\|\phi(V_l)\|=1.\]
    thus, $\gdindgam_{3}$ is a maximum over $2$-Lipschitz linear functions, thus, it is convex and $2$-Lipschitz.
    Now, for $\gdindgam_{1}$,  for every set $V\subseteq U$, let $\alpha_V(w)\in \R^{T-1}$ to be the vector which its $k$'th coordinate is $\vecpart{\alpha_V(w)}{k}=\max\left(\frac{3\eta}{32}, \max_{u\in V} \langle u\vecpart{w}{k+1}\rangle \right)$ and prove convexity and $1$-Lipshitzness.
    For establishing convexity, observe 
    \begin{align*}
    &\sqrt{\sum_{k=2}^T\max\left(\frac{3\eta}{32}, \max_{u\in V} \langle u,\vecpart{(\lambda x + (1-\lambda) y)}{k}\rangle \right)^2} 
    \\&= \sqrt{\sum_{k=2}^T\max\left(\frac{3\eta}{32}, \max_{u\in V} \left(\lambda \langle u,\vecpart{x}{k}\rangle +(1-\lambda)\langle u,\vecpart{ y}{k}\rangle\right) \right)^2} 
    \\&\leq \sqrt{\sum_{k=2}^T\max\left(\frac{3\eta}{32}, \max_{u\in V} \left(\lambda \langle u,\vecpart{x}{k}\rangle\right) +\max_{u\in V}\left((1-\lambda)\langle u,\vecpart{ y}{k}\rangle\right) \right)^2} \tag{convexity of $\max$ \& monotonicity of square root}
    \\&\leq \sqrt{\sum_{k=2}^T \left( \lambda\max\left(\frac{3\eta}{32}, \max_{u\in V} \left( \langle u,\vecpart{x}{k}\rangle\right)\right) +(1-\lambda)\max\left(\frac{3\eta}{32},\max_{u\in V}\langle u,\vecpart{ y}{k}\rangle\right) \right)^2} 
    \\& =\|\lambda \alpha_V(x) +(1-\lambda)\alpha_V(y)\|_2
    \\& \leq \lambda\|\alpha_V(x)\|_2 +(1-\lambda)\alpha_V(y)\|_2 \tag{convexity of $\ell_2$ norm}
    \\& = \lambda \sqrt{\sum_{k=2}^T\max\left(\frac{3\eta}{32}, \max_{u\in V} \langle u\vecpart{x}{k}\rangle \right)^2}+(1-\lambda)\sqrt{\sum_{k=2}^T\max\left(\frac{3\eta}{32}, \max_{u\in V} \langle u\vecpart{y}{k}\rangle \right)^2}.
    \end{align*}
   For $1$-Lipschitzness, 
    for every $w\in \R^d$ and sub-gradient $g(w,V)\in \partial \gdindgam_1(w,V)$, there exists a sub gradient  $h(w,V)\in \partial\left(\sum_{k=2}^T\max\left(\frac{3\eta}{32}, \max_{u\in V} \langle u\vecpart{w}{k}\rangle \right)^2\right)$ such that
    \begin{align*}
        \|g(w,V)\|=%
        \frac{\|h(w,V)\|}{2\sqrt{\sum_{k=2}^T\max\left(\frac{3\eta}{32}, \max_{u\in V} \langle u\vecpart{w}{k}\rangle \right)^2}}
        =\frac{\|h(w,V)\|}{2\sqrt{\sum_{k=2}^T{\vecpart{\alpha_V(w)}{k}}^{2}}}.
    \end{align*}
    Moreover, for every $k$ and sub gradient $b_{k,V}(w)\in \partial \left( \vecpart{\alpha_V(w)}{k}\right)$ we denote $r_{k,V}(w)\in R^d$ the vector with $\vecpart{r_{k,V}(w)}{k}=b_{k,V}(w)$ and for $j\neq k$, $\vecpart{r_{k,V}(w)}{j}=0$.
    Then, 
    for every sub gradient $h(w,V)\in \partial\left(\sum_{k=2}^T\max\left(\frac{3\eta}{32}, \max_{u\in V} \langle u\vecpart{w}{k}\rangle \right)^2\right)$, there exists $T-1$  such vectors $r_{k,V}(w)\in R^d$ ($2\leq k\leq T$) such that,
    \begin{align*}
        h(w,V)=2 \sum_{k=2}^{T} r_{k,V}(w)\vecpart{\alpha_V(w)}{k}
    \end{align*}
    As a result, by the fact that every sub gradient of $b_{k,V}(w)\in \partial \left( \vecpart{\alpha_V(w)}{k}\right)$ is either 0 or $\lambda_1 u_1 +\lambda_2 u_2+\ldots+\lambda_p u_p$ for $\sum_{i}\lambda_i\leq 1$, such that for all every $j,k$, $u_j\in U$ and $\vecpart{\alpha_V(w)}{k}=\langle u_j,\vecpart{w}^{k}\rangle$,
    combining by the facts that for distinct $k,k'$, $r_{k,V},r_{k',V}$ are orthogonal, it holds, for $u_j^2,\ldots u_j^{T}\in U$ such that for every $k$, $\vecpart{\alpha_V(w)}{k}=\langle u_j^k,\vecpart{w}{k}\rangle$,
    \begin{align*}
        \|h(w,V)\|&=\|2 \sum_{k=2}^{T} r_{k,V}(w)\vecpart{\alpha_V(w)}{k}\rangle\|
        \\&=\|2 \sum_{k=2,b_{k,V}}^{T} r_{k,V}(w)\vecpart{\alpha_V(w)}{k}\rangle\|
        \\&= 2\|\sum_{k=2}^{T} r_{k,V} \langle u_j^k,\vecpart{w}{k}\rangle\|.
    \end{align*}
    Now, we denote $c_j^k(w)\in R^d$ the vector with $\vecpart{c_j^k(w)}{k}=u_j^k$ and for $j\neq k$, $\vecpart{c_j^k(w)}{j}=0$, and,
    \begin{align*}
        \|h(w,V)\|&=2\|\sum_{k=2}^{T} r_{k,V} \langle c_j^k,w\rangle\|
        \\&\leq 2\sqrt{\langle \sum_{k=2}^{T} r_{k,V} \langle c_j^k,w\rangle, \sum_{l=2}^{T} r_{l,V} \langle c_j^l,w\rangle\rangle}
        \\&= 2\sqrt{\sum_{k=2}^{T}\|r_{k,V}\|^2  \langle c_j^k ,w\rangle^2}
         \\&\leq 2\sqrt{\sum_{k=2}^{T}\langle u_j^k ,\vecpart{w}{k}\rangle^2}
         \\&=2\sqrt{\sum_{k=2}^T{\vecpart{\alpha_V(w)}{k}}^{2}}.
    \end{align*}
    The lemma follows.
\end{proof}

\subsection{Proof of algorithm's dynamics}
\label{sec:proofs_alg_GD}
In this section we describe the dynamics of GD when applied on $\gdindempf$ for training set $S$ that is sampled from a distribution $\gdindD$.
We begin with showing that the good event $\cE$ (\cref{good_event_gd}) occurs with a constant probability.

\begin{proof}[of \cref{intersection_events_gd}]
By the fact that every $V_i$ and $j_i$ are independent, it is enough to show that \[\Pr\left(\bigcup_{i=1}^nV_i\neq U_{d '}\right)\geq\frac{1}{2},\]and,\[\Pr\left(\text{for every $k\neq l$, $j_k\neq j_l$}\right)\geq \frac{1}{3}.\]
For the former,
for every $u\in U_{d'}$, since $V_i$ are sampled independently and every vector $u\in U_{d'}$ is in every $V_i$ with probability $\frac{1}{2}$,
\begin{align*}
    \Pr\left(u \in\bigcup_{i=1}^nV_i \right)&=
    1-\Pr\left(u \notin\bigcup_{i=1}^nV_i \right)=1-2^{-n},
\end{align*}
thus, since by \cref{lem:set_direc_exists}, $|U_{d'}|\geq \frac{d'}{178}=n$ , it holds that,
\begin{align*}
\Pr\left(\bigcup_{i=1}^nV_i =U_{d'} \right)
    &=
    \Pr\left(\forall u\in U_{d'} \ u \in\bigcup_{i=1}^nV_i \right)
    \\
    &=
    \left(1-2^{-n}\right)^{|U_{d'}|}
    \\&\leq
    \left(1-2^{-n}\right)^{2^\frac{d'}{178}}
    \\&=
    \left(1-2^{-n}\right)^{2^n}
    \\
    &\leq \frac{1}{e}
    \\&<\frac{1}{2}
    .
\end{align*}
We conclude,
\[
    \Pr\left(\bigcup_{i=1}^nV_i\neq U_{d'}\right)\geq \frac{1}{2}
    .
\]
For the latter, since all $j_i$s are sampled independently, for a single pair $k\neq l$, it holds that
    \[\Pr(j_k\neq j_l)=1-\frac{1}{n^2}\]
   As a result,
   \[\Pr\left(\text{for every $k\neq l$, $j_k\neq j_l$}\right)=\left(1-\frac{1}{n^2}\right)^{\frac{n(n-1)}{2}}\geq \left(1-\frac{1}{n^2}\right)^\frac{n^2}{2}\geq \frac{1}{\sqrt{2e}}\geq \frac{1}{e}.
   \qedhere\]
\end{proof}

From now on, we analyze the dynamics of the GD conditioned on $\cE$ (\cref{good_event_gd}). We begin with several lemmas.
\begin{proof}[of \cref{gradient_uc}]
For the first part, we know that, for every $2\leq k\leq T$, $\vecpart{w}{k}=c\eta u_0$ for $c\leq \frac{1}{2}$. In addition, by the facts that $u_0\in U\setminus \bigcup_{i=1}^n V_i$ and that for every $u\neq v\in U$, it holds that $\langle u,v\rangle\leq \frac{1}{8}$, we get for every $i$, $\max_{u\in V_i}\langle u_0,u\rangle\leq \frac{1}{8}$, thus, for every $i$ and $k\geq 2$,
\begin{align*}
     \max_{u\in V_i} u\vecpart{w}{k}=\max_{u\in V_i}\langle u,c\eta u_{0}\rangle \leq\frac{1}{8}\cdot c\eta\leq\frac{\eta}{16}.
\end{align*}

For the second part, for every sub-gradient $g(w,V_i)\in \partial \gdindgam_1(w,V_i)$, there exists a sub gradient  $h(w,V_i)\in \partial\left(\sum_{k=2}^T\max\left(\frac{3\eta}{32}, \max_{u\in V} \langle u\vecpart{w}{k}\rangle \right)^2\right)$ such that
    \begin{align*}
g(w,V_i)=%
        \frac{h(w,V_i)}{2\sqrt{\sum_{k=2}^T\max\left(\frac{3\eta}{32}, \max_{u\in V_i} \langle u\vecpart{w}{k}\rangle \right)^2}}
        .
    \end{align*} Then, since for every $k$, it holds that $\max_{u\in U}  \langle \vecpart{w}{k}, u_0 \rangle\leq \frac{\eta}{16}$, every such sub-gradient $h(w,V_i)$ is zero, $\nabla \gdindgam_{1}(w,V_i)=0$.
\end{proof}

\begin{proof}[of \cref{gradient_decode}]
    First, for the first part, by
\cref{diff_set_encode exists_GD}, 
the fact that for every $\psi$,
$\|\alpha(\psi)\|\leq 1$, and by $\|\vecpart{w}{1}\|\leq \eta$, for every $\psi\in \Psi$, $\psi^*=\frac{1}{n}\sum_{i=1}^n \phi(V_i,j_{i})$ holds,
    \begin{align*}
    \langle\wenc, \psi^*\rangle - \frac{1}{4}\frac{\epsilon}{T^2}\langle \alpha(\psi^*), \vecpart{w}{1}\rangle
    &\geq \langle\frac{\eta}{n}\sum_{i=1}^n\phi(V_i,j_{i}), \psi^*\rangle - \frac{\eta \epsilon}{4}
    \\&\geq \eta\langle\frac{1}{n}\sum_{i=1}^n\phi(V_i,j_{i}), \psi^*\rangle - \frac{\eta \epsilon}{4}
    \\&\geq \eta\langle\frac{1}{n}\sum_{i=1}^n\phi(V_i,j_{i}), \psi\rangle +\eta\epsilon -\frac{\eta \epsilon}{4} \tag{\cref{diff_set_encode exists_GD}}
     \\&= \eta\langle\frac{1}{n}\sum_{i=1}^n\phi(V_i,j_{i}), \psi\rangle +\frac{\eta\epsilon}{2}+\frac{\eta \epsilon}{4}
    \\&>  \langle\wenc, \psi\rangle - \frac{1}{4}\frac{\epsilon}{T^2}\langle \alpha(\psi), \vecpart{w}{1}\rangle + \frac{\eta\epsilon}{4},
\end{align*}
thus,
\begin{align*}
    \arg\max_{\psi\in \Psi} \left(\langle\wenc, \psi\rangle - \frac{1}{4}\frac{\epsilon}{T^2}\langle \alpha(\psi), \vecpart{w}{1}\rangle\right)=\psi^*=\frac{1}{n}\sum_{i=1}^n\phi(V_i,j_{i}).
\end{align*}
For the second part, by the fact that $\epsilon<\frac{1}{n}$ and \cref{diff_set_encode exists_GD},
\begin{align*}
    \langle\wenc, \psi^*\rangle - \frac{1}{4}\frac{\epsilon}{T^2}\langle \alpha(\psi^*), \vecpart{w}{1}\rangle 
    \geq \frac{7\eta}{8n}-\frac{\eta}{4n}
    > \frac{\eta}{2n}+\frac{\eta}{16n}=\delta_1+\frac{\eta}{16n}.
\end{align*}
Now, by $\cE$, for $u_0=\alpha(\psi^*)$, which is the $u$ with the minimal index in $U\setminus\bigcup_{i=1}^n V_i$,
\begin{align*}
    \alpha(\psi^*)=u_0\in U\setminus\bigcup_{i=1}^n V_i.
\end{align*}
As a result, by the fact that the maximum is attained uniquely at $\psi^*$, we derive that,
\begin{align*}
   \vecpart{\nabla \gdindgam_{3}(w)}{\indsec}=\begin{cases}
        \frac{1}{n}\sum_{i=1}^n \phi(V_i,j_{i}) &\quad \indsec=\last\\
        -\frac{1}{4}\frac{\epsilon}{T^2}u_0 &\quad \indsec=1\\    
        0 &\quad \text{otherwise}.
    \end{cases}
\end{align*}
\end{proof}
\begin{proof}[of \cref{gradient_stab}]
We show that the maximum is attained uniquely at $k=m$ and $u=u_0$.
For $k=1$ and every $u\in U$, 
\begin{align*}
    \frac{3}{8}\langle u,\vecpart{w_t}{1}\rangle-\frac{1}{2}\langle u,\vecpart{w_t}{2}\rangle=\frac{3}{8}\langle u,c\eta u_0\rangle-\frac{1}{2}\langle u,\frac{\eta}{8}u_0\rangle\leq \frac{9\eta}{512}+\frac{\eta}{128}=\frac{13\eta}{512}.
\end{align*}
Moreover, for every $2\leq \indsec\leq m-2$ and every $u\in U$,
\begin{align*}
    \frac{3}{8}\langle u,\vecpart{w}{\indsec}\rangle-\frac{1}{2}\langle u,\vecpart{w}{\indsec+1}\rangle=\frac{3}{8}\langle u,\frac{\eta}{8}u_{0}\rangle-\frac{1}{2}\langle u,\frac{\eta}{8}u_{0}\rangle\leq \frac{3\eta}{64}+\frac{\eta}{128}=\frac{7\eta}{128}.
\end{align*}
For $\indsec= m-1$ and every $u\in U$,
\begin{align*}
    \frac{3}{8}\langle u,\vecpart{w}{m-1}\rangle-\frac{1}{2}\langle u,\vecpart{w}{m}\rangle=\frac{3}{8}\langle u,\frac{\eta}{8}u_{0}\rangle-\frac{1}{2}\langle u,\frac{\eta}{2}u_{0}\rangle\leq \frac{3\eta}{64}+\frac{\eta}{32}=\frac{5\eta}{64}.
\end{align*}
For $\indsec=m$ and $u=u_{0}$,
\begin{align*}
    \frac{3}{8}\langle u,\vecpart{w}{m}\rangle-\frac{1}{2}\langle u,\vecpart{w}{m+1}\rangle=\frac{3}{8}\langle u_0,\frac{\eta}{2}u_{0}\rangle-\frac{1}{2}\langle u_0,0\rangle= \frac{3\eta}{16}.
\end{align*}
For $\indsec=m$ and $u\neq u_0$,
\begin{align*}
    \frac{3}{8}\langle u,\vecpart{w}{m}\rangle-\frac{1}{2}\langle u',\vecpart{w}{m+1}\rangle=\frac{3}{8}\langle u,\frac{\eta}{2}u_{0}\rangle-\frac{1}{2}\langle u',0\rangle\leq  \frac{3\eta}{128}.
\end{align*}
For every $m+1\leq \indsec< T-1$ and every $u\in U$,
\begin{align*}
     \frac{3}{8}\langle u,\vecpart{w}{\indsec}\rangle-\frac{1}{2}\langle u',\vecpart{w}{\indsec+1}\rangle=0.
\end{align*}
Moreover, since $T\geq 4,\eta<1,\epsilon<1$, $\delta_1\leq \frac{3\eta}{1024}$, and 
\[\frac{3}{8}\langle u,\vecpart{w}{m}\rangle-\frac{1}{2}\langle u,\vecpart{w}{m+1}\rangle=\frac{3\eta}{16}>\delta_2+\frac{\eta}{64}.\]
We derive that, 
\begin{align*}
    \vecpart{\nabla \gdindgam_4(w)}{\indsec}=
    \begin{cases}
         \frac{3}{8}u_0 &\quad \indsec=m\\
          -\frac{1}{2}u_0 &\quad \indsec=m+1\\
        0 & \text{otherwise}
        .\end{cases}
\end{align*}
\end{proof}

\begin{lemma}
\label{diff_GD_w_2}
Under the conditions of \cref{nonspec_lower_bound_GD}, if $\cE$ occurs and $w_t$ is %
the iterate of Unprojected GD on $\gdindempf$, with step size $\eta\leq\frac{1}{\sqrt{T}}$ and $w_1=0$, %
then, for $t=2$ it holds that,
\begin{align*}
    \vecpart{w_2}{\indsec}=\begin{cases}
        \frac{\eta}{n}\sum_{i=1}^n\phi(V_i,j_i) &\quad \indsec=\last\\
        0 &\quad \text{otherwise}.
    \end{cases}
\end{align*}
\end{lemma}

\begin{proof}
    For $t=1$, $w_1=0$.
    By \cref{gradient_uc} we know that for every $i$, $\nabla\gdindgam_{1}(w_1,V_i)=0$. Moreover, by
    the fact that $\delta_1,\delta_2>0$ the maximum in $\gdindgam_3$ and $\gdindgam_4$ is attained in $\delta_1$ and $\delta_2$, respectively, thus
    we get that \[\nabla\gdindgam_{3}(w_1)=\nabla\gdindgam_{4}(w)=0\]
    As a result, 
    \begin{align*}
    \vecpart{\nabla \gdindempf(w_1)}{\indsec}=\frac{1}{n}\sum_{i=1}^n\vecpart{\nabla \gdindgam_{2}(w_1,(V_i,j_i))}{\indsec}=\begin{cases}
        -\frac{1}{n}\sum_{i=1}^n(V_i,j_{i}) &\quad \indsec=\last\\
        0 &\quad \text{otherwise},
    \end{cases}
\end{align*}
and,
\begin{align*}
    \vecpart{w_2}{\indsec}=\begin{cases}
        \frac{\eta}{n}\sum_{i=1}^n\phi(V_i,j_i) &\quad \indsec=\last\\
        0 &\quad \text{otherwise}.
    \end{cases}
\end{align*}
\end{proof}
\begin{lemma} \label{diff_GD_w_3}
    Under the conditions of \cref{nonspec_lower_bound_GD}, if $\cE$ occurs and $w_t$ is %
the iterate Unprojected GD on $\gdindempf$, with step size $\eta\leq\frac{1}{\sqrt{T}}$ and $w_1=0$, %
then, for $t=3$ it holds that,
    \begin{align*}
    \vecpart{w_3}{\indsec}=
    \begin{cases}
        \frac{\eta}{4}\frac{\epsilon}{T^2}u_0&\quad \indsec=1\\
        0 &\quad 2\leq \indsec\leq T\\
        \frac{\eta}{n}\sum_{i=1}^n\phi(V_i,j_i) &\quad \indsec=\last
        ,\end{cases}
\end{align*}where $u_0\in U \setminus \bigcup_{i=1}^n V_i$.
\end{lemma}
\begin{proof}
By \cref{diff_GD_w_2}, $\vecpart{w_2}{1},...\vecpart{w_2}{T}=0$, thus,
by \cref{gradient_uc}, we know that for every $i$, $\nabla \gdindgam_{1}(w_1,V_i)=0$.
Moreover, by the fact that $\delta_2>0$, we get that
$\nabla \gdindgam_{4}(w_2)=0$.
For $\gdindgam_{3}(w_2)$, by \cref{gradient_decode}, using the fact that $\vecpart{w_2}{1}=0$ and $\wenct{2}=\frac{\eta}{n}\sum_{i=1}^n\phi(V_i,j_i)$, we get that \begin{align*}
   \vecpart{\nabla \gdindgam_{3}(w_2)}{\indsec}=\begin{cases}
        \frac{1}{n}\sum_{i=1}^n \phi(V_i,j_{i}) &\quad \indsec=\last\\
        -\frac{1}{4}\frac{\epsilon}{T^2}u_0 &\quad \indsec=1\\    
        0 &\quad \text{otherwise}.
    \end{cases}
\end{align*}

For $\gdindgam_{2}(w_2)$, for every $i$, the gradient is
\begin{align*}
   \vecpart{\nabla \gdindgam_{2}(w_2,(V_i,j_i))}{\indsec}=\begin{cases}
    -\phi(V_i,j_{i}) &\quad \indsec=\last\\
        0 &\quad \text{otherwise}.
    \end{cases}
\end{align*}

Combining all together, we conclude that, for $u_0\in U\setminus \bigcup_{i=1}^{n}V_{i}$, it holds that,
\begin{align*}
    \vecpart{\nabla \gdindempf(w_2)}{\indsec}=
    \begin{cases}
         -\frac{1}{4}\frac{\epsilon}{T^2}u_0 &\quad \indsec=1\\
        0 &\quad 2\leq \indsec\leq T
        \\
        0 &\quad \indsec=\last
        ,\end{cases}
\end{align*}
and
\begin{align*}
 \vecpart{w_3}{\indsec}=
    \begin{cases}
       \frac{\eta}{4}\frac{\epsilon}{T^2}u_0&\quad \indsec=1\\
        0 &\quad 2\leq \indsec\leq T\\
        \frac{\eta}{n}\sum_{i=1}^n\phi(V_i,j_i) &\quad \indsec=\last
        .\end{cases}
\end{align*}
\end{proof}
\begin{lemma} \label{diff_GD_w_4}
    Under the conditions of \cref{nonspec_lower_bound_GD}, if $\cE$ occurs and $w_t$ is %
the iterate Unprojected GD on $\gdindempf$, with step size $\eta\leq\frac{1}{\sqrt{T}}$ and $w_1=0$, %
then, for $t=4$ it holds that,
    \begin{align*}
    \vecpart{w_4}{\indsec}=
    \begin{cases}
        -\frac{3\eta}{8}u_0+\frac{\eta}{2}\frac{\epsilon}{T^2}u_0&\quad \indsec=1\\
         \frac{\eta}{2}u_0 &\quad \indsec=2\\
        0 &\quad 3\leq \indsec\leq T
        \\
        \frac{\eta}{n}\sum_{i=1}^n\phi(V_i,j_i) &\quad \indsec=\last
        ,\end{cases}
\end{align*}
where $u_0\in U\setminus \bigcup_{i=1}^n V_i$.%
\end{lemma}
\begin{proof}
We start with $\gdindgam_{1},\gdindgam_{2},\gdindgam_{3}$.
For $\gdindgam_1$, by \cref{diff_GD_w_3}, for every $2\leq \indsec\leq T$, $\vecpart{w_3}{\indsec}=0$, thus, by \cref{gradient_uc}, we know that for every $i$, $\nabla\gdindgam_{1}(w_1,V_i)=0$.
For $\gdindgam_2$, we know that, for every $i$,
\begin{align*}
   \vecpart{\nabla \gdindgam_{2}(w_3,(V_i,j_i))}{\indsec}=\begin{cases}
        -
    \phi(V_i,j_{i}) &\quad \indsec=\last\\
        0 &\quad \text{otherwise}.
    \end{cases}
\end{align*}
For $\gdindgam_{3}$, by \cref{gradient_decode}, using the fact that $\vecpart{w_3}{1}=c\eta u_0$ for $|c|\leq1$ and $\wenct{3}=\frac{\eta}{n}\sum_{i=1}^n\phi(V_i,j_i)$, we get that, \begin{align*}
   \vecpart{\nabla \gdindgam_{3}(w_3)}{\indsec}=\begin{cases}
        \frac{1}{n}\sum_{i=1}^n \phi(V_i,j_{i}) &\quad \indsec=\last\\
        -\frac{1}{4}\frac{\epsilon}{T^2}u_0 &\quad \indsec=1\\    
        0 &\quad \text{otherwise}.
    \end{cases}
\end{align*}

Now, for $\gdindgam_{4}$, we show that the maximum is attained uniquely in $\indsec=1$ and $u=u_0$:
For $\indsec\neq 1$, for every $u\in U$
    \[\frac{3}{8}\langle u,\vecpart{w_3}{\indsec}\rangle-\frac{1}{2}\langle u,\vecpart{w_3}{\indsec+1}\rangle=0.\]
    For $\indsec=1$ and $u\neq u_0$,
    \begin{align*}
        \frac{3}{8}\langle u,\vecpart{w_3}{\indsec}\rangle-\frac{1}{2}\langle u,\vecpart{w_3}{\indsec+1}\rangle&=\frac{3}{8}\langle u,\vecpart{w_3}{1}\rangle-\frac{1}{2}\langle u,\vecpart{w_3}{2}\rangle\\&=\frac{3}{8}\langle u,\frac{\eta}{4}\frac{\epsilon}{T^2}u_0\rangle 
        \\&\leq \frac{3\eta}{256}\frac{\epsilon}{T^2}
    \end{align*}
     For $k=1$ and $u= u_0$,
    \begin{align*}
        \frac{3}{8}\langle u,\vecpart{w_3}{\indsec}\rangle-\frac{1}{2}\langle u,\vecpart{w_3}{\indsec+1}\rangle&=\frac{3}{8}\langle u_0,\vecpart{w_3}{1}\rangle-\frac{1}{2}\langle u_0,\vecpart{w_3}{2}\rangle\\&=\frac{3}{8}\langle u_0,\frac{\eta}{4}\frac{\epsilon}{T^2}u_0\rangle 
        \\&= \frac{3\eta}{32}\frac{\epsilon}{T^2}\\&>\delta_2
    \end{align*}
We derive that, 
\begin{align*}
    \vecpart{\nabla \gdindgam_4(w_3)}{\indsec}=
    \begin{cases}
         \frac{3}{8}u_0 &\quad \indsec=1\\
          -\frac{1}{2}u_0 &\quad \indsec=2\\
        0 &\quad 3\leq \indsec\leq T
        \\
        0 &\quad \indsec=\last
        .\end{cases}
\end{align*}
Combining all together, we get that,
\begin{align*}
    \vecpart{\nabla \gdindempf(w_3)}{\indsec}=
    \begin{cases}
         \frac{3}{8}u_0 -\frac{1}{4}\frac{\epsilon}{T^2}u_0 &\quad \indsec=1\\
          -\frac{1}{2}u_0 &\quad \indsec=2\\
        0 &\quad 3\leq \indsec\leq T
        \\
        0 &\quad \indsec=\last
        ,\end{cases}
\end{align*}
and
\begin{align*}
    \vecpart{w_4}{\indsec}=
    \begin{cases}
        -\frac{3\eta}{8}u_0+\frac{\eta}{2}\frac{\epsilon}{T^2}u_0&\quad \indsec=1\\
         \frac{\eta}{2}u_0 &\quad \indsec=2\\
        0 &\quad 3\leq s\leq T
        \\
        \frac{\eta}{n}\sum_{i=1}^n\phi(V_i,j_i) &\quad \indsec=\last
        ,\end{cases}
\end{align*}
where $u_0\in U\setminus \bigcup_{i=1}^n V_i$.
\end{proof}

\begin{lemma} \label{diff_GD_w_5}
    Under the conditions of \cref{nonspec_lower_bound_GD}, if $\cE$ occurs and $w_t$ is %
the iterate Unprojected GD on $\gdindempf$, with step size $\eta\leq\frac{1}{\sqrt{T}}$ and $w_1=0$, %
then, for $t=5$ it holds that,
    \begin{align*}
    \vecpart{w_{5}}{\indsec}=\begin{cases}
     -\frac{3}{8}\eta u_{0}+ \frac{3\eta}{4}\frac{\epsilon}{T^2}u_0 & \indsec=1\\
        \frac{1}{8}\eta u_{0} &\quad \indsec=2
        \\
        \frac{1}{2}\eta u_{0} &\quad \indsec=3
        \\
        0 &\quad 4\leq s\leq T
         \\
        \frac{1}{n}\sum_{i=1}^n\phi(V_i,j_i) & \indsec=\last,
    \end{cases}
\end{align*}
where $u_0\in U \setminus \bigcup_{i=1}^n V_i$.
\end{lemma}
\begin{proof}
    We begin with $\gdindgam_{1},\gdindgam_{2},\gdindgam_{3}$.
Note that, by \cref{diff_GD_w_4}, for every $2\leq \indsec \leq T$, $\vecpart{w_4}{\indsec}=c\eta u_0$ for $c\leq\frac{1}{2}$, thus, by \cref{gradient_uc}, for every $i$, $\nabla \gdindgam_{1}(w_4, V_i)=0$.
For $\gdindgam_2$, we know that, for every $i$,
\begin{align*}
   \vecpart{\nabla \gdindgam_{2}(w_4,(V_i,j_i))}{\indsec}=\begin{cases}
        -
    \phi(V_i,j_{i}) &\quad \indsec=\last\\
        0 &\quad \text{otherwise}.
    \end{cases}
\end{align*}
For $\gdindgam_{3}$, by \cref{gradient_decode}, using \cref{diff_GD_w_4}, where we showed that $\vecpart{w_4}{1}=c\eta u_0$ for $|c|\leq1$ and $\wenct{4}=\frac{\eta}{n}\sum_{i=1}^n\phi(V_i,j_i)$, we get that, \begin{align*}
   \vecpart{\nabla \gdindgam_{3}(w_4)}{\indsec}=\begin{cases}
        \frac{1}{n}\sum_{i=1}^n \phi(V_i,j_{i}) &\quad \indsec=\last\\
        -\frac{1}{4}\frac{\epsilon}{T^2}u_0 &\quad \indsec=1\\    
        0 &\quad \text{otherwise}.
    \end{cases}
\end{align*}

It is left to calculate $\nabla \gdindgam_{4}(w_4)$. We show that the maximum is attained uniquely at $\indsec=2$ and $u=u_0$. First,
\begin{align*}
\frac{3}{8}\langle u,\frac{\eta}{2}\frac{\epsilon}{T^2} u_0\rangle= \frac{3\eta}{16}\frac{\epsilon}{T^2}\langle u,u_0\rangle
\leq \frac{3\eta}{16T^2},
\end{align*}
thus, since $T\geq 4$,
\begin{align*}
    \frac{3}{8}\langle u,\vecpart{w_4}{1}\rangle-\frac{1}{2}\langle u,\vecpart{w_4}{2}\rangle=\frac{3}{8}\langle u,-\frac{3\eta}{8}u_0+\frac{\eta}{2}\frac{\epsilon}{T^2}u_0\rangle-\frac{1}{2}\langle u,\frac{\eta}{2}u_0\rangle\leq \frac{9\eta}{512}+\frac{\eta}{32}+\frac{9\eta}{256}=\frac{43\eta}{512}<\frac{3\eta}{16}.
\end{align*}
For $\indsec=2$ and $u=u_{0}$,
\begin{align*}
    \frac{3}{8}\langle u,\vecpart{w_4}{2}\rangle-\frac{1}{2}\langle u,\vecpart{w_4}{3}\rangle=\frac{3}{8}\langle u_0,\frac{\eta}{2}u_{0}\rangle-\frac{1}{2}\langle u_0,0\rangle= \frac{3\eta}{16}(> \delta_2).
\end{align*}
For $\indsec=2$ and $u\neq u_{t-2}$,
\begin{align*}
    \frac{3}{8}\langle u,\vecpart{w_4}{2}\rangle-\frac{1}{2}\langle u',\vecpart{w_3}{3}\rangle=\frac{3}{8}\langle u,\frac{\eta}{2}u_{0}\rangle-\frac{1}{2}\langle u,0\rangle\leq  \frac{3\eta}{128}.
\end{align*}
For every $3\leq \indsec\leq T-1$,
\begin{align*}
     \frac{3}{8}\langle u,\vecpart{w_4}{\indsec}\rangle-\frac{1}{2}\langle u',\vecpart{w_4}{\indsec+1}\rangle=0.
\end{align*}
We derive that, 
\begin{align*}
    \vecpart{\nabla \gdindgam_4(w_4)}{\indsec}=
    \begin{cases}
         \frac{3}{8}u_0 &\quad \indsec=2\\
          -\frac{1}{2}u_0 &\quad \indsec=3\\
        0 &\quad 3\leq \indsec\leq T
        \\
        0 &\quad \indsec=\last
        .\end{cases}
\end{align*}
Combining all together, we get that,
\begin{align*}
    \vecpart{\nabla \gdindempf(w_4)}{\indsec}=\begin{cases}
        -\frac{1}{4}\frac{\epsilon}{T^2} u_{0} & \indsec=1\\
        \frac{3}{8} u_{0} &\quad \indsec=2
        \\
        -\frac{1}{2} u_{0} &\quad \indsec=3
        \\
        0 &\quad 4\leq \indsec\leq T
         \\
        0 & \indsec=\last
    \end{cases}
\end{align*}
and 
\begin{align*}
    \vecpart{w_{5}}{\indsec}=\begin{cases}
     -\frac{3}{8}\eta u_{0}+ \frac{3\eta}{4}\frac{\epsilon}{T^2}u_0 & \indsec=1\\
        \frac{1}{8}\eta u_{0} &\quad \indsec=2
        \\
        \frac{1}{2}\eta u_{0} &\quad \indsec=3
        \\
        0 &\quad 4\leq s\leq T
         \\
        \frac{1}{n}\sum_{i=1}^n\phi(V_i,j_i) & \indsec=\last,
    \end{cases}
\end{align*}
where $u_0\in U \setminus \bigcup_{i=1}^n V_i$.

\end{proof}

\begin{lemma}
    \label{bounded_norm_gd}
    Assume the conditions of \cref{nonspec_lower_bound_GD}, and consider the iterates of unprojected GD on $\gdindempf$, with step size $\eta \leq \ifrac{1}{\sqrt{T}}$ initialized at $w_1=0$. 
Under the event $\cE$, we have for all $t\in[T]$ that
\[\|w_t\|\leq1.\]
\end{lemma}
\begin{proof}
If $\cE$ holds, by \cref{diff_GD_w_2,diff_GD_w_3,diff_GD_w_4,diff_GD_expression}, we know that for every $t\geq 2$, $\|\vecpart{w_t}{1}\|\leq \frac{\eta}{2}$, $\|\vecpart{w_t}{t-1}\|\leq \frac{\eta}{2}$ and for every $k\in\{2,\ldots,t-2\}$, $\|\vecpart{w_t}{t-1}\|\leq \frac{\eta}{8}$. As a result,
\begin{align*}
    \norm{w_t}^2
    &\leq \sum_{i=1}^d w_t[i]^2
    \\&\leq \sum_{k=0}^T \|\vecpart{w_t}{k}\|^{2}
    \\&< 2\cdot \left(\frac{\eta}{2}\right)^{2}+ (T-3)\left(\frac{\eta}{8}\right)^{2}+\norm[2]{ \frac{\eta}{n}\sum_{i=1}^n\phi(V_i,j_i) }^2
    \\&\leq \frac{\eta^2}{2} + \frac{\eta^2(T-3)}{64}+\eta^2
    \\&\leq \frac{1}{64}+\frac{3}{2T} 
        \tag{$\eta\leq \frac{1}{\sqrt{T}}$}  
    \\&\leq 1 
        \tag{$T\geq 2$}
\end{align*}
\end{proof}

\subsection{Proof of \cref{nonspec_lower_bound_GD}}
\label{sec:proof_nondiff_GD}

\begin{proof} [of \cref{nonspec_lower_bound_GD}]
    By \cref{intersection_events_gd}, with probability of at least $\frac{1}{6}$, $\cE$ occurs and by \cref{diff_GD_expression}, it holds for every $2\leq \indsec\leq T-3$ that,
    \begin{align}
        \label{suff_exp}
        \vecpart{\suff}{\indsec}=\frac{1}{\indsuff}\sum_{i=1}^\indsuff\vecpart{w_{T-i+1}}{\indsec}&=\begin{cases}
            \frac{\eta}{8}u_0 & \indsec\leq T-\indsuff-2\\
            \frac{1}{\indsuff}\left(\frac{\eta}{2}+\frac{\eta}{8}(T-\indsec-2)\right)u_0 & \indsec\geq T-\indsuff-1       
        \end{cases}
        \\&=\begin{cases}
            \frac{\eta}{8} u_0& \indsec\leq T-\indsuff-2\\
            \frac{\eta(T-\indsec+2)}{8\indsuff}u_0 & \indsec\geq T-\indsuff-1      
        \end{cases}   \notag
    \end{align}
     Then, we denote $\mathbbm{\alpha}_V\in \R^{T-4}$ the vector which its $\indsec$th entry is $\max\left(\frac{3\eta}{32}, \max_{u\in V} \langle u,\vecpart{\suff}{\indsec+1}\rangle \right)$.
    By the fact that every vector $u\in U$ is in $V$ with probability $\frac{1}{2}$, the following holds,
\begin{align*}
\E_V\sqrt{\sum_{\indsec=2}^T\max\left(\frac{3\eta}{32}, \max_{u\in V} \langle u,\vecpart{\suff}{\indsec}\rangle \right)^2}&\geq 
    \E_V\sqrt{\sum_{\indsec=2}^{T-3}\max\left(\frac{3\eta}{32}, \max_{u\in V} \langle u,\vecpart{\suff}{\indsec}\rangle \right)^2}
    \\&= \E_V\sqrt{\sum_{\indsec=1}^{T-4}\max\left(\frac{3\eta}{32}, \max_{u\in V} \langle u,\vecpart{\suff}{\indsec+1}\rangle \right)^2}
     \\&=\E_V \|\alpha_V\|
     \\&\geq \|\E_V \alpha_V\|
     \\&=\sqrt{\sum_{\indsec=2}^{T-3
    }\left(\E_V \max\left(\frac{3\eta}{32}, \max_{u\in V} \langle u,\vecpart{\suff}{\indsec}\rangle \right) \right)^2}
    \end{align*}
    Then, by \cref{suff_exp}, 
    \begin{align*}
        &\E_V\sqrt{\sum_{\indsec=2}^T\max\left(\frac{3\eta}{32}, \max_{u\in V} \langle u,\vecpart{\suff}{\indsec}\rangle \right)^2}\\&\geq 
        \sqrt{\sum_{\indsec=2}^{T-\indsuff-2}\left(\E_V \max\left(\frac{3\eta}{32}, \max_{u\in V} \langle u,\vecpart{\suff}{\indsec}\rangle \right) \right)^2+\sum_{\indsec=T-\indsuff-1}^{T-3
    }\left(\E_V \max\left(\frac{3\eta}{32}, \max_{u\in V} \langle u,\vecpart{\suff}{\indsec}\rangle \right) \right)^2}
    \\&\geq \sqrt{\sum_{\indsec=2}^{T-\indsuff-2}\left(\E_V \max\left(\frac{3\eta}{32}, \max_{u\in V} \langle u,\frac{\eta}{8}u_0\rangle \right) \right)^2+\sum_{\indsec=T-\indsuff-1}^{T-3
    }\left(\E_V \max\left(\frac{3\eta}{32}, \max_{u\in V} \langle u,\frac{\eta(T-\indsec+2)}{8\indsuff}u_0\rangle \right) \right)^2}
    \\&=
    \frac{\eta}{8}\sqrt{\sum_{\indsec=2}^{T-\indsuff-2}\left(\E_V \max\left(\frac{3}{4}, \max_{u\in V} \langle u,u_0\rangle \right) \right)^2+\sum_{\indsec=T-\indsuff-1}^{T-3
    }\left(\E_V \max\left(\frac{3}{4}, \frac{T-\indsec+2}{\indsuff} \max_{u\in V} \langle u,u_0\rangle \right) \right)^2}
     \\&\geq
    \frac{\eta}{8}\sqrt{\sum_{\indsec=2}^{T-\indsuff-2}\left(\E_V \max\left(\frac{3}{4}, \max_{u\in V} \langle u,u_0\rangle \right) \right)^2+\sum_{\indsec=T-\indsuff-1}^{T-3
    }\left(\E_V \max\left(\frac{3}{4}, \frac{T-\indsec+2}{T} \max_{u\in V} \langle u,u_0\rangle \right) \right)^2}
     \\&=
    \frac{\eta}{8}\sqrt{\sum_{\indsec=2}^{T-\indsuff-2}\left(\E_V \max\left(\frac{3}{4}, \max_{u\in V} \langle u,u_0\rangle \right) \right)^2+\sum_{\indsec=1}^{\indsuff-1
    }\left(\E_V \max\left(\frac{3}{4}, \frac{\indsec+4}{T} \max_{u\in V} \langle u,u_0\rangle \right) \right)^2}
    \end{align*}
    Now, treating each of the term separately, with probability $\frac{1}{2}$ on $V$, $ \max_{u\in V} \langle u,u_0\rangle\leq \frac{1}{8}$ (otherwise it is $1$), thus, 
    \begin{align*}
       \E_V \max\left(\frac{3}{4}, \max_{u\in V} \langle u,u_0\rangle \right)=\frac{1}{2}\cdot \frac{3}{4} + \frac{1}{2}\cdot 1=\frac{7}{8}
    \end{align*}
    Moreover, if $\indsec\leq \frac{3T}{4}-4$ 
    \begin{align*}
        \E_V \max\left(\frac{3}{4}, \frac{\indsec+4}{T}\max_{u\in V} \langle u,u_0\rangle \right)=\frac{3}{4},
    \end{align*}
    otherwise, 
    \begin{align*}
    \E_V \max\left(\frac{3}{4}, \frac{\indsec+4}{T}\max_{u\in V} \langle u,u_0\rangle \right)&\geq\frac{1}{2}\max\left(\frac{3}{4}, \frac{\indsec+4}{T}\right)+ \frac{1}{2}\cdot \frac{3}{4}
    \\&\geq \frac{3}{8}+ \frac{\indsec+4}{2T}
    \end{align*}
    Then, we get, if $\indsuff\geq T-3$, (note that it implies $l-1 \geq \frac{3T}{4}-4$),
    \begin{align*}
        &\E_V\sqrt{\sum_{\indsec=2}^T\max\left(\frac{3\eta}{32}, \max_{u\in V} \langle u,\vecpart{\suff}{\indsec}\rangle \right)^2}\\&\geq
        \frac{\eta}{8}\sqrt{\sum_{\indsec=1}^{\indsuff-1
    }\left(\E_V \max\left(\frac{3}{4}, \frac{\indsec+4}{T} \max_{u\in V} \langle u,u_0\rangle \right) \right)^2}
    \\&\geq\frac{\eta}{8}\sqrt{\sum_{\indsec:1\leq \indsec\leq \frac{3T}{4}-4}\frac{9}{16}
        +\sum_{\indsec:\frac{3T}{4}-4< \indsec\leq \indsuff-1}\left( \frac{3}{8}+ \frac{\indsec+4}{2T}\right)^2}
        \\&\geq \frac{\eta}{8}\sqrt{\sum_{\indsec:1\leq \indsec\leq \frac{3T}{4}-4}\frac{9}{16}
        +\sum_{\indsec:\frac{3T}{4}< \indsec\leq T}\left( \frac{3}{8}+ \frac{\indsec}{2T}\right)^2}
        \\&\geq \frac{\eta}{8}\sqrt{\frac{27T-144}{64} +\sum_{\indsec:\frac{3T}{4}< \indsec\leq T}\left( \frac{9}{64}+ \frac{3\indsec}{8T}\right)}
         \\&\geq \frac{\eta}{8}\sqrt{\frac{27T-144}{64} +\left( \frac{9T}{256}+  \frac{3}{8T}\sum_{\indsec:\frac{3T}{4}< \indsec\leq T}\indsec
         \right)}
         \\&\geq \frac{\eta}{8}\sqrt{\frac{27T-144}{64} +\left( \frac{9T}{256}+ \frac{3T^2}{16T}-\frac{3(\frac{3T}{4}+1)^2}{16T}
         \right)} \tag{$\frac{i^2}{2}\leq\sum_{i=1}^ni^2\leq \frac{(i+1)^2}{2}$}
         \\&= \frac{\eta}{8}\sqrt{\frac{27T-144}{64} +\left( \frac{9T}{256}+\frac{3T}{16}-\frac{27T}{256}-\frac{3}{16T}-\frac{9}{32}
        \right)}
        \\&= \frac{\eta}{8}\sqrt{\frac{148T}{256}-\frac{45}{32} -\frac{3}{16T}
        }
        \\&\geq \frac{\eta}{8}\sqrt{\frac{147T}{256}}\tag{$T\geq 512\implies \frac{45}{32}+\frac{3}{16T}\leq \frac{T}{256}$}
         \\&\geq  \frac{3\eta}{32}\cdot \frac{101\sqrt{T}}{100}.
    \end{align*}
    Otherwise, if $\indsuff< T-4$, by similar arguments,
    \begin{align*}
        &\E_V\sqrt{\sum_{\indsec=2}^T\max\left(\frac{3\eta}{32}, \max_{u\in V} \langle u,\vecpart{\suff}{\indsec}\rangle \right)^2}
        \\&\geq 
    \frac{\eta}{8}\sqrt{\sum_{\indsec=2}^{T-\indsuff-2}\left(\E_V \max\left(\frac{3}{4}, \max_{u\in V} \langle u,u_0\rangle \right) \right)^2+\sum_{\indsec=T-\indsuff-1}^{T-3
    }\left(\E_V \max\left(\frac{3}{4}, \frac{T-\indsec+2}{T} \max_{u\in V} \langle u,u_0\rangle \right) \right)^2}
        \\&\geq \frac{\eta}{8}\sqrt{\sum_{\indsec=2}^{T-\indsuff-2}\left(\frac{7}{8}\right)^2+\sum_{\indsec:1\leq \indsec\leq \frac{3T}{4}-4}\frac{9}{16}
        +\sum_{\indsec:\frac{T}{2}< \indsec\leq \indsuff+3}\left( \frac{3}{8}+ \frac{\indsec}{2T}\right)^2}
        \\&=  \frac{\eta}{8}\sqrt{\sum_{\indsec=\indsuff+4}^{T}\left(\frac{7}{8}\right)^2+\sum_{\indsec:1\leq \indsec\leq \frac{3T}{4}-4}\frac{9}{16}
        +\sum_{\indsec:\frac{T}{2}< \indsec\leq \indsuff+3}\left( \frac{3}{8}+ \frac{\indsec}{2T}\right)^2}
        \\&\geq 
       \frac{\eta}{8}\sqrt{\sum_{\indsec:1\leq \indsec\leq \frac{3T}{4}-4}\frac{9}{16}
        +\sum_{\indsec:\frac{3T}{4}< \indsec\leq T}\left( \frac{3}{8}+ \frac{\indsec}{2T}\right)^2}
        \tag{$\frac{3}{8}+\frac{\indsec}{2T}\leq \frac{7}{8}$}
         \\&\geq 
          \frac{3\eta}{32}\cdot \frac{101\sqrt{T}}{100}.
    \end{align*}
    Moreover, we notice that for every $t$, $\gdindgam_2(w_t)\geq-\|\wenct{t}\|\geq -\eta$, $\gdindgam_3(w_t)\geq \delta_1$ and  $\gdindgam_4(w_t)\geq \delta_2$, thus,  it holds that,
\begin{align*}
 \gdindpopf(w_{T,l})\geq 
\frac{303\eta}{3200}\sqrt{T}+\delta_1+\delta_2-\eta
&\geq 
\eta\left(\frac{303}{3200}\sqrt{T}-1\right)
\end{align*}
and
\begin{align*}
\gdindpopf(\opt)&\leq
\frac{3\eta}{32}\sqrt{T}+\eta
\end{align*}
Then, with probability of at least $\frac{1}{6}$,
\begin{align*}
    \gdindpopf(w_{T,l})-\gdindpopf(\opt) 
    &\geq \eta (\frac{303}{3200}\sqrt{T}-2-\frac{3}{32}\sqrt{T})
    \\&\geq \eta (\frac{303}{3200}\sqrt{T}-\frac{302}{3200}\sqrt{T})
    \tag{$T\geq 3200^2\implies 2 \leq \frac{2}{3200}\sqrt{T}$}
    \\&= \frac{\eta}{3200}\sqrt{T}.
\end{align*}
    
\end{proof}

%% file: appendix_proof_sgd.tex
\section{Proofs of \cref{sec:sgd_appendix}}
\label{sec:proofs_SGD}
\subsection{Proofs for the full construction}

\begin{lemma}
    \label{diff_set_encode exists_SGD}
    Let $n$, a set $U\in \R^d$. Let $P(U)$ be the power set of U. Then, 
    there exist sets $\{\Psi_1,...\Psi_n\}\subseteq \R^{2n}$, a number $0<\epsilon<\frac{1}{n}$ and two mappings $\phi:P(U) \times [n]\to \R^{2n}$, $\alpha:\R^{2n}\to U$ such that,
    \begin{enumerate}
        \item For every $j\in[n]$ and $V\subseteq U$, $\|\phi\left(V,j\right)\|\leq 1$.
        \item For every $k$, $\psi\in\Psi_k$, $\|\alpha(\psi)\|\leq 1,\|\psi\|\leq 1$.
        \item  Let $V_1,\ldots,V_k\subseteq U$. Then, for every $k$, $\psi^*_k=\frac{1}{n}\sum_{i=1}^k\phi(V_i,i)$ holds,
        \begin{itemize}
            \item For every $\psi\in \Psi_k$, $\psi\neq \psi^*_k$:
            \[\langle \psi^*_k, \frac{1}{n}\sum_{i=1}^k\phi(V_i,i)\rangle \geq \langle \psi, \frac{1}{n}\sum_{i=1}^k\phi(V_i,i)\rangle + \epsilon ;\]
            \item If $\bigcap_{i=1}^kV_i\neq \emptyset$ and $m=\argmin_{i}\{i:v_i \in \bigcap_{i=1}^kV_i$\}, then $\alpha(\psi^*)=v_m\in \bigcap_{i=1}^kV_i$.
           \end{itemize}
         \end{enumerate}
\end{lemma}

\begin{proof}[of \cref{diff_set_encode exists_SGD}]
The construction is similar to \cref{diff_set_encode exists_GD}.
First, we consider an arbitrary enumeration of $P(U)=\{V^1,...V^{|P(U)|}\}$ and define $g:P(U)\to \R^2$, $g(V^i)= \left(\sin\left(\frac{2\pi i}{|P(U)|}\right),\cos\left(\frac{2\pi i}{|P(U)|}\right)\right)$. 
Here, we refer to a vector $a\in \R^{2n}$ as a concatenation of $n$ vectors in $\R^2$, $\vecpart{a}{1},...,\vecpart{a}{n}$. Then, we define $\delta=1-\cos\left(\frac{2\pi}{|P(U)|}\right)$, $\epsilon=\frac{\delta}{n^2}$ and \[\vecpart{\phi(V,j)}{i}=
\left\{
\begin{array}{cc}
    g(V) & i=j \\
    0 & \text{otherwise} 
\end{array}
\right.\]
As a result, for every $V_i,j$ it holds that
\begin{align*}
    \|\phi(V^i,j)\|=\|g(V^i)\|=\sqrt{\sin\left(\frac{2\pi i}{|P(U)|}\right)^2+\cos\left(\frac{2\pi i}{|P(U)|}\right)^2}=1
\end{align*}
Moreover, if $j_1\neq j_2$,
\begin{align*}
    \langle \phi(V^i,j_1),\phi(V^i,j_2)\rangle =0,
\end{align*}
and if $i>k$,
\begin{align*}
   \langle \phi(V^i,j),\phi(V^k,j)\rangle=&\langle  g(V^i), g(V^k)\rangle \\&=\sin\left(\frac{2\pi i}{|P(U)|}\right)\sin\left(\frac{2\pi k}{|P(U)|}\right) + \cos\left(\frac{2\pi i}{|P(U)|}\right)\cos\left(\frac{2\pi k}{|P(U)|}\right)
   \\&=\cos\left(\frac{2\pi (i-k)}{|P(U)|}\right)
   \\&\leq \cos\left(\frac{2\pi}{|P(U)|}\right) 
   \tag{$\cos$ is monotonic decreasing in $[0,\pi/2]$}
   \\&=1-\delta
\end{align*}
We notice that $0<\delta< 1$.
Now, we consider an arbitrary enumeration of $U=\{v_1,...v_{|U|}\}$, and define the following sets $\Psi_1,\ldots\Psi_n\subseteq \R^{2n}$ and the following two mappings $\sigma: R^{2n}\to P(U), \alpha:  R^{2n}\to U$,
\[\Psi_k = \{ \frac{1}{n}\sum_{i=1}^k \phi(V_i,i): \forall i\  V_i\ \subseteq U\}\]
Note that, for every $\psi \in \Psi$,
\[\|\psi\|=\|\frac{1}{n}\sum_{i=1}^k \phi(V_i,j_i)\|\leq \frac{1}{n}\sum_{i=1}^k \|\phi(V_i,j_i)\|\leq 1.\]

Then, for every $a\in \R^{2n}$ and $j\in[n]$, we denote the index $q(a,j)\in [|P(U)|]$ as \[q(a,j)=\argmax_{r}\langle g(V_r), \vecpart {a}{j}\rangle,\] and define the following mapping $\sigma:\R^{2n}\to P(U)$,
\[\sigma(a)=\bigcap_{j=1,\vecpart {a}{j}\neq 0}^{n}V_{q(a,j)}.\]
Moreover, for every $a\in \R^{2n}$, we denote the index $p(a)\in [|U|]$ as 
\[p(a)=\argmin_{i}\{i:v_i \in \sigma(a)\},\] and define the following mapping $\alpha:\R^{2n^2}\to U$,
\[\alpha(a)=\left\{\begin{array}{cc}
    v_{|U|} & \sigma(a)=\emptyset \\
     v_{p(a)} & \sigma(a)\neq \emptyset
\end{array}\right..\]
Note that for every $a\in \R^{2n}$, $\alpha(a)\in U$, thus, $\|\alpha(a)\|\leq 1$.

Now, Let $V_1,\ldots, V_n\subseteq U$, $k\in[n]$ and $\psi^*_k=\frac{1}{n}\sum_{i=1}^k \phi(V_i,i)$.
Then, 
\begin{align*}
    \langle \psi^*, \frac{1}{n}\sum_{i=1}^k \phi(V_i,i)\rangle&=
   \langle \frac{1}{n}\sum_{i=1}^k \phi(V_i,i), \frac{1}{n}\sum_{i=1}^k \phi(V_i,i)\rangle \\&= \frac{1}{n^2}\sum_{i=1}^k\langle \phi(V_i,i),\phi(V_i,i)\rangle
   \\&= \frac{k}{n^2}
\end{align*}
For $\psi=\frac{1}{n}\sum_{i=1}^k \phi(V'_i,i)$ such that $\psi\neq \psi^*$, 
there exists a index $r$ such that $V'_r\neq V_r$ ,thus,
\begin{align*}
\langle \psi, \frac{1}{n}\sum_{l=1}^k \phi(V_i,i)\rangle&=
   \langle \frac{1}{n}\sum_{i=1}^k \phi(V'_i,i), \frac{1}{n}\sum_{i=1}^k \phi(V_i,i)\rangle \\&=\frac{1}{n^2}\sum_{i=1}^k\langle \phi(V_i,i),\phi(V'_i,i)\rangle 
\\&\leq \frac{1}{n^2}\left(1-\delta+\sum_{i=1,i\neq r}^k 1 \right)
   \\&\leq \frac{1}{n^2}(1-\delta+k-1)\\&= \frac{k}{n^2}-\frac{\delta}{n^2}
   \\&= \langle \psi^*_k, \frac{1}{n}\sum_{i=1}^n \phi(V_i,j_i)\rangle-\epsilon
\end{align*}
Furthermore, it holds that, $\vecpart {\frac{1}{n}\sum_{i=1}^n\phi(V_i,i)}{i}=\frac{1}{n}g(V_i)$, thus,
\begin{align*}
q\left(\frac{1}{n}\sum_{i=1}^k\phi(V_i,i),i\right)&=\argmax_{r}\langle g(V_r), \vecpart {\frac{1}{n}\sum_{i=1}^k\phi(V_i,i)}{i}\rangle
\\&=\argmax_{r}\langle g(V_r), \frac{1}{n}g(V_i)\rangle
\\&=i, 
\end{align*}
thus, we get,
\begin{align*}
    \sigma (\psi^*)&=\sigma \left(\frac{1}{n}\sum_{i=1}^k\phi(V_i,i)\right)\\
&=\bigcap_{j=1,\vecpart {\frac{1}{n}\sum_{i=1}^k\phi(V_i,i)}{j}\neq 0}^{n}V_{q(\vecpart {\frac{1}{n}\sum_{i=1}^k\phi(V_i,i)}{j},j)}
\\&=\bigcap_{j=1}^{k}V_{q(\frac{1}{n}\sum_{i=1}^k\phi(V_i,i),j)}\tag{The indices that are non-zero are $j=1,\ldots,k$}
\\&=\bigcap_{i=1}^{k}V_{i} %
\end{align*}
Then, assuming that $\bigcap_{i=1}^kV_i\neq  \emptyset$, and let and $m=\argmin_{i}\{i:v_i \in \bigcap_{i=1}^kV_i$\}, $p(a)=m$ and,
\[\alpha(\psi^*)= v_{m}
\in
\bigcap_{i=1}^{k}V_{i}.\]
\end{proof}

\begin{proof}[of \cref{convex_lip_SGD}]
First, $\sgdindgam_1$ is convex and $1$-Lipschitz by the fact that $\sgdindgam_1=\gdindgam_1$ and \cref{convex_lip}. Moreover, by \cref{diff_set_encode exists_SGD}, $\sgdindgam_2$ is a maximum over $1$-Lipschitz linear functions, thus, $\sgdindgam_2$ is convex and $1$-Lipschitz. Finally,  $\sgdindgam_3$ is a summation of two $1$-Lipschitz linear functions, thus, $\sgdindgam_3$ is convex and $2$-Lipschitz.
Combining all together, we get the lemma.
\end{proof}

\subsection{Proof of algorithm's dynamics}
In this section we describe the dynamics of SGD.
We begin with showing that the good event $\cE'$ (\cref{good_event_sgd}) occurs with a constant probability.

\begin{proof} [of \cref{prefix_event}]
First, by union bound, 
 \begin{align*}
        \Pr\left(\forall t\in[n] \ P_t\neq \emptyset
        \text{ and } J_t \in S_t \right)\geq\frac{1}{2}&=1- Pr\left(\exists t \ P_t= \emptyset
        \text{ or } J_t \notin S_t \right)
        \\&\geq 1- \sum_{t=1}^n Pr\left(P_t= \emptyset
        \text{ or } \left(P_t\neq \emptyset \text{ and }J_t \notin S_t\right) \right)
        \\&\geq 1- \sum_{t=1}^n \Pr\left(P_t= \emptyset\right)
        -\sum_{t=1}^n\Pr\left(P_t\neq \emptyset \text{ and }J_t \notin S_t\right)
        \\&= 1- \sum_{t=1}^n\Pr\left(P_t= \emptyset\right)
        -\sum_{t=1}^n\Pr\left(P_t\neq \emptyset\right)\Pr\left(J_t \notin S_t | P_t\neq \emptyset\right).
    \end{align*}
Now,  for every $v_l\in U$,
\begin{align*}
    \Pr(v_l \notin \bigcap_{i=1}^{t-1} V_i)= 1- \Pr(v_l \notin \bigcap_{i=1}^{t-1} V_i)=1-\delta^{t-1},
\end{align*}
and, 
\begin{align*}
    \Pr\left(v_l \notin S_t \right)=1- (1-\delta)^{n-t+1}\leq 1-(1-\delta)^n.
\end{align*}
Then, 
\begin{align*}
    \Pr\left(P_t= \emptyset\right)&=\Pr\left(\bigcap_{i=1}^t V_i =\emptyset\right)
    \\&= \Pr(\forall v_l\in U \ w\notin \bigcap_{i=1}^t V_i)
    \\&=
    (1-\delta^{t-1})^{|U|}
    \\&\leq  (1-\delta^n)^{|U|}.
\end{align*}
Moreover, by the fact that for every $t$, $P_t$ is independent of $V_{t+1},...V_n$,
\begin{align*}
    Pr\left(P_t\neq \emptyset\right)\Pr\left(J_t \notin S_t | P_t\neq \emptyset\right)&= \sum_{l:v_l\in U} \Pr\left(P_t\neq \emptyset\right)\Pr\left(v_l \notin S_t | P_t\neq \emptyset\right)\Pr(J_t=v_l)
    \\&= \sum_{l:v_l\in U} \Pr\left(P_t\neq \emptyset\right)\Pr\left(v_l \notin S_t \right)\Pr(J_t=v_l)
    \\&\leq 1- (1-\delta)^n
\end{align*}
Combining all of the above, we get that,
\begin{align*}
    &\Pr(\forall t\in[n] \ P_t\neq \emptyset
        \text{ and } J_t \in S_t)\\&= 1- \sum_{t=1}^n \Pr\left(P_t= \emptyset\right)
        -\sum_{t=1}^n\Pr\left(P_t\neq \emptyset\right)\Pr\left(J_t \notin S_t | P_t\neq \emptyset\right)
        \\&\geq 1- n (1-\delta^n)^{|U|}- n\left(1- (1-\delta)^n\right).
\end{align*}
For $\delta=\frac{1}{4n^2}$, by the fact that $|U|\geq 2^{\frac{d'}{178}}=2^{4n\log(n)}=n^{4n}$, 
\[|U|\delta^{n}\geq n^{4n}n^{-2n}4^{-n}\geq n^{2n}4^{-n}\geq \log(4n)\]
\begin{align*}
    \Pr\left(\forall t\in[n] \ P_t\neq \emptyset
        \text{ and } J_t \in S_t \right)\geq\frac{1}{2}
        &\geq 1- n (1-\delta^\frac{n}{500})^{|U|}- n\left(1- (1-\delta)^n\right)
        \\&\geq 1- n e^{-|U|\delta^\frac{n}{500}}- n\left(1- (1-n\delta)\right)
        \\&\geq 1- n e^{-\log(4n)}- n^2\delta
        \\&\geq 1- \frac{1}{4} -\frac{1}{4}
        \\&= \frac{1}{2}
        .
\end{align*}
\end{proof}

\begin{proof}[of \cref{gradient_uc_sgd}]
First, by the fact that for every $t\leq k\leq T$, $\vecpart{w}{k}=0$, 
     for every such $k$, \[\max_{u\in V_t} \langle u,\vecpart{w}{k}\rangle=0<\frac{3\eta}{32},\]
     For $2\leq k\leq t-1$, $\vecpart{w}{k}=c\eta u_k$, where $c\leq \frac{1}{2}$ and every $u_k\in \bigcap_{i=k}^T\overline{V_i}\subseteq \overline{V_{t}}$,
     thus, 
     \[\max_{u\in V_t} \langle u,\vecpart{w}{k}\rangle\leq\frac{\eta}{2}\cdot \frac{1}{8}<\frac{3\eta}{32}.\]
    We derive that
    $\nabla \sgdindgam_1(w_t,V_t)= 0$.
\end{proof}

\begin{proof}[of \cref{gradient_decode_sgd}]
First, we show that the maximum of $\sgdindgam_2(w,V)$ is attained with $k=m$ and $u=u_m$.
For $k\geq m+1$, for every $u\in U$ and $\psi\in \Psi_k$,
\begin{align*}
&\frac{3}{8}\langle u,\vecpart{w}{k}\rangle-\frac{1}{2}\langle \alpha(\psi),\vecpart{w}{k+1}\rangle +\langle\vecpart{w}{ \last,k},\frac{1}{4n}\psi\rangle
-\langle\vecpart{w}{ \last,k+1},\frac{1}{4n}\psi\rangle+\langle \vecpart{w}{\last,k+1},-\frac{1}{4n^2}\phi(V,k+1)\rangle=0
.
\end{align*}
For $k=1$, for every $u\in U$ and $\psi\in \Psi_1$, by \cref{diff_set_encode exists_SGD}, we know that for every $\psi,V,j$, $\|\psi\|,\|\phi(V,j)\|\leq 1$, and $\alpha(\psi)\in U$, thus, 
\begin{align*}
&\frac{3}{8}\langle u,\vecpart{w}{k}\rangle-\frac{1}{2}\langle \alpha(\psi),\vecpart{w}{k+1}\rangle +\langle\vecpart{w}{ \last,k},\frac{1}{4n}\psi\rangle
-\langle\vecpart{w}{ \last,k+1},\frac{1}{4n}\psi\rangle+\langle \vecpart{w}{\last,k+1},-\frac{1}{4n^2}\phi(V,k+1)\rangle
\\&=\frac{3c}{8}\langle u_1,u\rangle
    -\frac{\eta}{16}\langle u_2,\alpha(\psi)\rangle + \langle\vecpart{w}{ \last,k},\frac{1}{4n}\psi\rangle - 0+ 0
    \\&\leq 
    \frac{9\eta}{512}
    +\frac{\eta}{128} +\frac{\eta}{4n} 
    \\& < \frac{\eta}{8}\tag{$n\geq 4$}
.
\end{align*}
For $2\leq k\leq m-2$, for every $u\in U$ and $\psi\in \Psi_k$, by \cref{diff_set_encode exists_SGD}, we know that for every $\psi,V,j$, $\|\psi\|,\|\phi(V,j)\|\leq 1$, and $\alpha(\psi)\in U$, thus, 
\begin{align*}
&\frac{3}{8}\langle u,\vecpart{w}{k}\rangle-\frac{1}{2}\langle \alpha(\psi),\vecpart{w}{k+1}\rangle +\langle\vecpart{w}{ \last,k},\frac{1}{4n}\psi\rangle
-\langle\vecpart{w}{ \last,k+1},\frac{1}{4n}\psi\rangle+\langle \vecpart{w}{\last,k+1},-\frac{1}{4n^2}\phi(V,k+1)\rangle
\\&=\frac{3}{64}\langle u_k,u\rangle
    -\frac{\eta}{16}\langle u_{k+1},\alpha(\psi)\rangle +0 - 0+ 0
    \\&\leq \frac{3\eta}{64} +\frac{\eta}{16}
    \\& < \frac{\eta}{8}
.
\end{align*}
For $k= m-1$, for every $u\in U$ and $\psi\in \Psi_{k}$, by \cref{diff_set_encode exists_SGD}, we know that for every $\psi,V,j$, $\|\psi\|,\|\phi(V,j)\|\leq 1$, and $\alpha(\psi)\in U$, thus, 
\begin{align*}
&\frac{3}{8}\langle u,\vecpart{w}{k}\rangle-\frac{1}{2}\langle \alpha(\psi),\vecpart{w}{k+1}\rangle +\langle\vecpart{w}{ \last,k},\frac{1}{4n}\psi\rangle
-\langle\vecpart{w}{ \last,k+1},\frac{1}{4n}\psi\rangle+\langle \vecpart{w}{\last,k+1},-\frac{1}{4n^2}\phi(V,k+1)\rangle
\\&=\frac{3}{64}\langle u_k,u\rangle
    -\frac{\eta}{4}\langle u_{k+1},\alpha(\psi)\rangle +0 - \langle \vecpart{w}{ \last,k+1},\frac{1}{4n}\psi\rangle+  \langle \vecpart{w}{ \last,k+1},\frac{1}{4n^2}\phi(V,m)\rangle
    \\&\leq \frac{3\eta}{64}+\frac{\eta}{32}+\frac{1}{16n^2}+\frac{1}{16n^3}
    \\& < \frac{\eta}{8} \tag{$n\geq 4$}
.
\end{align*}

For $k=m$, $u\neq u_{m}$ and every $\psi\in \Psi_{m}$, by \cref{diff_set_encode exists_SGD}, we know that for every $\psi,V,j$, $\|\psi\|,\|\phi(V,j)\|\leq 1$, and $\alpha(\psi)\in U$, thus, 
\begin{align*}
&\frac{3}{8}\langle u,\vecpart{w}{k}\rangle-\frac{1}{2}\langle \alpha(\psi),\vecpart{w}{k+1}\rangle +\langle\vecpart{w}{ \last,k},\frac{1}{4n}\psi\rangle
-\langle\vecpart{w}{ \last,k+1},\frac{1}{4n}\psi\rangle+\langle \vecpart{w}{\last,k+1},-\frac{1}{4n^2}\phi(V,k+1)\rangle\\&=\frac{3}{8}\langle u,\vecpart{w}{k}\rangle+\langle \frac{1}{4n}\psi,\vecpart{w}{\last,k}\rangle
\\&=\frac{3}{8}\langle u,\frac{\eta}{2}u_{m}\rangle+\langle \frac{1}{4n}\psi,\vecpart{w}{\last, k} \rangle
\\&\leq \frac{3\eta}{128}+\frac{\eta}{16n^2}\\&<\frac{\eta}{32}
\tag{$n\geq4$}.\end{align*}
For $k=m$, $u=u_{m}$ and every $\psi\in \Psi_{m}$, by \cref{diff_set_encode exists_SGD}, we know that for every $\psi,V,j$, $\|\psi\|,\|\phi(V,j)\|\leq 1$, and $\alpha(\psi)\in U$, thus, 
\begin{align*}
&\frac{3}{8}\langle u,\vecpart{w}{k}\rangle-\frac{1}{2}\langle \alpha(\psi),\vecpart{w}{k+1}\rangle +\langle\vecpart{w}{ \last,k},\frac{1}{4n}\psi\rangle
-\langle\vecpart{w}{ \last,k+1},\frac{1}{4n}\psi\rangle+\langle \vecpart{w}{\last,k+1},-\frac{1}{4n^2}\phi(V,k+1)\rangle
\\&=\frac{3}{8}\langle u,\vecpart{w_t}{k}\rangle+\langle \frac{1}{4n}\psi,\vecpart{w_t}{\last,k}\rangle\\&=\frac{3}{8}\langle u,\frac{\eta}{2}u_{m}\rangle+\langle \frac{1}{4n}\psi,\vecpart{w_t}{\last,k}\rangle
\\&\geq \frac{3\eta}{16}-\frac{\eta}{16n^2}\\&>\frac{5\eta}{32}
\tag{$n\geq4$}
\\&>\delta_1.\end{align*}

Second, we show that when $k=m$ and $u=u_m$, the maximum among $\psi\in \Psi_{m}$ is attained uniquely in $\psi^*_{m}=\frac{1}{n}\sum_{t=1}^{m}\phi(V_t,t)$. For any $\psi\in \Psi_m$, with $\psi\neq \psi^*_{m}$, by \cref{diff_set_encode exists_SGD}, for $k=m$, $u=u_m$,
\begin{align*}
&\frac{3}{8}\langle u,\vecpart{w}{m}\rangle-\frac{1}{2}\langle \alpha(\psi^*_{m}),\vecpart{w}{m+1}\rangle +\langle\vecpart{w}{ \last,m},\frac{1}{4n}\psi^*_{m}\rangle
-\langle\vecpart{w}{ \last,m+1},\frac{1}{4n}\psi^*_{m}\rangle+\langle \vecpart{w}{\last,m+1},-\frac{1}{4n^2}\phi(V,m+1)\rangle\\&=\frac{3}{8}\langle u,\vecpart{w}{m}\rangle+\langle \frac{1}{4n}\psi^*_{m},\vecpart{w}{\last,m}\rangle
\\&=\frac{3\eta}{16}+\frac{\eta}{16n^2}\langle \psi^*_{m},\frac{1}{n}\sum_{t=1}^{m}\phi(V_t,t)\rangle
\\&\geq 
\frac{3\eta}{16}+\frac{\eta}{16n^2}\langle \psi,\frac{1}{n}\sum_{t=1}^{m}\phi(V_t,t)\rangle +\frac{\eta\epsilon}{16n^2}\\&= \frac{3}{8}\langle u,\vecpart{w}{k}\rangle+\langle \frac{1}{4n}\psi,\vecpart{w}{\last,m}\rangle +\frac{\eta\epsilon}{16n^2} 
\\&=\frac{3}{8}\langle u,\vecpart{w}{m}\rangle-\frac{1}{2}\langle \alpha(\psi),\vecpart{w}{m+1}\rangle +\langle\vecpart{w}{ \last,m},\frac{1}{4n}\psi\rangle
-\langle\vecpart{w}{ \last,m+1},\frac{1}{4n}\psi\rangle+\langle \vecpart{w}{\last,m+1},-\frac{1}{4n^2}\phi(V,m+1)\rangle+\frac{\eta\epsilon}{16n^2}
\end{align*}
We derive that, 
\begin{align*}
    \vecpart{\nabla \sgdindgam_2(w,V)}{k}=
    \begin{cases}
        \frac{3}{8} u_{m} &\quad k=m
        \\
        -\frac{1}{2} \alpha(\psi_{m}^*) &\quad k=m+1
        \\
        0 &\quad k\notin \{m,m+1\}   
    \end{cases}
\end{align*}
\begin{align*}
    \vecpart{\nabla \sgdindgam_2(w,V)}{\last, k}=
    \begin{cases}
        \frac{1}{4n^2}\sum_{t=1}^{m}\phi(V_t,t)&\quad k= m
        \\ -\frac{1}{4n^2}\sum_{t=1}^{m}\phi(V_t,t)-\frac{1}{4n^2}\phi(V,m+1)&\quad k= m+1
        \\  0&\quad k\notin \{m,m+1\}.
    \end{cases}
\end{align*}
\end{proof}
\begin{lemma}
    \label{sgd_exp_2}
    Under the conditions of \cref{SGD_lower_bound}, if $\cE'$ occurs and $w_t$ is %
the iterate of Unprojected SGD with step size $\eta\leq\frac{1}{\sqrt{n}}$ and $w_1=0$,
    \begin{align*}
    \vecpart{w_2}{k}=\begin{cases}
        \frac{\eta}{n^3}u_1 &\quad k=1
        \\
        0 &\quad k\geq 2
        \end{cases},
\end{align*}
and,
\begin{align*}
    \vecpart{w_2}{0,k}=\begin{cases}
        \frac{\eta}{4n^2}\phi(V_1,1) &\quad k=1
        \\ 0 &\quad k\neq 1.
    \end{cases}
\end{align*}
\end{lemma}
\begin{proof}
    $w_1=0$, thus, for every $k$, \[\max_{u\in V_1} \langle u,\vecpart{w_1}{k}\rangle=0<\frac{3\eta}{32},\]
    and we derive that
    $\nabla \sgdindgam_1(w_1,V_1)= 0$.
    By the same argument, $\nabla \sgdindgam_2(w_1,V_1)= 0$ (where the maximum is attained uniquely in $\delta_2$).
    Moreover, $\sgdindgam_3$ is a linear function, then, we get that,
\begin{align*}
    \vecpart{\nabla \sgdindgam_3(w_1,V_1)}{k}=\begin{cases}
        -\frac{1}{n^3}u_1 &\quad k=1
        \\
        0 &\quad k\geq 2
        \end{cases},
\end{align*}
and,
\begin{align*}
    \vecpart{\nabla \sgdindgam_3(w_1,V_1)}{0,k}=\begin{cases}
        -\frac{1}{4n^2}\phi(V_1,1) &\quad k=1
        \\ 0 &\quad k\neq 1,
    \end{cases}
\end{align*}
and the lemma follows.
\end{proof}
\begin{lemma}
    \label{sgd_exp_3}
    Under the conditions of \cref{SGD_lower_bound}, if $\cE'$ occurs and $w_t$ is %
the iterate of Unprojected SGD with step size $\eta\leq\frac{1}{\sqrt{n}}$ and $w_1=0$,
\begin{align*}
    \vecpart{w_3}{k}=
    \begin{cases}
        \frac{2\eta}{n^3}u_1-\frac{3\eta}{8}u_{1} &\quad k=1\\
        \frac{\eta}{2} u_{2} &\quad k=2
        \\
        0 &\quad 3\leq k\leq n
    \end{cases}
    \end{align*}
    \begin{align*}
    \vecpart{w_3}{0,k}=
    \begin{cases}
        \frac{\eta}{4n^2}\phi(V_2,1) &\quad k=1
        \\ \frac{\eta}{4n^2}\phi(V_1,1) +\frac{\eta}{4n^2}\phi(V_2,2)&\quad k=2
        \\  0&\quad k\geq 3.
    \end{cases}
\end{align*}
    where $u_1 \in U$, and $u_{2}$ holds $u_{2}\in P_2\cap S_2$.
\end{lemma}
\begin{proof}

    First, by the fact that for every $2\leq k\leq T$, $\vecpart{w_2}{k}=0$, 
     for every such $k$, \[\max_{u\in V_2} \langle u,\vecpart{w_2}{k}\rangle=0<\frac{3\eta}{32},\]
    and we derive that
    $\nabla \sgdindgam_1(w_2,V_2)= 0$.
    
Moreover, $\sgdindgam_3$ is a linear function, thus,
    \begin{align*}
    \vecpart{\sgdindgam_3(w_2,V_2)}{k}=\begin{cases}
        \frac{\eta}{n^3}u_1 &\quad k=1
        \\
        0 &\quad k\geq 2
        \end{cases},
\end{align*}
and,
\begin{align*}
   \vecpart{\sgdindgam_3(w_2,V_2)}{k}=\begin{cases}
        \frac{\eta}{4n^2}\phi(V_2,1) &\quad k=1
        \\ 0 &\quad k\neq 1.
    \end{cases}
\end{align*}
For $\sgdindgam_2(w_2,V_2)$, we get by the fact that for every $k\geq 1$, $\vecpart{w_2}{k+1}=\vecpart{w_2}{\last,k+1}=0$,
\begin{align*}
\sgdindgam_2(w_2,V_2)=\max\left(\delta_2,\max_{ k\in[n-1],u\in U,\psi\in \Psi_k}\left(\frac{3}{8}\langle u,\vecpart{w_2}{k}\rangle+\langle \frac{1}{4n}\psi,\vecpart{w_2}{\last,k}\rangle\right)\right)
\end{align*}
As a first step, we show that the the maximum is attained with $k=1$ and $u=u_1$,
For $k\neq 1$, for every $u\in U$ and $\psi\in \Psi_k$,\[\frac{3}{8}\langle u,\vecpart{w_2}{k}\rangle+\langle \frac{1}{4n}\psi,\vecpart{w_2}{\last,k}\rangle=0.\]
For $k=1$, $u\neq u_1$ and every $\psi\in \Psi_1$, by the fact that $\|\psi\|,\|\phi(V_1,1)\|\leq 1$, \[\frac{3}{8}\langle u,\vecpart{w_2}{k}\rangle+\langle \frac{1}{4n}\psi,\vecpart{w_2}{\last,k}\rangle\leq \frac{3\eta}{64n^3}+\frac{\eta}{16n^3}=\frac{7\eta}{64n^3}<\frac{3\eta}{16n^3}.\]
For $k=1$, $u=u_1$ and every $\psi\in \Psi_1$, by the fact that $\|\psi\|,\|\phi(V_1,1)\|\leq 1$,
\[\frac{3}{8}\langle u,\vecpart{w_2}{k}\rangle+\langle \frac{1}{4n^2}\psi,\vecpart{w_2}{\last,k}\rangle\geq\frac{3\eta}{8n^3}-\frac{\eta}{16n^3} > \frac{3\eta}{16n^3}>\delta_1.\]

As a second step we show that the maximum among $\psi\in \Psi_1$ is attained uniquely in $\psi^*_1=\frac{1}{n}\phi(V_1,1)$. For any $\psi\in \Psi_1$, with $\psi\neq \psi^*_1$. By \cref{diff_set_encode exists_SGD}, for $k=1$, $u=u_1$,
\begin{align*}
\frac{3}{8}\langle u,\vecpart{w_2}{k}\rangle+\langle \frac{1}{4n}\psi^*_1,\vecpart{w_2}{\last,k}\rangle&=\frac{3\eta}{8n^3}+\langle \frac{1}{4n}\psi^*_1,\frac{\eta}{4n^2}\phi(V_1,1)\rangle
\\&=\frac{3\eta}{8n^3}+\frac{\eta}{16n^2}\langle \psi^*_1,\frac{1}{n}\phi(V_1,1)\rangle
\\&\geq \frac{3\eta}{8n^3}+\frac{\eta}{16n^2}\langle \psi,\frac{1}{n}\phi(V_1,1)\rangle+\frac{\eta\epsilon}{16n^2}
\\&= \frac{3}{8}\langle u,\vecpart{w_2}{k}\rangle+\langle \frac{1}{4n}\psi,\vecpart{w_2}{\last,k}\rangle +\frac{\eta\epsilon}{16n^2}
\end{align*}
We got that the maximum is uniquely attained at $k=1, u=u_1, \psi=\frac{1}{n}\phi(V_1,1)$.
Now, by \cref{diff_set_encode exists_SGD}, for $j=\argmin_{i}\{i:v_i \in V_1\}$, we get that
\begin{align*}
    \alpha(\psi)=v_{j}\in V_1.
\end{align*}
 We notice that $V_1=P_2$ and thus $\alpha(\psi)=J_2$. Then, 
by $\cE'$, $\alpha(\psi)$ also holds $\alpha(\psi)\in S_{2}$.
Combining the above together, we get, for $u_2=\alpha(\psi)\in P_2\cap S_2$,
\begin{align*}
    \vecpart{\nabla f(w_2,V_2)}{k}=
    \begin{cases}
        \frac{3}{8}u_{1} - \frac{1}{n^3}u_1&\quad k=1\\
        -\frac{1}{2} u_{2} &\quad k=2
        \\
        0 &\quad k\geq 3
        \end{cases}
        \end{align*}
        and,
        \begin{align*}
            \vecpart{\nabla f(w_2,V_2)}{0,k}=
            \begin{cases}\frac{1}{4n^2}\phi(V_1,1)-\frac{1}{4n^2}\phi(V_2,1) &\quad k= 1
        \\ -\frac{1}{4n^2}\phi(V_1,1) -\frac{1}{4n^2}\phi(V_2,2)&\quad k= 2
        \\  0&\quad k\geq 3,
    \end{cases}
\end{align*}
and the lemma follows.
\end{proof}
\begin{lemma}\label{sgd_exp_4}
    Under the conditions of \cref{SGD_lower_bound}, if $\cE'$ occurs and $w_t$ is %
the iterate of Unprojected SGD with step size $\eta\leq\frac{1}{\sqrt{n}}$ and $w_1=0$,
    \begin{align*}
    \vecpart{w_4}{k}=
    \begin{cases}
        \frac{3\eta}{n^3}u_1-\frac{3\eta}{8}u_{1} &\quad k=1\\
         \frac{\eta}{8} u_{2} &\quad k=2
        \\
        \frac{\eta}{2} u_{3} &\quad k=3
        \\
        0 &\quad k\geq 4
        \end{cases},
        \end{align*}
and,
     \begin{align*}
    \vecpart{w_4}{0,k}=
    \begin{cases} \frac{\eta}{4n^2}\phi(V_2,1)+\frac{\eta}{4n^2}\phi(V_3,1) &\quad k=1
        \\ \frac{\eta}{4n^2}\phi(V_1,1) +\frac{\eta}{4n^2}\phi(V_2,2)+\frac{\eta}{4n^2}\phi(V_3,3)&\quad k=3
        \\  0&\quad k\notin \{1,3\}
    \end{cases}
    .
\end{align*}

\begin{proof}
First, we notice that by \cref{sgd_exp_3}, it holds that $\vecpart{w_3}{2}=c\eta u_2$ for $c\leq \frac{1}{2}$ and $u_2$ holds $u_2\in V_1\cap \bigcap_{i=2}^{n}\overline{V_i}$, and for every $3\leq k\leq T$, $\vecpart{w_t}{k}=0$. Then, by \cref{gradient_uc_sgd}, we have that $\nabla\sgdindgam_1(w_3,V_3)=0.$
Moreover, $\sgdindgam_3$ is a linear function, thus,
    \begin{align*}
    \vecpart{\sgdindgam_3(w_3,V_3)}{k}=\begin{cases}
        \frac{\eta}{n^3}u_1 &\quad k=1
        \\
        0 &\quad k\geq 2
        \end{cases},
\end{align*}
and,
\begin{align*}
   \vecpart{\sgdindgam_3(w_3,V_3)}{k}=\begin{cases}
        \frac{\eta}{4n^2}\phi(V_3,1) &\quad k=1
        \\ 0 &\quad k\neq 1.
    \end{cases}
\end{align*}
For $\sgdindgam_2(w_3,V_3)$,  we first show that the the maximum is attained with $k=2$ and $u=u_2$.
For $k\geq 3$, for every $u\in U$ and $\psi\in \Psi_k$,
\begin{align*}
&\frac{3}{8}\langle u,\vecpart{w_3}{k}\rangle-\frac{1}{2}\langle \alpha(\psi),\vecpart{w_3}{k+1}\rangle +\langle\vecpart{w_3}{ \last,k},\frac{1}{4n}\psi\rangle
-\langle\vecpart{w_3}{ \last,k+1},\frac{1}{4n}\psi\rangle+\langle \vecpart{w_3}{\last,k+1},-\frac{1}{4n^2}\phi(V,k+1)\rangle=0
.
\end{align*}
For $k=1$, for every $u\in U$ and $\psi\in \Psi_1$, by the fact that for every $\psi,V,j$, $\|\psi\|,\|\phi(V,j)\|\leq 1$, 
\begin{align*}
&\frac{3}{8}\langle u,\vecpart{w_3}{k}\rangle-\frac{1}{2}\langle \alpha(\psi),\vecpart{w_3}{k+1}\rangle +\langle\vecpart{w}{ \last,k},\frac{1}{4n}\psi\rangle
-\langle\vecpart{w_3}{ \last,k+1},\frac{1}{4n}\psi\rangle+\langle \vecpart{w_3}{\last,k+1},-\frac{1}{4n^2}\phi(V,k+1)\rangle
\\&=\frac{3}{8}(\frac{2\eta}{n^3}-\frac{3\eta}{8})\langle u_1,u\rangle
    -\frac{\eta}{4}\langle u_2,\alpha(\psi)\rangle + \langle \frac{1}{4n^2}\phi(V_2,1),\psi\rangle - \langle \frac{\eta}{4n^2}\phi(V_1,1) +\frac{\eta}{4n^2}\phi(V_2,2), \psi\rangle \\&+\langle \frac{\eta}{4n^2}\phi(V_1,1) +\frac{\eta}{4n^2}\phi(V_2,2),\frac{1}{4n^2}\phi(V_3,2)\rangle
    \\&\leq 
    \frac{9\eta}{512}
    +\frac{\eta}{32} +\frac{\eta}{4n^2}+\frac{\eta}{2n^2} +\frac{\eta}{8n^4}
    \\& < \frac{29\eta}{256}\tag{$n\geq 4$}
    \\& < \frac{\eta}{8}
.
\end{align*}
For $k=2$, $u\neq u_2$ and every $\psi\in \Psi_2$, , by the fact that for every $\psi,V,j$, $\|\psi\|,\|\phi(V,j)\|\leq 1$, 
\begin{align*}
&\frac{3}{8}\langle u,\vecpart{w_3}{k}\rangle-\frac{1}{2}\langle \alpha(\psi),\vecpart{w_3}{k+1}\rangle +\langle\vecpart{w_3}{ \last,k},\frac{1}{4n}\psi\rangle
-\langle\vecpart{w_3}{ \last,k+1},\frac{1}{4n}\psi\rangle+\langle \vecpart{w_3}{\last,k+1},-\frac{1}{4n^2}\phi(V,k+1)\rangle\\&=\frac{3}{8}\langle u,\vecpart{w_3}{k}\rangle+\langle \frac{1}{4n}\psi,\vecpart{w_3}{\last,k}\rangle
\\&=\frac{3}{8}\langle u,\frac{\eta}{2}u_2\rangle+\langle \frac{1}{4n}\psi,\frac{\eta}{4n^2}\phi(V_1,1) +\frac{\eta}{4n^2}\phi(V_2,2)\rangle
\\&\leq \frac{3\eta}{128}+\frac{\eta}{8n^3}\\&<\frac{\eta}{32}
\tag{$n\geq4$}.\end{align*}
For $k=2$, $u=u_2$ and every $\psi\in \Psi_2$, by the fact that $\|\psi\|,\|\phi(V_1,1)\|\leq 1$,
\begin{align*}
&\frac{3}{8}\langle u,\vecpart{w_3}{k}\rangle-\frac{1}{2}\langle \alpha(\psi),\vecpart{w_3}{k+1}\rangle +\langle\vecpart{w_3}{ \last,k},\frac{1}{4n}\psi\rangle
-\langle\vecpart{w_3}{ \last,k+1},\frac{1}{4n}\psi\rangle+\langle \vecpart{w_3}{\last,k+1},-\frac{1}{4n^2}\phi(V,k+1)\rangle\\&=\frac{3}{8}\langle u,\vecpart{w_3}{k}\rangle+\langle \frac{1}{4n}\psi,\vecpart{w_3}{\last,k}\rangle
\\&=\frac{3}{8}\langle u,\frac{\eta}{2}u_2\rangle+\langle \frac{1}{4n}\psi,\frac{\eta}{4n^2}\phi(V_1,1) +\frac{\eta}{4n^2}\phi(V_2,2)\rangle
\\&\geq \frac{3\eta}{16}-\frac{\eta}{8n^3}\\&>\frac{5\eta}{32}
\tag{$n\geq4$}
\\&>\delta_1.\end{align*}
Second, we show that the maximum among $\psi\in \Psi_2$ is attained uniquely in $\psi^*_2=\frac{1}{n}\phi(V_1,1)+\frac{1}{n}\phi(V_2,2)$. For any $\psi\in \Psi_2$, with $\psi\neq \psi^*_2$, by \cref{diff_set_encode exists_SGD}, for $k=2$, $u=u_2$,
\begin{align*}
&\frac{3}{8}\langle u,\vecpart{w_3}{k}\rangle-\frac{1}{2}\langle \alpha(\psi^*_2),\vecpart{w_3}{k+1}\rangle +\langle\vecpart{w_3}{ \last,k},\frac{1}{4n}\psi^*_2\rangle
-\langle\vecpart{w_3}{ \last,k+1},\frac{1}{4n}\psi^*_2\rangle+\langle \vecpart{w_3}{\last,k+1},-\frac{1}{4n^2}\phi(V,k+1)\rangle\\&=\frac{3}{8}\langle u,\vecpart{w_3}{k}\rangle+\langle \frac{1}{4n}\psi^*_2,\vecpart{w_3}{\last,k}\rangle\\&=\frac{3\eta}{16}+\langle \frac{1}{4n}\psi^*_2,\frac{\eta}{4n^2}\phi(V_1,1)+\frac{\eta}{4n^2}\phi(V_2,2)\rangle
\\&=\frac{3\eta}{16}+\frac{\eta}{16n^2}\langle \psi^*_2,\frac{1}{n}\phi(V_1,1)+\frac{1}{n}\phi(V_2,2)\rangle
\\&\geq \frac{3\eta}{16}+\frac{\eta}{16n^2}\langle \psi,\frac{1}{n}\phi(V_1,1)+\frac{1}{n}\phi(V_2,2)\rangle+\frac{\eta\epsilon}{16n^2}
\\&= \frac{3}{8}\langle u,\vecpart{w_3}{k}\rangle+\langle \frac{1}{4n}\psi,\vecpart{w_3}{\last,k}\rangle +\frac{\eta\epsilon}{16n^2}
\\&=\frac{3}{8}\langle u,\vecpart{w_3}{k}\rangle-\frac{1}{2}\langle \alpha(\psi),\vecpart{w_3}{k+1}\rangle +\langle\vecpart{w_3}{ \last,k},\frac{1}{4n}\psi\rangle
-\langle\vecpart{w_3}{ \last,k+1},\frac{1}{4n}\psi\rangle+\langle \vecpart{w_3}{\last,k+1},-\frac{1}{4n^2}\phi(V,k+1)\rangle+\frac{\eta\epsilon}{16n^2}
\end{align*}
We got that the maximum is uniquely attained at $k=2, u=u_2, \psi=\psi^*_2$.
Now, by \cref{diff_set_encode exists_SGD}, for $j=\argmin_{i}\{i:v_i \in V_1\cap V_2\}$, we get that
\begin{align*}
    \alpha(\psi)=v_{j}\in V_1\cap V_2.
\end{align*}
We notice that $V_1\cap V_2=P_3$ and thus $\alpha(\psi)=J_3$. Then, 
by $\cE'$, $\alpha(\psi)$ also holds $\alpha(\psi)\in S_{3}$.
Combining the above together, we get, for $u_1 \in U$, $u_2\in P_2\cap S_2$ and $u_3=\alpha(\psi_2^*)\in P_3\cap S_3$,
\begin{align*}
    \nabla f(w_3,V_3)=
    \begin{cases}
        - \frac{1}{n^3}u_1&\quad s=1\\
        \frac{3}{8} u_{2} &\quad s=2
        \\
        -\frac{1}{2} u_{3} &\quad s=3
        \\
        0 &\quad 4\leq s\leq n
        \\ -\frac{1}{4n^2}\phi(V_3,1) &\quad s= \last,1
        \\ \frac{1}{4n^2}\phi(V_1,1)+\frac{1}{4n^2}\phi(V_2,2)&\quad s= \last,2
        \\ -\frac{1}{4n^2}\phi(V_1,1) -\frac{1}{4n^2}\phi(V_2,2)-\frac{1}{4n^2}\phi(V_3,3)&\quad s= \last,3
        \\  0&\quad s=\last, k\text{ for } k\geq 3,
    \end{cases}
\end{align*}
and the lemma follows.
\end{proof}
\subsection{Proof of \cref{SGD_lower_bound}}
\begin{proof}[of \cref{SGD_lower_bound}]
    We show that the theorem holds if the event $\cE'$ occurs.
    First, we prove that for every $t$, $\|w_t\|\leq 1$.
    By \cref{diffSGDexpression}, 
    \begin{align*}
        \|w_t\|&\leq \sqrt{\sum_{i=1}^d w_t[i]^2} 
            \\&\leq \sqrt{\sum_{k=1}^n \|\vecpart{w_t}{k}\|^{2}+\sum_{l=1}^n \|\vecpart{w_t}{\last,l}\|^{2}}
             \\&< \sqrt{2\cdot \left(\frac{\eta}{2}\right)^{2}+ (n-2)\left(\frac{\eta}{8}\right)^{2}+2\cdot \left(\frac{\eta}{4n}\right)^2}
             \\&\leq \sqrt{\left(\frac{\eta^2}{2}\right)+ \frac{\eta^2(n-2)}{64}+2\eta^2}
             \\&\leq \sqrt{\frac{1}{64}+\frac{5}{2n}}\tag{$\eta\leq \frac{1}{\sqrt{n}}$}  
             \\&\leq 1 \tag{$n\geq 4$}
    \end{align*}
    Now, denote $\mathbbm{\alpha}_V\in \R^{n-3}$ the vector which its $k$th entry is $\max\left(\frac{\eta}{16}, \max_{u\in V_i} \langle u,(n-k+2)\frac{\eta}{8}u_{k+1}\rangle \right)$. For $\overline{w}_n=w_{n,n}$, and any $2\leq s\leq n-2$,
\begin{align*}
    \vecpart{\overline{w}_{n}}{s}=\frac{\eta}{2n}u_s + (n-s-1)\frac{\eta}{8n}u_s= (n-s+3)\frac{\eta}{8n}u_s.
\end{align*}
Then,
\begin{align*}
&\frac{1}{n}\sum_{i=1}^n\sqrt{\sum_{k=2}^n\max\left(\frac{3\eta}{32}, \max_{u\in V_i} \langle u,\vecpart{\overline{w}_n}{k}\rangle \right)^2}\geq 
    \frac{1}{n}\sum_{i=1}^n\sqrt{\sum_{k=2}^{n-2}\max\left(\frac{3\eta}{32}, \max_{u\in V_i} \langle u,\vecpart{\overline{w}_n}{k}\rangle \right)^2}
    \\&=\frac{1}{n}\sum_{i=1}^n\sqrt{\sum_{k=2}^{n-2}\max\left(\frac{3\eta}{32}, \max_{u\in V_i} \langle u,(n-k+3)\frac{\eta}{8n}u_k\rangle \right)^2}
    \\&=\frac{1}{n}\sum_{i=1}^n\sqrt{\sum_{k=1}^{n-3}\max\left(\frac{3\eta}{32}, \max_{u\in V_i} \langle u,(n-k+2)\frac{\eta}{8n}u_{k+1}\rangle \right)^2}
    \\&=\frac{1}{n}\sum_{i=1}^n\sqrt{\sum_{k=1}^{n-3}\max\left(\frac{3\eta}{32}, \max_{u\in V_i} \langle u,(n-k+2)\frac{\eta}{8n}u_{k+1}\rangle \right)^2}
    \\&=\frac{1}{n}\sum_{i=1}^n\|\alpha_{V_{i}}\|
    \\&\geq \|\frac{1}{n}\sum_{i=1}^n\alpha_{{V_i}}\|
    \\&=\sqrt{\sum_{k=2}^{n-2}\left(\frac{1}{n}\sum_{i=1}^n\max\left(\frac{3\eta}{32}, \max_{u\in V_i} \langle u,(n-k+3)\frac{\eta}{8n}u_{k}\rangle \right)\right)^2}
    \\&=\frac{\eta}{8}\sqrt{\sum_{k=2}^{n-2}\left(\frac{1}{n}\sum_{i=1}^n\max\left(\frac{3}{4}, \max_{u\in V_i} \langle u,\frac{n-k+3}{n}u_{k}\rangle \right)\right)^2}
    \\&=\frac{\eta}{8}\sqrt{\sum_{k=2}^{n-2}\left(\frac{1}{n}\sum_{i=1}^n\max\left(\frac{3}{4}, \max_{u\in V_i} \langle u,\frac{n-k+2}{n}u_{k}\rangle \right)\right)^2}
    \end{align*}
    Now, by the fact that if $\cE'$ holds, by \cref{diffSGDexpression}, for $2\leq k\leq n-2$, $u_k\in P_k=\bigcap_{i=1}^{k-1}V_k$,
    \begin{align*}
    &\frac{1}{n}\sum_{i=1}^n\sqrt{\sum_{k=2}^n\max\left(\frac{3\eta}{32}, \max_{u\in V_i} \langle u,\vecpart{\overline{w}_n}{k}\rangle \right)^2}
    \\&\geq \frac{\eta}{8}\sqrt{\sum_{k=2}^{n-2}\left(\frac{1}{n}\sum_{i=1}^{k-1}\max\left(\frac{3}{4}, \max_{u\in V_i} \langle u,\frac{n-k+2}{n}u_{k}\rangle \right)+\frac{1}{n}\sum_{i=k}^{n}\max\left(\frac{3}{4}, \max_{u\in V_i} \langle u,\frac{n-k+2}{n}u_{k}\rangle \right)\right)^2}
    \\&\geq \frac{\eta}{8}\sqrt{\sum_{k=2}^{n-2}\left(\frac{3(n-k+1)}{4n}+\frac{k-1}{n}\max\left(\frac{3}{4}, \frac{n-k+2}{n}\right)\right)^2}
    \\&\geq \frac{\eta}{8}\sqrt{\sum_{2\leq k\leq \frac{n}{4}-2}\left(\frac{3(n-k+1)}{4n}+\frac{(k-1)(n-k+1)}{n^2}\right)^2+\sum_{\frac{n}{4}-3<k\leq n-2}\left(\frac{3(n-k+1)}{4n}+\frac{3(k-1)}{4n}\right)^2}
     \\&= \frac{\eta}{8}\sqrt{\sum_{2\leq k\leq \frac{n}{4}-2}\left(\frac{(n-k+1)(3n+4(k-1))}{4n^2}\right)^2+\frac{27n}{64}}
     \\&= \frac{\eta}{8}\sqrt{\sum_{1\leq k\leq \frac{n}{4}-3}\left(\frac{(n-k)(3n+4k)}{4n^2}\right)^2+\frac{27n}{64}}
     \\&\geq \frac{\eta}{8}\sqrt{\sum_{1\leq k\leq \frac{n}{4}-3}\left(\frac{3}{4}+\frac{k}{4n}-\frac{k^2}{n^2}\right)^2+\frac{27n}{64}}
    \end{align*}
    Now, the fact that for $\frac{n}{8}\leq k\leq \frac{n}{4}$, $\frac{k}{4n}\leq \frac{k^2}{n^2}$ and for $k\leq \frac{n}{8}$, $\frac{k}{8n}\leq \frac{k^2}{n^2}$,
    \begin{align*}
         &\frac{1}{n}\sum_{i=1}^n\sqrt{\sum_{k=2}^n\max\left(\frac{3\eta}{32}, \max_{u\in V_i} \langle u,\vecpart{\overline{w}_n}{k}\rangle \right)^2}
         \\&\geq \frac{\eta}{8}\sqrt{\sum_{1\leq k\leq \frac{n}{8}}\left(\frac{3}{4}+\frac{k}{8n}\right)^2+\frac{9n}{128}-\frac{27}{16}+\frac{27n}{64}}
    \\&\geq \frac{\eta}{8}\sqrt{\frac{9n}{128}+\frac{3}{64n}\sum_{k=1}^{\lfloor{\frac{n}{8}\rfloor}}k+\frac{9n}{128}-\frac{27}{16}+\frac{27n}{64}}
    \\&\geq \frac{\eta}{8}\sqrt{\frac{1}{2}\left(\frac{n}{8}-1\right)^2-\frac{27}{16}+\frac{36n}{64}}
    \\&\geq \frac{\eta}{8}\sqrt{\frac{n}{512}-\frac{27}{16}+\frac{36n}{64}} \tag{$n\geq 16$}
    \\&\geq \frac{\eta}{8}\sqrt{\frac{577n}{1024}} \tag{$n\geq 2048$}
    \\&\geq\frac{3\eta}{32}\cdot \frac{2001}{2000}
    \end{align*}
        
Now, for $m<n$ and $2\leq k\leq n-2$,
\begin{align*}
    \vecpart{w_{n,m}}{k}&=
    \begin{cases}
        \frac{\eta}{8} u_{k}&\quad k\leq n-m-1\\
        \frac{1}{m}\left(\frac{\eta}{2} u_{k} + (n-k-1)\frac{\eta}{8} u_{s}\right)&\quad k\geq n-m
    \end{cases}
    \\&=
    \begin{cases}
        \frac{\eta}{8} u_{k}&\quad k\leq n-m-1\\
        \frac{\eta(n-k+3)}{8m}u_s&\quad k \geq n-m.
    \end{cases}
\end{align*}
Then, by similar arguments, it holds that,
\begin{align*}
&\frac{1}{n}\sum_{i=1}^n\sqrt{\sum_{k=2}^n\max\left(\frac{3\eta}{32}, \max_{u\in V_i} \langle u,\vecpart{w_{n,m}}{k}\rangle \right)^2}\geq 
    \frac{1}{n}\sum_{i=1}^n\sqrt{\sum_{k=2}^{n-2}\max\left(\frac{3\eta}{32}, \max_{u\in V_i} \langle u,\vecpart{w_{n,m}}{k}\rangle \right)^2}
    \\&\geq \sqrt{\sum_{k=2}^{n-2}\left(\frac{1}{n}\sum_{i=1}^n\max\left(\frac{3\eta}{32}, \max_{u\in V_i} \langle u,\vecpart{w_{n,m}}{k}\rangle \right)\right)^2}
    \\&=\frac{\eta}{8}\sqrt{\sum_{k=2}^{n-m-1}\left(\frac{1}{n}\sum_{i=1}^n\max\left(\frac{3}{4}, \max_{u\in V_i} \langle u,u_{k}\rangle \right)\right)^2+\sum_{k=n-m}^{n-2}\left(\frac{1}{n}\sum_{i=1}^n\max\left(\frac{3}{4}, \max_{u\in V_i} \langle u,\frac{n-k+3}{m}u_{k}\rangle \right)\right)^2}
     \\&\geq\frac{\eta}{8}\sqrt{\sum_{k=2}^{n-m-1}\left(\frac{1}{n}\sum_{i=1}^n\max\left(\frac{3}{4}, \max_{u\in V_i} \langle u,u_{k}\rangle \right)\right)^2+\sum_{k=n-m}^{n-2}\left(\frac{1}{n}\sum_{i=1}^n\max\left(\frac{3}{4}, \max_{u\in V_i} \langle u,\frac{n-k+2}{n}u_{k}\rangle \right)\right)^2}
      \\&\geq\frac{\eta}{8}\sqrt{\sum_{k=2}^{n-2}\left(\frac{1}{n}\sum_{i=1}^n\max\left(\frac{3}{4}, \max_{u\in V_i} \langle u,\frac{n-k+2}{n}u_{k}\rangle \right)\right)^2}\tag{$k\geq 2\implies \frac{n-k+2}{n}\leq 1$}
     \\&\geq \frac{3\eta}{32}\cdot \frac{2001}{2000} \tag{calculation for $w_{n,n}$}
    \end{align*}

As a result, we notice that for every $t$, $\sgdindgam_2(w_t)\geq -\frac{1}{4n^2}-\frac{1}{n^3}$ and $\gdindgam_2(w_t)\geq \delta_1$ thus,  it holds that,
\begin{align*}
\widehat{F}(w_{n,m})&\geq \frac{3\eta\sqrt{n}}{32}\cdot \frac{2001}{2000} -\frac{1}{4n^2}-\frac{1}{n^3}+\delta_1
   \\&\geq  \frac{3\eta\sqrt{n}}{32}\cdot \frac{2001}{2000}-\frac{\eta}{2n^2}
  \\&\geq \frac{3\eta\sqrt{n}}{32}\cdot \frac{2001}{2000}-\frac{\eta\sqrt{n} } {80000}\tag{$n\geq 256$}
  \\&\geq \frac{3\eta\sqrt{n}}{32}\cdot \left(\frac{2001}{2000}-\frac{1}{4000}\right)
   \\&\geq \frac{3\eta\sqrt{n}}{32}\cdot \frac{4001}{4000}
\end{align*}
and
\begin{align*}
\widehat{F}(\erm)&\leq \widehat{F}(0)\leq
\frac{3\eta}{32}\sqrt{n}
\end{align*}
Then, if $\cE'$ holds
\begin{align*}
    \widehat{F}(w_{n,m})-\widehat{F}(\erm)&\geq \frac{3\eta\sqrt{n}}{32}\cdot \frac{2001}{2000}-\frac{3\eta}{32}\sqrt{n} 
    \\&=\frac{\eta\sqrt{n}}{64000}
\end{align*}
\end{proof}
    
\end{lemma}

%% file: appendix_proof_differentiable.tex
\section{Proofs of \cref{sec:const_diff}}
\subsection{Proofs of \cref{sec:gd_diff}}
The proof of \cref{nonspec_lower_bound_GD_diff} appears in \cref{sec:gd_diff}. Here we prove some auxiliary lemmas that are used for the proof.

\begin{proof}[of \cref{convex_lip_diff}]
First, differentiability can be derived immediately from \cref{grads_after_smooth}.
Second, for $5$-Lipschitzness, for every $(V,j)\in Z$, we define $\dgdindf_{V,j}:\R^d\to\R$ as 
$\dgdindf_{V,j}(w)\eqq\dgdindf(w,(V,j))$.
By the $5$-Lipschitzness of $\gdindf$ with respect to its first argument and Jensen Inequality, for every $x,y\in \R^d$, it holds that
\begin{align*}
    |\dgdindf_{V,j}(x)-\dgdindf_{V,j}(y)|&=\left|\E_{v\in\delta B}\left(\gdindf_{V,j}(y+v)\right)-\E_{v\in\delta B}\left(\gdindf_{V,j}(w+v)\right)\right|
    \\&=\left|\E_{v\in\delta B}\left(\gdindf_{V,j}(x+v)-\gdindf_{V,j}(y+v)\right)\right|
    \\&\leq \E_{v\in\delta B}\left|\left(\gdindf_{V,j}(x+v)-\gdindf_{V,j}(y+v)\right)\right|
    \\&\leq 5|x-y|.
\end{align*}
Third, for convexity, by the convexity of $\gdindf$ for every $x,y\in \R^d$ and $\alpha\in[0,1]$,
\begin{align*}
\dgdindf_{V,j}\left(\alpha x + (1-\alpha )y\right)&=
\E_{v\in\delta B}\left(\gdindf_{V,j}(\alpha x + (1-\alpha )y+v)\right)
\\&=
\E_{v\in\delta B}\left(\gdindf_{V,j}(\alpha (x+v) + (1-\alpha )(y+v))\right)
\\&\leq
\E_{v\in\delta B}\left(\alpha \gdindf_{V,j}(x+v) + (1-\alpha )\gdindf_{V,j}(y+v))\right)
\\&=
\alpha \E_{v\in\delta B}\left(\gdindf_{V,j}(x+v)\right) + (1-\alpha )\left(\E_{v\in\delta B}\gdindf_{V,j}(y+v)\right)
\\&=\alpha \dgdindf_{V,j}(x) +(1-\alpha) \dgdindf_{V,j}(y). 
\end{align*}
\end{proof}
\begin{lemma} (Lemma 1 in \cite{withoutgrad})
\label{grads_after_smooth}
    Let $d$ and $\delta>0$, $\mathbb{B}$ be the $d$-dimensional unit ball and $\mathbb{S}$ be the $d$-dimensional unit sphere. Moreover, let $\D_{\mathbb{B}}$ and $\D_{\mathbb{S}}$ be the uniform distributions on $\mathbb{B},\mathbb{S}$ respectively.
    If $\tilde{f}(x)=\E_{v\sim\D_{\mathbb{B}}}\left[f(x+\delta v)\right]$, then,
    \[\nabla \tilde{f}(x)=\frac{d}{\delta}\E_{a\sim\D_{\mathbb{S}}}\left[f(x+\delta a)a\right]\]
\end{lemma}
\begin{lemma} (e.g., \cite{muller1959note})
\label{sphere_unif_dist}
Let $d$. Let $\mathbb{S}$ be the $d$-dimensional unit sphere and $\D_{\mathbb{S}}$ the uniform distributions on $\mathbb{S}$.
Moreover, we define random variables $Y_1,\ldots,Y_d\in \R$,$X_1,\ldots,X_d\in \R$ and $Y\in R^d$ such that $X_i\sim N(0,1)$ (where $N(0,1)$ is the normal univariate distribution with expectation 0 and variance $1$), $Y_i=\frac{x_i}{\sqrt{\sum_{i=1}^d}X_i^2}$ and $Y=(Y_1,\ldots,Y_d)$. Then, $Y\sim \D_{\mathbb{S}}$.
\end{lemma}
\begin{lemma}
\label{grad_constant_threshold}
    Let $d$. Let $\mathbb{B}$ be the $d$-dimensional unit ball and $\D_{\mathbb{B}}$ the uniform distributions on $\mathbb{B}$.
    Let $\zeta_1>\zeta_2>0$, a function $g:\R\to\R$ and $a_{1},...a_{l}\in\mathbb{B}$. Moreover, let $h:\mathbb{B}\to R$, $h(x)=g\left(\max (\zeta_1, \max_{1\leq r\leq l}\langle a_r,x\rangle\right)$ and $x_0\in \mathbb{B}$ such that $\max_{1\leq r\leq l} \langle a_r,x_0\rangle \leq \zeta_2$. We define $\Tilde{h}(x)\eqq \E_{v\sim\D_{\mathbb{B}}}\left[h(x+\delta v)\right]$. 
    Then, for any $0<\delta<\zeta_1-\zeta_2$,
\begin{align*}
&\nabla\Tilde{h}(x_0)=0,
\\&\Tilde{h}(x_0)=g(\zeta_1).
    \end{align*}
\end{lemma}
\begin{proof}
First, for every $r$ and $v\in\mathbb{B}$, by Cauchy-Schwartz Inequality,
    \[\langle a_r,x_0+\delta v\rangle=\langle a_r,x_0\rangle+\langle a_r,\delta v\rangle\leq \zeta_2+\delta<\zeta_1\]
    Then, 
    \[\max (\zeta_1, \max_{1\leq r\leq l}\langle a_r,x_0+\delta v\rangle)=\zeta_1,\]
    and \[h(x_0+\delta v)=g(\max (\zeta_1, \max_{1\leq r\leq l}\langle a_r,x_0+\delta v\rangle)=g(\zeta_1).\]
As a result,
\begin{align*}
\Tilde{h}(x_0)&=
\E_{v\sim\D_{\mathbb{B}}}\left[h(x_0+\delta v)\right]=
g(\zeta_1)
\end{align*}
and by \cref{grads_after_smooth},
\begin{align*}
\nabla \Tilde{h}(x_0)&=
\frac{d}{\delta}\E_{v\sim\D_{\mathbb{S}}}\left[h(x_0+\delta v)v\right]\\&=
\frac{d}{\delta}\E_{v\sim\D_{\mathbb{S}}}\left[g\left(\max (\zeta_1, \max_{1\leq r\leq l}\langle a_r,x_0+\delta v\rangle\right)v\right]
\\&=\frac{d}{\delta}\E_{v\sim\D_{\mathbb{S}}}\left[g(\zeta_1)v\right]
\\&=\frac{d}{\delta}g(\zeta_1)\E_{v\sim\D_{\mathbb{S}}}\left[v\right]
\\&=0
\end{align*}
\end{proof}
\begin{lemma}
\label{grad_constant_threshold_multi}
    Let $d$ and $K$. Let $\mathbb{B}$ be the $dK$-dimensional unit ball and $\D_{\mathbb{B}}$ the uniform distributions on $\mathbb{B}$.
    Let $\zeta_1>\zeta_2>0$ and $a_{1},...a_{l}\in\unitballd$. Moreover, let $g:\unitballd\to R$, $g(x)=\max (\zeta_1, \max_{1\leq r\leq l}\langle a_r,x\rangle)$ and $h:\mathbb{B} \to \R$, $h(x)=\sqrt{\sum_{k=1}^{K} g(\vecpart{x}{k})^2}$. Let
    $x_0\in \mathbb{B}$ such that for every $k$, $\max_{1\leq r\leq l} \langle a_r,\vecpart{x_0}{k}\rangle \leq \zeta_2$. We define $\Tilde{h}(x)\eqq \E_{v\sim\D_{\mathbb{B}}}\left[h(x+\delta v)\right]$. 
    Then, for any $0<\delta<\zeta_1-\zeta_2$,
\begin{align*}
&\nabla\Tilde{h}(x_0)=0,
\\&\Tilde{h}(x_0)=\zeta_1\sqrt{K}.
    \end{align*}
\end{lemma}
\begin{proof}
    First, for every $k,r$ and $u\in\unitballd$, by Cauchy-Schwartz Inequality,
    \[\langle a_r,\vecpart{x_0}{k}+\delta u\rangle=\langle a_r,\vecpart{x_0}{k}\rangle+\langle a_r,\delta u\rangle\leq \zeta_2+\delta<\zeta_1\]
    Then, 
    \[g(\vecpart{x_0}{k}+\delta u)=\max (\zeta_1, \max_{1\leq r\leq l}\langle a_r,\vecpart{x_0}{k}+\delta u\rangle)=\zeta_1,\]
     and for every $v\in \mathbb{B}$,\[h(x_0+\delta v)=\sqrt{\sum_{k=1}^Kg(\max (\zeta_1, \max_{1\leq r\leq l}\langle a_r,\vecpart{x_0}{k}+\delta \vecpart{v}{k}\rangle)}=\zeta_1\sqrt{K}.\]
As a result,
\begin{align*}
\Tilde{h}(x_0)&=
\E_{v\sim\D_{\mathbb{B}}}\left[h(x_0+\delta v)\right]=
\zeta_1\sqrt{K}.
\end{align*}
Now, by \cref{grads_after_smooth},
\begin{align*}
\nabla \Tilde{h}(x_0)&=
\frac{d}{\delta}\E_{v\sim\D_{\mathbb{S}}}\left[h(x_0+\delta v)v\right]\\&=
\frac{d}{\delta}\E_{v\sim\D_{\mathbb{S}}}\left[\left(\sqrt{\sum_{k=1}^{K}\left(\max (\zeta_1, \max_{1\leq r\leq l}\langle a_r,\vecpart{x_0}{k}+\delta \vecpart{v}{k}\rangle\right)^2}\right)\cdot v\right]
\\&=\frac{d}{\delta}\E_{v\sim\D_{\mathbb{S}}}\left[\zeta_1\sqrt{K}v\right]
\\&=\frac{d}{\delta}\zeta_1\sqrt{K}\E_{v\sim\D_{\mathbb{S}}}\left[v\right]
\\&=0.
\end{align*}
\end{proof}
\begin{lemma}
    \label{grad_linear_max}
    Let $d$. Let $\mathbb{B}$ be the $d$-dimensional unit ball and $\D_{\mathbb{B}}$ the uniform distributions on $\mathbb{B}$.
     Let $\zeta_1>\zeta_2,\zeta_3>0$ and vectors $a_{1},...a_{l}\in B_G^d(0)$. Moreover, let $h:\mathbb{B}\to \R$, $h(x)= \max\left(\zeta_3,\max_{1\leq r\leq l}\langle a_r,x\rangle\right)$ and $x_0\in \mathbb{B}, r_{0}\in[l]$ such that $\langle a_{r_{0}},x_0\rangle=\zeta_1$ and $\max_{1\leq r\leq l, r\neq r_{0}}\langle a_{r},x_0\rangle\leq \zeta_2$. We define $\Tilde{h}(x)\eqq \E_{v\sim\D_{\mathbb{B}}}\left[h(x+\delta v)\right]$. Then, for any $0<\delta<\frac{1}{2G}\left(\zeta_1-\max\left(\zeta_2,\zeta_3\right)\right)$,
\begin{align*}
&\Tilde{h}(x_0)=\langle a_{r_{0}},x_0\rangle
\\&\nabla\Tilde{h}(x_0)=a_{r_{0}}
    \end{align*}
\end{lemma}  
\begin{proof}
    First, by Cauchy-Schwartz Inequality,
    \begin{align*}
\max\left(\zeta_3,\max_{r\neq r_0}\langle a_r,x_0+\delta v\rangle\right)&\leq\max\left(\zeta_3,\max_{r\neq r_0}\langle a_r,x_0\rangle+\max_{r\neq r_0}\langle\delta v,a_r\rangle\right)\\&\leq \max\left(\zeta_3,\zeta_2+G\delta\right)
\\&\leq \max\left(\zeta_3,\zeta_2\right)+G\delta\\&<\frac{1}{2}(\zeta_1+\max\left(\zeta_3,\zeta_2)\right),        
    \end{align*}\
    and for $r_0$,
    \[\langle a_{r_{0}},x_0+\delta v\rangle=\langle a_{r_{0}},x_0\rangle+\langle\delta v,a_{r_{0}}\rangle\geq \zeta_1-G\delta>\frac{1}{2}\left(\zeta_1+\max\left(\zeta_2,\zeta_3\right)\right).\]
    We derive that for every $v\in \mathbb{B}$, \[h(x_0+\delta v)=\max\left(\zeta_3,\max_{1\leq r\leq l}\langle a_r,x+\delta v\rangle\right))=\langle a_{r_{0}},x+\delta v\rangle.\] and that the maximum is attained in $r_0$.
    Then, 
    \begin{align*}
        \Tilde{h}(x_0)&=
\E_{v\sim\D_{\mathbb{B}}}\left[h(x_0+\delta v)\right]\\&=
\E_{v\sim\D_{\mathbb{B}}}\left[\langle a_{r_{0}},x_0+\delta v\rangle\right]
\\&=\langle a_{r_{0}},x_0+\delta \E_{v\sim\D_{\mathbb{B}}}v\rangle
\\&=\langle a_{r_{0}},x_0\rangle
    \end{align*}
    and by \cref{grads_after_smooth},
    \begin{align*}
        \nabla\Tilde{h}(x_0)&= \frac{d}{\delta}\E_{v\sim\D_{\mathbb{S}}}\left[\max\left(\zeta_3,\left(\max_{1\leq r\leq l}\langle a_r,x_0+\delta v\rangle\right)\right)v\right]
        \\&=\frac{d}{\delta}\E_{v\sim\D_{\mathbb{S}}}\left[\left(\langle a_{r_{0}},x_0+\delta v\rangle\right)v\right]
        \\&=\langle a_{r_{0}},x_0\rangle\frac{d}{\delta}\E_{v\sim\D_{\mathbb{S}}}\left[v\right]+\frac{d}{\delta}\E_{v\sim\D_{\mathbb{S}}}\left[\langle a_{r_{0}},\delta v\rangle v\right]
       \\&=0+d a_{r_{0}}^T\E_{v\sim\D_{\mathbb{S}}}\left[ v v^T\right]
       \\&=d a_{r_{0}}^T\E_{v\sim\D_{\mathbb{S}}}\left[ v v^T\right].
        \end{align*}
   Now, we define random variables $Y_1,\ldots,Y_d\in \R$,$X_1,\ldots,X_d\in \R$ and $Y\in R^d$ such that $X_i\sim N(0,1)$ (where $N(0,1)$ is the normal univariate distribution with expectation 0 and variance $1$), $Y_i=\frac{x_i}{\sqrt{\sum_{i=1}^d}X_i^2}$.  By \cref{sphere_unif_dist}, we get, for the standard basis vectors $e_1\ldots e_d$,
    \begin{align*}
        \nabla\Tilde{h}(x)&=d a_{r_{0}}^T\E_{Y_1,\ldots,Y_d}\left[\sum_{i=1}^d Y_i^2e_ie_i^T\right] 
       \\&=d a_{r_{0}}^T\sum_{i=1}^d \E_{Y_i}\left[Y_i^2\right] e_ie_i^T
       \\&=d a_{r_{0}}^T\sum_{i=1}^d \E_{X_i}\left[\frac{X_i^2}{\sum_{l=1}^d X_l^2}\right] e_ie_i^T
       \\&=d a_{r_{0}}^T\sum_{i=1}^d \frac{1}{d}e_ie_i^T
       \\&=a_{r_{0}}
    \end{align*}
\end{proof}
\begin{proof}[of \cref{iden_grad_iter}]

We assume that $\cE$ (\cref{good_event_gd}) holds and show \cref{iden_grad_iter} under this event. 
We prove the claim by induction on $t$.
For $t=1$, it is trivial. Now, we assume that $w_t=\Tilde{w}_t$.

For $\gdindgam_1$, in every $t$, by the proofs of \cref{diff_GD_w_2,diff_GD_w_3,diff_GD_w_4,diff_GD_w_5} and \cref{diff_GD_expression}, 
it can be observed that for every $i\in[n]$, $k\geq 2$, $\max_{t}\max_{u
\in V_i}\langle u,\vecpart{w_t}{k}\rangle\leq \frac{\eta}{16}$, thus, in every iteration the term that gets the maximal value is $\frac{3\eta}{32}$.
Then, by \cref{grad_constant_threshold_multi} and the hypothesis of the induction, for every $i$,
\begin{align*}
\nabla \dgdindgam_1(\Tilde{w}_t,V_i)&=\nabla\dgdindgam_1(w_{t},V_i)
=0
=\nabla \gdindgam_1(w_{t},V_i).
\end{align*}

For $\dgdindgam_2$ and every $w\in \R^d$, $V\subseteq U$ and $j\in[n^2]$, by linearity of expectation,
\begin{align*}
    \dgdindgam_2(w,(V,j))&=\E_{v\in\delta B}\left(\langle \wenc+\enc{v},-\phi(V,j)\rangle\right)
    \\&=\gdindgam_2(w,(V,j))+\left(\langle \E_{v\in\delta B}\enc{v},-\phi(V,j)\rangle\right)
     \\&=\gdindgam_2(w,(V,j))
\end{align*}
Then, we derive that for every $w$ and $i$, $\nabla\dgdindgam_2(w,(V_i,j_i))=\nabla\gdindgam_2(w,(V_i,j_i))$.

For $\dgdindgam_3$, which is a $2$-Lipschitz linear function, for $t=1$, by the proof of \cref{diff_GD_w_2},  the term that gets the maximal value is $\delta_1$.  Moreover, it can be observed that for such $t$, and every $\psi\in\Psi$, \[\max_{\psi\in \Psi} \left(\langle\wenc_1, \psi\rangle - \frac{1}{4}\frac{\epsilon}{T^2}\langle \alpha(\psi), \vecpart{w_1}{1}\rangle \right)=0.\]
Then, we can apply \cref{grad_constant_threshold}, and get by the hypothesis of the induction,
\begin{align*}
\nabla \dgdindgam_3(\Tilde{w}_1)&=\nabla\dgdindgam_3(w_{1})
=0
=\nabla \gdindgam_3(w_{1}).
\end{align*}
If $t\geq 2$, it can be observed that
\begin{align*}
\gdindgam_3(w_{t})&=\max\left(\delta_2,\max_{\psi\in \Psi} \left(\langle\wenc_t, \psi\rangle - \frac{1}{4}\frac{\epsilon}{T^2}\langle \alpha(\psi), \vecpart{w_t}{1}\rangle \right)\right)\\&=\max_{\psi\in \Psi} \left(\langle\wenc_t, \psi\rangle - \frac{1}{4}\frac{\epsilon}{T^2}\langle \alpha(\psi), \vecpart{w_t}{1}\rangle \right)\end{align*}
Then, by the proofs of \cref{diff_GD_w_3,diff_GD_w_4,diff_GD_w_5} and \cref{diff_GD_expression}, the maximal value of $\langle\wenc, \psi\rangle - \frac{1}{4}\frac{\epsilon}{T^2}\langle \alpha(\psi), \vecpart{w}{1}\rangle$ is attained in $\psi=\psi^*$ and the difference from the second maximal possible value of this term is at most $\frac{\epsilon}{2T^2}$. As a result, using the fact that this maximum is also larger than $\delta_2$ by at least $\frac{\eta}{8n}$ (which is also larger than $\delta$),
we can apply \cref{grad_linear_max} and get by the hypothesis of the induction that, \begin{align*}
\vecpart{\nabla \dgdindgam_3(\Tilde{w}_t)}{\indsec}&=\vecpart{\nabla\dgdindgam_3(w_{t})}{\indsec}
\\&=\begin{cases}
         \frac{1}{n}\sum_{i=1}^n\phi(V_i,j_i) &\quad \indsec=\last\\
          -\frac{\epsilon}{4T^2}\alpha(\frac{1}{n}\sum_{i=1}^n\phi(V_i,j_i) )&\quad \indsec=1\\
   0 &\quad \text{otherwise} 
\end{cases}
\\&=\vecpart{\nabla \gdindgam_3(w_{t})}{\indsec}.
\end{align*}

For $\dgdindgam_4$, for $t\in
\{1,2\}$, by the proofs of \cref{diff_GD_w_2,diff_GD_w_3}, the term that gets the maximal value is $\delta_2$.  Moreover, it can be observed that for every such $t$, and every $k\in[T-1$] and $u\in U$, \[\frac{3}{8}\langle u,\vecpart{w_t}{k}\rangle-\frac{1}{2}\langle u,\vecpart{w_t}{k+1}\rangle=0.\]
Then, we can apply \cref{grad_constant_threshold}, and get by the hypothesis of the induction,
\begin{align*}
\nabla \dgdindgam_4(\Tilde{w}_t)&=\nabla\dgdindgam_4(w_{t})
=0
=\nabla \gdindgam_4(w_{t}).
\end{align*}
For $t=3$, it can be observed by the proof of \cref{diff_GD_w_4} that,
\begin{align*}
\gdindgam_4(w_{t})&=\max\left(\delta_2,\max_{k\in [T-1], u \in U}\left(\frac{3}{8}\langle u,\vecpart{w_t}{k}\rangle-\frac{1}{2}\langle u,\vecpart{w_t}{k+1}\rangle\right)\right)\\&=\max_{k\in [T-1], u \in U}\left(\frac{3}{8}\langle u,\vecpart{w_t}{k}\rangle-\frac{1}{2}\langle u,\vecpart{w_t}{k+1}\rangle\right)\end{align*}
Moreover, the maximal value is $\frac{3\eta\epsilon}{32T^2}$ and is attained in $k_0=1,u=u_0=\alpha(\psi^*)$. The second maximal possible value of this term is $\delta_2=\frac{3\eta\epsilon}{64T^2}$, then, by the fact that $\delta<\frac{3\eta\epsilon}{32T^2}-\frac{3\eta\epsilon}{64T^2}=\frac{3\eta\epsilon}{64T^2}$, we can apply \cref{grad_linear_max} and get by the hypothesis of the induction that \begin{align*}
\vecpart{\nabla \dgdindgam_4(\Tilde{w}_t)}{\indsec}&=\vecpart{\nabla\dgdindgam_4(w_{t})}{\indsec}
\\&=\begin{cases}
         \frac{3}{8}u_0 &\quad \indsec=1\\
          -\frac{1}{2}u_0 &\quad \indsec=2\\
   0 &\quad \text{otherwise} 
\end{cases}
\\&=\vecpart{\nabla \gdindgam_4(w_{t})}{\indsec}.
\end{align*}
For $t\geq 4$, it can be observed by the proofs of \cref{diff_GD_expression,diff_GD_w_5} that,
\begin{align*}
\gdindgam_4(w_{t})&=\max\left(\delta_2,\max_{k\in [T-1], u \in U}\left(\frac{3}{8}\langle u,\vecpart{w_t}{k}\rangle-\frac{1}{2}\langle u,\vecpart{w_t}{k+1}\rangle\right)\right)\\&=\max_{k\in [T-1], u \in U}\left(\frac{3}{8}\langle u,\vecpart{w_t}{k}\rangle-\frac{1}{2}\langle u,\vecpart{w_t}{k+1}\rangle\right)\end{align*}
Moreover, the maximal value is $\frac{3\eta}{16}$ and is attained in $k_0=t-2,u=u_0=\alpha(\psi^*)$. The second maximal possible value of this term is smaller then $\frac{5\eta}{64}$, then we can apply again \cref{grad_linear_max} and get by the hypothesis of the induction that \begin{align*}
\vecpart{\nabla \dgdindgam_4(\Tilde{w}_t)}{\indsec}&=\vecpart{\nabla\dgdindgam_4(w_{t})}{\indsec}
\\&=\begin{cases}
         \frac{3}{8}u_0 &\quad \indsec=t-2\\
          -\frac{1}{2}u_0 &\quad \indsec=t-1\\
   0 &\quad \text{otherwise} 
\end{cases}
\\&=\vecpart{\nabla \gdindgam_4(w_{t})}{\indsec}.
\end{align*}

In conclusion, we proved that $\nabla \gdindempf (w_t)=\nabla \dgdindempf (\Tilde{w}_t)$, thus, by the hypothesis of the induction,
\[w_{t+1} = w_t - \nabla \gdindempf (w_t)=\Tilde{w}_t - \nabla \dgdindempf (\Tilde{w}_t)=\Tilde{w}_{t+1}\]
\end{proof}

\begin{lemma}
\label{expec_change_bound}
      Let $d$ and $\delta>0$. Let $f:\R^d\to \R$ be a $G$-Lipschitz function. Let $\mathbb{B}$ be the $d$-dimensional unit ball. Moreover, let $\D_{\mathbb{B}}$ be the uniform distributions on $\mathbb{B}$.
    If $\tilde{f}(x)=\E_{v\sim\D_{\mathbb{B}}}\left[f(x+\delta v)\right]$,
    then for every $x$,
    \[|\tilde{f}(x)-f(x)|\leq G\delta\]
\end{lemma}
\begin{proof}
    By the fact that $f$ is $G$-Lipschitz,
    \begin{align*}
    |\tilde{f}(x)-f(x)|&= |\E_{v\sim\D_{\mathbb{B}}}\left[f(x+\delta v)\right]-f(x)|\\&\leq
|\E_{v\sim\D_{\mathbb{B}}}\left[f(x)\right]+G\delta\E_{v\sim\D_{\mathbb{B}}}+\left[\|v\|\right]-f(x)|
\\&=G\delta\E_{v\sim\D_{\mathbb{B}}}\left[\|v\|\right]
\\&\leq G\delta
    \end{align*}

\end{proof}

\label{sec:proof_diff_GD}

\subsection{Proofs of \cref{sec:sgddiff}}

\begin{proof}[of \cref{convex_lip_diff_sgd}]
First, differentiability can be derived immediately from \cref{grads_after_smooth}.
Second, for $4$-Lipschitzness, for every $V\in Z$, we define $\dsgdindf_{V}:\R^d\to\R$ as 
$\dsgdindf_{V}(w)\eqq\dsgdindf(w,V)$.
By the $5$-Lipschitzness of $\sgdindf$ with respect to its first argument and Jensen Inequality, for every $x,y\in \R^d$, it holds that
\begin{align*}
    |\dsgdindf_{V}(x)-\dsgdindf_{V}(y)|&=\left|\E_{v\in\delta B}\left(\sgdindf_{V}(y+v)\right)-\E_{v\in\delta B}\left(\sgdindf_{V}(w+v)\right)\right|
    \\&=\left|\E_{v\in\delta B}\left(\sgdindf_{V}(x+v)-\sgdindf_{V}(y+v)\right)\right|
    \\&\leq \E_{v\in\delta B}\left|\left(\sgdindf_{V}(x+v)-\sgdindf_{V}(y+v)\right)\right|
    \\&\leq 4|x-y|.
\end{align*}
Third, for convexity, by the convexity of $\sgdindf$ for every $x,y\in \R^d$ and $\alpha\in[0,1]$,
\begin{align*}
\dsgdindf_{V}\left(\alpha x + (1-\alpha )y\right)&=
\E_{v\in\delta B}\left(\sgdindf_{V}(\alpha x + (1-\alpha )y+v)\right)
\\&=
\E_{v\in\delta B}\left(\sgdindf_{V}(\alpha (x+v) + (1-\alpha )(y+v))\right)
\\&\leq
\E_{v\in\delta B}\left(\alpha \sgdindf_{V}(x+v) + (1-\alpha )\gdindf_{V}(y+v))\right)
\\&=
\alpha \E_{v\in\delta B}\left(\sgdindf_{V}(x+v)\right) + (1-\alpha )\left(\E_{v\in\delta B}\sgdindf_{V}(y+v)\right)
\\&=\alpha \dsgdindf_{V}(x) +(1-\alpha) \dsgdindf_{V,j}(y). 
\end{align*}
\end{proof}
\begin{proof}[of \cref{iden_grad_iter_sgd}]
We assume that $\cE'$ (\cref{good_event_sgd}) holds and prove \cref{iden_grad_iter_sgd} under this event. 
We prove the claim by induction on $t$.
For $t=1$, it is trivial. Now, we assume that $w_t=\Tilde{w_t}$.
First, for $\dsgdindgam_3$ and every $w$ and $V$, by linearity of expectation,
\begin{align*}
    \dsgdindgam_3(w,V)&=\E_{v\in\delta B}\left(\langle \vecpart{w}{\last,1}+\vecpart{v}{\last,1},-\frac{1}{4n^2}\phi(V,1)\rangle -\langle \frac{1}{n^3} u_1, \vecpart{w}{1}+\vecpart{v}{1}\rangle\right)\\&=\sgdindgam_3(w,V)+\left(\langle \E_{v\in\delta B}\vecpart{v}{\last,1},-\frac{1}{4n^2}\phi(V,1)\rangle -\langle \frac{1}{n^3} u_1, \E_{v\in\delta B}\vecpart{v}{1}\rangle\right)
    \\&=\sgdindgam_3(w,V)
\end{align*}
Then, we derive that for every $w$, $\nabla\dsgdindgam_3(w,V)=\nabla\sgdindgam_3(w,V)$.
Now, for $r\in\{1,2\}$ we show that in each term $\dsgdindgam_r(w_t,V_t)$, the argument that gives the maximum value is the same as $\sgdindgam_r(w_t,V_t)$.

For $\dsgdindgam_1(w_t,V_t)$, in every $t$, by the proofs of \cref{sgd_exp_2}, \cref{sgd_exp_3} and \cref{diffSGDexpression}, the maximal value is $\frac{3\eta}{32}$. Moreover, it can be observed that for every $k\geq 2$, $\max_{t}\max_{u
\in V_t}\langle u,\vecpart{w_t^\sgdind}{k}\rangle\leq \frac{\eta}{16}$.
Then, by \cref{grad_constant_threshold_multi}, and the hypothesis of the induction,
\begin{align*}
\nabla \dsgdindgam_1(\Tilde{w_{t}},V_{t})&=\nabla\dsgdindgam_1(w_{t},V_{t})
=0
=\nabla \sgdindgam_1(w_{t},V_t).
\end{align*}
Now, for $\dsgdindgam_2$, for $t=1$, $\nabla \sgdindgam_2(w_1,V_1)= 0$ and the maximum is attained uniquely in $\delta_1=\frac{\eta}{8n^3}$ (the second maximal value is zero). Then, we can apply \cref{grad_constant_threshold} and by the hypothesis of the induction, it follows that,
\begin{align*}
\nabla \dsgdindgam_2(\Tilde{w_{1}},V_1)&=\nabla\dsgdindgam_2(w_{1},V_1)
=0
=\nabla \sgdindgam_2(w_{1},V_1).
\end{align*}
For every $t\geq 2$, the maximum is attained uniquely in the linear term of $k=t-1$,$u=u_{t-1}$ and $\psi=\psi^*_{t-1}$ such that the difference between the maximal value the second largest value is larger than $\frac{\eta\epsilon}{16n^2}$. Then, we can apply \cref{grad_linear_max}  and by the hypothesis of the induction, it follows that,
\begin{align*}
\nabla \dsgdindgam_2(\Tilde{w_{t}},V_t)&=\nabla\dsgdindgam_2(w_{t},V_t)
=\nabla \sgdindgam_t(w_{t},V_t).
\end{align*}
    In conclusion, we proved that $\nabla \dsgdindf(\Tilde{w_{t}},V_t)=\nabla \sgdindf(w_{t},V_t)$, thus, by the hypothesis of the induction,
\[w_{t+1}= w_t- \nabla \sgdindf(w_{t},V_t)=\Tilde{w}_t - \nabla \dsgdindf(\Tilde{w}_{t},V_t)=\Tilde{w}_{t+1}.\]
\end{proof}

%% file: appendix_optimization_bound.tex
\section{Lower bound of \(\Omega\left(\min\left(1,\frac{1}{\eta T}\right)\right)\)}
\label{sec_small_eta}
In this section, we prove the $\Omega\left(\min\left(1,\frac{1}{\eta T}\right)\right)$ lower bound. Since our hard construction for getting this bound involves a deterministic loss function, GD is equivalent to SGD. For clarity, we refer in our proof to the performance of GD, however, the same result is applicable also for SGD with $T=n$ iterations. 
\subsection{Construction of a non-differentiable loss function.}

For $d=\max(25\eta^2T^2,1)$, we define the hard loss function $\optindf:\R^d\to\R$, as follows,
\begin{align}
\label{loss_function_small_eta}
\optindf(w)=\max \left(0, \max_{i\in[d]} \{\frac{1}{\sqrt{d}}-\vecentry{w}{i}-\frac{\eta i}{4d}\}\right).
\end{align}
For this loss function, we prove the following lemma,
\begin{lemma}
\label{nonspec_lower_bound_GD_small_eta}
Assume $n,T>0,\eta\leq \frac{1}{5\sqrt{T}}$. Consider the loss function $\optindf$ that is defined in \cref{loss_function_small_eta} for $d=\max(25\eta^2T^2,1)$.
Then, %
for Unprojected GD (cf. \cref{gd_update_rule} with $W=\R^d$) on $\optindf$, initialized at $w_1=0$ with step size $\eta$, we have, %
\begin{enumerate}[label=(\roman*)]
    \item 
    The iterates of GD remain within the unit ball, namely $w_t \in \unitballd$ for all $t=1,\ldots,T$;
    \item
    For all $m=1,\ldots,T$, the $m$-suffix averaged iterate has:
    \[ 
        \optindf(\suff)
        - \optindf(\opt) 
        = 
        \Omega\br[3]{\min\left(1,\frac{1}{\eta T}\right)}
        .
    \]
\end{enumerate}
\end{lemma}
\paragraph{Algorithm's dynamics}
We start by proving a lemma that characterizes the dynamics of the algorithm.
\begin{lemma}
\label{dynamics_small_eta}
Assume the conditions of \cref{nonspec_lower_bound_GD_small_eta}, and consider the iterate of Unprojected GD on $\optindf$, initialized at $w_1=0$ with step size $\eta\leq \frac{1}{5{\sqrt{T}}}$
Let $w_t$ be the iterate of %
Then, it holds that,
\begin{enumerate}[label=(\roman*)]
    \item For every $i\in[d]$ and for every $t\in[T]$, \[\vecentry{w_t}{i}\leq \frac{1}{2\sqrt{d}}\]
     \item For every $t\in [T]$, there exists an index $j_t\in[d]$ such that 
     $k\neq j_t$,
        \[\frac{1}{\sqrt{d}}-\vecentry{w_t}{j_t}-\frac{\eta j}{4d}>\frac{1}{\sqrt{d}}-\vecentry{w_t}{k}-\frac{\eta k}{4d}+\frac{\eta}{8d}.\]
    \item For every $t\in [T]$, $j_t$ also holds
    \[\frac{1}{\sqrt{d}}-\vecentry{w_t}{j_t}-\frac{\eta j_t}{4d}>\frac{\eta}{8d}.\]
   
\end{enumerate}
\end{lemma}
\begin{proof}
We prove by induction on $t$. For $t=1$, $w_t=0$, thus,
\[\vecentry{w_1}{i}=0\leq \frac{1}{2\sqrt{d}}.\]
Moreover, the maximizer is $j_1=1$. Then, we notice that for both $d=1$ and $d=25\eta^2T^2$, $\eta \leq\frac{1}{5\sqrt{T}}\implies\eta\leq \frac{1}{5\sqrt{d}}$. Then, it holds that, 
\begin{align*}
    \frac{1}{\sqrt{d}}-\vecentry{w_1}{j_1}-\frac{\eta j_1}{4d}&\geq 
    \frac{1}{\sqrt{d}}-\vecentry{w_1}{j_1}-\frac{\eta}{4}\\&\geq
     \frac{19}{20\sqrt{d}}
     \\&>\frac{\eta}{8d},
\end{align*}
and, for every $k\neq j_1$, 
\begin{align*}
    \frac{1}{\sqrt{d}}-\vecentry{w_1}{j_1}-\frac{\eta j_1}{4d}&=
     \frac{1}{\sqrt{d}}-\vecentry{w_1}{k}-\frac{\eta k}{4d}+\frac{\eta (k-j_1)}{4d}
     \\&\geq  \frac{1}{\sqrt{d}}-\vecentry{w_1}{k}-\frac{\eta k}{4d}+\frac{\eta}{4d}
     \\&>\frac{1}{\sqrt{d}}-\vecentry{w_1}{k}-\frac{\eta k}{4d}+\frac{\eta}{8d}
    .
\end{align*}

In the step of the induction we assume that the lemma holds for every $s\leq t$ and prove it for $s=t+1$.
By the hypothesis of the induction, we know that for every iteration $s\leq t$, $\|w_t\|_2\leq \frac{1}{2}$, as a result, we know that the projections does not affect the dynamics of the algorithm until the iteration $t$. Moreover, we know that for every iteration $s\leq t$
there exists an index $j_s\in[d]$ such that the term that achieve the maximum value in $w_s$ is $\frac{1}{\sqrt{d}}-\vecentry{w_s}{j_s}-\frac{\eta j}{4d}$. This maximum is attained uniquely in $j_s$ by margin that is strictly larger than $\frac{\eta}{8d}$.
As a result, we derive that, for every $s\leq t$, $\nabla f (w_s)=-e_{j_{s}}$. 
Now, for every index $m\in[d]$, we define,
\[n_t^m=|\{s\leq t : m=\arg\max_{i\in[d]} \{\frac{1}{\sqrt{d}}-\vecentry{w_s}{i}-\frac{\eta i}{4d}\}\}|.\]
We get that, for every $i$ it holds that, 
\[w_{t+1}[i]=\eta n_t^i.\]
Then,
\begin{align*}
    \|w_{t+1}\|_1=\sum_{i} \eta n_t^i\leq \eta t,
\end{align*}
and ,thus, there exists a entry $k\in [d]$ with $w_{t+1}[k]\leq\frac{\eta t}{d}$.
Now, we prove the first part of the lemma using this observation and the step of the induction.
For every $i\neq j_t$,
\[w_{t+1}[i]=w_{t}[i]\leq \frac{1}{2\sqrt{d}}.\]
Otherwise, we know that, by the definition of $j_t$
\begin{align*}
    \frac{1}{\sqrt{d}}-\vecentry{w_{t}}{i}-\frac{\eta i}{4d}>\frac{1}{\sqrt{d}}-\vecentry{w_t}{k}-\frac{\eta k}{4d}+\frac{\eta}{8d},
\end{align*}
\begin{align*}
    \vecentry{w_{t}}{i}&<\vecentry{w_t}{k}+\frac{\eta (k-i)}{4d}-\frac{\eta}{8d}
    \\&\leq\frac{\eta t}{d}+\frac{\eta}{4}
    \\&\leq\frac{1}{25\sqrt{d}}+\frac{1}{20\sqrt{d}}
\end{align*}
and,
\begin{align*}
\vecentry{w_{t+1}}{i}
    &\leq \vecentry{w_{t}}{i}+\eta
    \\&\leq \frac{1}{25\sqrt{d}}+\frac{1}{20\sqrt{d}} +\frac{1}{5\sqrt{d}}
    \\&\leq\frac{1}{2\sqrt{d}},
\end{align*}
where we again used the fact that $\eta\leq \frac{1}{5\sqrt{T}}$ implies $\eta\leq\frac{1}{5\sqrt{d}}$ for both $d=1$ and $d=25\eta^2T^2$.

For the second part of the lemma, we define $J_t\subseteq [d]$, $J_t=\argmin_{j}\{n_t^j\}$ and $j_{t+1}=\min \{j\in J_t\}$ and show that $j_{t+1}$ holds the required.
We know, for every $j\neq i\in[d]$,
\begin{align*}
   \vecentry{w_{t+1}}{i}-\vecentry{w_{t+1}}{j}=\eta(n_t^i-n_t^j).
\end{align*}
For $k\neq j_{t+1}$ with $n_t^k > n_t^{j_{t+1}}$, %
\begin{align*}
    \frac{1}{\sqrt{d}}-\vecentry{w_{t+1}}{j_{t+1}}-\frac{\eta j_{t+1}}{4d}&\leq 
    \frac{1}{\sqrt{d}}-\vecentry{w_{t+1}}{k}-\eta-\frac{\eta j_{t+1}}{4d}
     \\& =
     \frac{1}{\sqrt{d}}-\vecentry{w_{t+1}}{k}-\eta- \frac{\eta k}{4d}+\frac{\eta(k-j_{t+1})}{4d}
     \\& \leq 
     \frac{1}{\sqrt{d}}-\vecentry{w_{t+1}}{k}-\eta- \frac{\eta k}{4d}+\frac{\eta}{4}
     \\& <
     \frac{1}{\sqrt{d}}-\vecentry{w_{t+1}}{k}- \frac{\eta k}{4d}-\frac{\eta}{2}.
\end{align*}
in contradiction to the fact that $j_{t+1}$ gets the maximal value.
For $k\neq j_{t+1}$ with $n_t^k > n_t^{j_{t+1}}$, it holds that $\vecentry{w_{t+1}}{j_{t+1}}\leq \vecentry{w_{t+1}}{k}-\eta$, and,
\begin{align*}
    \frac{1}{\sqrt{d}}-\vecentry{w_{t+1}}{j_{t+1}}-\frac{\eta j_{t+1}}{4d}&\geq 
    \frac{1}{\sqrt{d}}-\vecentry{w_{t+1}}{k}+\eta-\frac{\eta j_{t+1}}{4d}
     \\& =
     \frac{1}{\sqrt{d}}-\vecentry{w_{t+1}}{k}+\eta- \frac{\eta k}{4d}+\frac{\eta(k-j_{t+1})}{4d}
     \\& \geq 
     \frac{1}{\sqrt{d}}-\vecentry{w_{t+1}}{k}+\eta- \frac{\eta k}{4d}-\frac{\eta}{4}
     \\& >
     \frac{1}{\sqrt{d}}-\vecentry{w_{t+1}}{k}- \frac{\eta k}{4d}+\frac{\eta}{8d},
\end{align*}
as required.
For the third part of the lemma, we know that for every $i\in [d]$, 
\begin{align*}
    \frac{1}{\sqrt{d}}-\vecentry{w_{t+1}}{i}-\frac{\eta i}{4d}
    &\geq 
    \frac{1}{2\sqrt{d}}-\frac{\eta }{4}
    \\&= 
    \frac{9}{20\sqrt{d}}
    \\&>\frac{\eta}{8d}.
\end{align*}
\end{proof}
\paragraph{Proof of lower bound.}
Now we can prove \cref{nonspec_lower_bound_GD_small_eta}.
\begin{proof}[of \cref{nonspec_lower_bound_GD_small_eta}]
The first part of the theorem is an immediate corollary from \cref{dynamics_small_eta}.
Moreover, by applying this lemma again, we know that, for every $i\in[d]$,
\begin{align*}
    \vecentry{\suff}{i}\leq \frac{1}{2\sqrt{d}},
\end{align*}
thus,
\begin{align*}
    \optindf(\suff)-\optindf(\opt)&\geq \frac{1}{2\sqrt{d}}-\frac{\eta}{4}-0
    \\&\geq \frac{1}{2\sqrt{d}}-\frac{\eta}{20\sqrt{d}}
    \\&> \frac{1}{4\sqrt{d}}
    \\&=\min\left(\frac{1}{4}, \frac{1}{20\eta T}\right).
\end{align*}
\end{proof}
\subsection{Construction of a differentiable loss function.}
In this section, we prove the lower bound for a smoothing of $\optindf$, defined as \begin{align}
\label{loss_function_small_eta_diff}
\doptindf(w)=\E_{v\in\unitballd} \max \left(0, \max_{i\in[d]} \{\frac{1}{\sqrt{d}}-\vecentry{w}{i}-\delta\vecentry{v}{i}-\frac{\eta i}{4d}\}\right),
\end{align}
namely, we prove the following lemma,
\begin{lemma}
\label{nonspec_lower_bound_GD_diff_small_eta}
Assume $n,T>0,\eta\leq \frac{1}{5\sqrt{T}}$. Consider the loss function $\doptindf$ that is defined in \cref{loss_function_small_eta_diff} for $d=\max(25\eta^2T^2,1)$ and $\delta=\frac{\eta}{16d}$.
Then, %
for Unprojected GD (cf. \cref{gd_update_rule} with $W=\R^d$) on $\optindf$, initialized at $w_1=0$ with step size $\eta$, we have, 
\begin{enumerate}[label=(\roman*)]
    \item 
    The iterates of GD remain within the unit ball, namely $w_t \in \unitballd$ for all $t=1,\ldots,T$;
    \item
    For all $m=1,\ldots,T$, the $m$-suffix averaged iterate has:
    \[ 
        \doptindf(\suff)
        - \doptindf(\opt) 
        = 
        \Omega\br[3]{\min\left(1,\frac{1}{\eta T}\right)}
        .
    \]
\end{enumerate}
\end{lemma}
First, we prove that the smoothing of the loss function does not affect the dynamics of the algorithm, as stated in the following lemma,
\begin{lemma}
    \label{iden_grad_iter_small_eta}
Under the conditions of \cref{nonspec_lower_bound_GD_diff_small_eta,nonspec_lower_bound_GD_small_eta}, let $w_t,\Tilde{w}_t$ be the iterates of Unprojected GD with step size $\eta\leq\frac{1}{5\sqrt{T}}$ and $w_1=0$, on $\optindf$ and $\doptindf$ respectively. 
    Then, for every $t\in[T]$, it holds that 
    $w_t=\Tilde{w}_t$. 
\end{lemma}
\begin{proof}
    We proof the lemma by induction on $t$. For $t=1$, we know that $w_1=\Tilde{w}_1=0$.
    Now, we assume that $w_t=\Tilde{w}_t$.
    By \cref{dynamics_small_eta}, we know that the maximum of the loss function is attained uniquely with the property that the difference between the maximal value and the second maximal value is larger then $\frac{\eta}{8d}$. As a result, by the facts that $\gdindf$ is $1$-Lipschitz and $\delta=\frac{\eta}{16d}$, we can use \cref{grad_linear_max} for $\gdindempf(w_t)$ and get that, \[\nabla \gdindempf(w_t)=\nabla \dgdindempf(w_t)=\nabla \dgdindempf(\Tilde{w}_t).\]
    It follows
    by the hypothesis of the induction that,
\[w_{t+1} = w_t - \nabla \gdindempf (w_t)=\Tilde{w}_t - \nabla \dgdindempf (\Tilde{w}_t)=\Tilde{w}_{t+1}.\]
\end{proof}
Now we can prove \cref{nonspec_lower_bound_GD_diff_small_eta}.
\begin{proof}[of \cref{nonspec_lower_bound_GD_diff_small_eta}]
    Let $\overline{\suff}$ be the $\indsuff$-suffix average of $GD$ when is applied on $\optindf$. Let $\overline{\opt}=\argmin_w \optindf(w)$. By \cref{iden_grad_iter_small_eta}, we know that, $\suff=\overline{\suff}$. Then, by \cref{nonspec_lower_bound_GD_small_eta} and \cref{expec_change_bound},
\begin{align*}
  \frac{1}{4\sqrt{d}}&\leq  \optindf(\overline{\suff})-\optindf(\overline{\opt}) 
   \\&=\optindf(\suff)-\optindf(\overline{\opt})
    \\&\leq \doptindf(\suff)+\delta-\doptindf(\overline{\opt})+\delta 
    \\&\leq \doptindf(\suff)+\delta-\doptindf(\opt)+\delta,
\end{align*}
and, 
\begin{align*}
    \doptindf(\suff)-\doptindf(\opt)&\geq  \frac{1}{4\sqrt{d}}-\frac{\eta}{8d}
    \\&\geq 
    \frac{1}{4\sqrt{d}}-\frac{1}{8\sqrt{d}}
    \\&\geq 
    \frac{1}{8\sqrt{d}}
    \\&\geq 
    \min(\frac{1}{8},\frac{1}{40\eta T})
   . 
\end{align*}
\end{proof}